\documentclass{article}

\usepackage{amsmath,amssymb,amsfonts}
\usepackage{microtype}
\usepackage{graphicx}
\usepackage{subfigure}
\usepackage{booktabs} 

\usepackage{hyperref}

\usepackage[accepted]{icml2020}

\usepackage{times}
\usepackage{float}
\usepackage{amsmath}
\usepackage{amsthm}
\usepackage{mathrsfs}
\usepackage{bbm}
\usepackage{lineno}

\usepackage{fancyhdr}

\newtheorem{theorem}{Theorem}
\newtheorem{lemma}{Lemma}[section]
\newtheorem{assumption}{Assumption}
\newtheorem{corollary}{Corollary}
\def\cP{\mathcal{P}}

\addtocontents{toc}{\protect\setcounter{tocdepth}{-1}}

\begin{document}

\onecolumn

\begin{center} \Large \bf Minimax-Optimal Off-Policy Evaluation with Linear Function Approximation \end{center}
\setcounter{page}{1}
\cfoot{\vspace{2.5em}\thepage}

\vspace{1cm}




\begin{icmlauthorlist}
\icmlauthor{Yaqi Duan}{princeton}
\icmlauthor{Mengdi Wang}{princeton}
\end{icmlauthorlist}

\icmlaffiliation{princeton}{Princeton University}

\icmlcorrespondingauthor{Mengdi Wang}{mengdiw@princeton.edu}

\icmlkeywords{Machine Learning, ICML}

\vskip 0.3in



\printAffiliationsAndNotice{}  

\begin{abstract}
This paper studies the statistical theory of batch data reinforcement learning with function approximation. Consider the off-policy evaluation problem, which is to estimate the cumulative value of a new target policy from logged history generated by unknown behavioral policies. We study a regression-based fitted Q iteration method, and show that it is equivalent to a model-based method that estimates a conditional mean embedding of the transition operator. We prove that this method is information-theoretically optimal and has nearly minimal estimation error. In particular, by leveraging contraction property of Markov processes and martingale concentration, we establish a finite-sample instance-dependent error upper bound and a nearly-matching minimax lower bound. The policy evaluation error depends sharply on a restricted $\chi^2$-divergence over the function class between the long-term distribution of the target policy and the distribution of past data. This restricted $\chi^2$-divergence is both instance-dependent and function-class-dependent. It characterizes the statistical limit of off-policy evaluation. Further, we provide an easily computable confidence bound for the policy evaluator, which may be useful for optimistic planning and safe policy improvement.
\end{abstract}

\section{Introduction}

Batch data reinforcement learning (RL) is common in decision-making applications where rich experiences are available but new experiments are costly. A first-order question is how much one can learn from existing experiences to predict and improve the performance of new policies. This is known as the off-policy policy evaluation (OPE) problem, where one needs to estimate the cumulative rewards (aka value) to be earned by a new policy based on logged history. 

In this paper, we study the off-policy evaluation using linear function approximation. We assume that the Q-functions of interests belong to a known function class $\mathcal{Q}$ with $d$ basis functions.  
We adopt a direct regression-based approach and investigate the basic fitted Q iteration (FQI) \cite{bertsekas1995dynamic,sutton2018reinforcement}. It works by iteratively estimating Q-functions via supervised learning using the batch data. This approach turns out to be equivalent to the model-based plug-in estimator where one estimates the conditional mean embedding of the unknown transition model and uses it to compute a plug-in value estimator. It is also related to variants of importance sampling methods (see discussions in Sections 1.1 and 3.3).

We provide a finite-sample error upper bound for this policy evaluator, as well as a nearly matching  minimax-optimal lower bound. Putting them together, we see that the regression-based policy evaluator is nearly statistical-optimal.
For RL with horizon $H$, the minimax-optimal OPE error takes the form \vspace{-0.1cm}
$$|\widehat v^{\pi}- v^{\pi}| \asymp H^2 \sqrt{\frac{1+\chi_{\mathcal{Q}}^2 (\mu^{\pi}, \bar \mu)}{N}} +o(1/{\sqrt{N}}), $$
where $N$ is the number of observed state transitions,
$\mu^{\pi}$ is some long-term state-action occupancy measure of the target policy $\pi$ and $\overline \mu$ is the data distribution, $\chi_{\mathcal{Q}}^2$ is a variant of $\chi^2$-divergence {\it restricted to the family $\mathcal{Q}$}:
$$\chi_{\mathcal{Q}}^2(p_1,p_2) := \sup_{f\in\mathcal{Q}} \frac{\mathbb{E}_{p_1}[f(x)]^2}{\mathbb{E}_{p_2}[f(x)^2]}-1.$$
The term $\chi_{\mathcal{Q}}^2 (\mu^{\pi}, \overline \mu)$ captures the distributional mismatch, between the behavior policy and the target policy, that is relevant to the function class $\mathcal{Q}$. It determines the theoretical limits of OPE within this function class. In the tabular case, it relates to the worst-case density ratio, which often shows up in importance sampling methods. However, when we use function approximation, this $\chi_{\mathcal{Q}}^2$ divergence term can be significantly smaller than the worst-case density ratio. In particular, our analysis shows that $\chi_{\mathcal{Q}}^2 (\mu^{\pi}, \bar \mu)$ is the condition number of a finite matrix, which can be reliably estimated. This result suggests that OPE could be more data-efficient with appropriate function approximation. 

A summary of technical results of this paper: \vspace{-0.4cm}
\begin{itemize} \itemsep = -0.05cm
\item A regression-based algorithm that unifies FQI and plug-in estimation. It does not require knowledge of the behavior policy $\overline \pi$, or try to estimate it. It uses iterative regression but does not require Monte Carlo sampling. In the case of linear models, the estimator can be computed easily using simple matrix-vector operations.
\item Finite-sample error upper bound for the regression-based policy evaluator. Despite that regression may be biased for OPE, we show that the curse of horizon does not  occur as long as $N = \Omega(dH^3)$. A key to the analysis is the use of contraction properties of a Markov process to show that estimation error accumulates linearly in multi-step policy evaluation, instead of exponentially.

\item A minimax error lower bound that sets the statistical limit for OPE with function approximation. The lower bound nearly matches our upper bound, therefore proves the efficiency of regression-based FQI.
\item A data-dependent confidence bound that can be computed as a byproduct of the FQI algorithm.
\end{itemize}

\subsection{Related Literature}


Off-policy policy evaluation (OPE) is often the starting point of batch reinforcement learning.
A direct approach is to estimate the transition probability distributions and then execute the target policy on an estimated model. This has been studied in the tabular case with bias and variance analysis \cite{mannor2004bias}.
In real-world applications, in order to tackle MDPs with infinite or continuous state spaces, one often needs various forms of function approximation, and many methods like fitted Q-iteration and least square policy iteration were developed \cite{jong2007model,lagoudakis2003least,grunewalder2012modelling,fonteneau2013batch}.
Regression methods are often used to fit value functions and to satisfy the Bellman equation \cite{bertsekas1995dynamic, sutton2018reinforcement}.

A popular class of OPE methods use importance sampling (IS) to reweigh sample rewards to get unbiased value estimate of a new policy \cite{precup2000eligibility}. 
Doubly robust technique blends IS with model-based estimators to reduce the high variance \cite{jiang2015doubly, thomas2016data}.
\citet{liu2018breaking} suggested that one should estimate the stationary state occupancy measure instead of the cumulative importance ratio in order to break the curse of horizon. Many IS methods only apply to tabular MDP and require knowledge of the behavior policy. Following these ideas, \citet{nachum2019dualdice} proposed a minimax optimization problem that uses function approximation to learn the IS weights, without requiring knowledge of the behavior policy. \citet{dann2018policy} provided error bounds and certificates for the tabular case to achieve  accountability. \citet{liu2019off} studied off-policy gradient method for batch data policy optimization.


On the theoretical side, the sharpest OPE error bound to our best knowledge is given by \citet{xie2019towards} and \citet{ yin2020asymptotically}, which applies to time-inhomogeneous, tabular MDP. \citet{jiang2015doubly} provided a Cramer-Rao lower bound for discrete-tree MDP. To the authors' best knowledge, existing theoretical results on OPE mostly apply only to tabular MDP without function approximation. Our results appear to be the first and sharpest error bounds for OPE with linear function approximation.

\section{Problem and Model} \label{Section3}

In this paper, we study off-policy policy evaluation of an Markov decision process (MDP) when we only have a fixed dataset of empirical transitions. An instance of MDP is a controlled random walk over a state space $\mathcal{S}$, where at each state $s$, if we pick action $a\in \mathcal{A}$, the system evolves to a random next state $s'$ according to distribution $p(s' \, | \, s,a)$ and generates a reward $r' \in [0,1]$ with $\mathbb{E}[r' \, | \, s,a] = r(s,a)$. A policy $\pi$ specifies a distribution $\pi(\cdot \, | \,s)$ for choosing actions conditioned on the current state $s$. 

Our objective is to evaluate the performance of a {\it target} policy $\pi$ at a fixed initial distribution $\xi_0$, where the transition model $p$ is unknown. The value to be estimated is the expected cumulative reward in an $H$-horizon episode, given by \vspace{-0.05cm}
\begin{equation} \label{vpi} v^{\pi} := \mathbb{E}^{\pi} \Bigg[ \sum_{h=0}^H r(s_h,a_h) \, \Bigg| \, s_0 \sim \xi_0 \Bigg],\end{equation}
where $a_h \sim \pi(\cdot \, | \, s_h)$, $s_{h+1} \sim p(\cdot \, | \, s_h, a_h)$, $\mathbb{E}^{\pi} $ denotes expectation over the sample path generated under policy $\pi$.

Let $\mathcal{D} \!=\! \{ (s_n,a_n,s_n',r_n') \}_{n=1}^N$ be a set of sample transitions, where each $s_n'$ is sampled from distribution $p(\cdot \, | \,
s_n,a_n)$. The sample transitions may be collected from multiple trajectories and under a possibly unknown {\it behavior} policy denoted as $\overline{\pi}$. Our goal is to estimate $v^{\pi}$ from $\mathcal{D}$.


Given a {\it target} policy $\pi$ and a reward function $r$, the state-action value functions, also known as Q functions, are defined as, for $h=0,1,\ldots,H$, \vspace{-0.1cm}
\begin{equation} \label{Q_def} Q_h^{\pi}(s,a) := \mathbb{E}^{\pi} \Bigg[ \sum_{h'=h}^H \! r(s_{h'}, a_{h'}) \, \Bigg| \, s_h = s, a_h = a \Bigg], \end{equation}
where $a_{h'} \sim  \pi(\cdot \, | \, s_{h'}), s_{h'+1}  \sim p(\cdot \, | \, s_{h'}, a_{h'})$.
Let $\mathcal{X}:=\mathcal{S} \times \mathcal{A}$.
	Define the {\it conditional transition operator} $\mathcal{P}^{\pi} : \mathbb{R}^{\mathcal{X}} \rightarrow \mathbb{R}^{\mathcal{X}}$ as 
	\vspace{-0.1cm}
	\[ \mathcal{P}^{\pi} f(s,a) := \mathbb{E}^{\pi} \big[ f(s',a') \, \big| \, s, a \big] \quad \text{for any } f: \mathcal{X} \rightarrow \mathbb{R}, \vspace{-0.1cm} \]
where $s' \sim p(\cdot \, | \, s,a)$ and $a' \sim \pi(\cdot \, | \, s')$.
Throughout the paper, we suppose that $\mathcal{P}^{\pi}$ operates in a function class $\mathcal{Q}$, such that we can approximate unknown Q functions within this family.
Assume without loss of generality that {$\mathbf{1}\in \mathcal{Q}$}.
\begin{assumption}[Function class] \label{Q_class} 
For any $f\in\mathcal{Q}$, $\mathcal{P}^{\pi}f\in \mathcal{Q}$, and $r\!\in\!\mathcal{Q}.$ It follows that $Q^{\pi}_0,\ldots,Q^{\pi}_H \!\in\! \mathcal{Q},$ where $\mathcal{Q} \! \subseteq \! \mathbb{R}^{\mathcal{X}}.$ 
\end{assumption} 

 In most parts of the paper, we assume that the transition data are collected from multiple independent episodes.



\begin{assumption}[Data generating process] \label{episode}
	The dataset $\mathcal{D}$ consists of samples from $K$ i.i.d. episodes $\boldsymbol{\tau}_1, \boldsymbol{\tau}_2, \ldots, \boldsymbol{\tau}_K$. 
	Each $\boldsymbol{\tau}_k$ has $H$ consecutive sample transitions generated by some policy on a single sample path, i.e., $\boldsymbol{\tau}_k = (s_{k,0}, a_{k,0} ,r_{k,0}',s_{k,1}, a_{k,1}, r_{k,1}', \ldots, s_{k,H}, a_{k,H}, r_{k,H}')$. 
	We also denote $s'_{k,h}=s_{k,h+1}$.
\end{assumption} \vspace{-0.1cm}

We will focus mainly on the case where $\mathcal{Q}$ is a linear space spanned by $d$ feature functions $\phi_1,\ldots,\phi_d$. Also note that the behavior policy $\overline{\pi}$ is {\it not} known. 
\vspace{-0.2cm}

\paragraph{Notations}


Denote $\mathcal{X} = \mathcal{S} \times \mathcal{A}$.
Let $\mathbb{R}^{\mathcal{X}}$ be the collection of all functions $f: \mathcal{X} \rightarrow \mathbb{R}$. For any $f \in \mathbb{R}^{\mathcal{X}}$, define $f^{\pi}: \mathcal{S} \rightarrow \mathbb{R}$ by $f^{\pi}(s) = \int_{\mathcal{A}} f(s,a) \pi(a \, | \, s) {\rm d}a$. 
If $A$ is a positive symmetric semidefinite matrix, let $\sigma_{\min}(A)$ denote its smallest eigenvalue, and let $A^{1/2}$ denote the positive symmetric semidefinite matrix that $A^{1/2}A^{1/2}= A$.
For nonnegative $\{ a_n \}_{n=1}^{\infty}$ and $\{ b_n \}_{n=1}^{\infty}$, we denote $a_n \lesssim b_n$ if there exists $c>0$ such that $a_n \leq c b_n$ for $n=1,2,\ldots$.
Let $\{ X_n \}_{n=1}^{\infty}$ be a sequence of random variables and $\{ a_n \}_{n=1}^{\infty} \subseteq \mathbb{R}$ be deterministic. We write $X_n = O_{\mathbb{P}}(a_n)$ if for any $\delta > 0$ there exists $M>0$ such that $\mathbb{P}(|X_n| > a_n M) \leq \delta$ for all $n$. 
If a distribution $p$ is absolutely continuous with respect to a distribution $q$, the Pearson $\chi^2$-divergence is defined by $\chi^{2}(p,q) := \mathbb{E}_q \big[ (\frac{{\rm d}p}{{\rm d}q} - 1)^2 \big]$.

\vspace{-0.05cm}

\section{Regression-Based Off-Policy Evaluation} \label{section:estimator}

We consider a fitted Q-iteration method for new policy evaluation using linear function approximation. We show that it is equivalent to a model-based method that estimates a conditional mean operator and embeds the unknown $p$ into the feature space. They admit a simple matrix-vector implementation when $\mathcal{Q}$ is a linear model with finite dimension. 

	\vspace{-0.2cm}

	\subsection{Fitted Q-iteration (FQI)}
	The Q-functions satisfy the Bellman equation \vspace{-0.1cm} \begin{equation} \label{Bellman}Q_{h-1}^{\pi}(s,a) = r(s,a) + \mathbb{E} \big[ V_h^{\pi}(s') \, \big| \, s, a \big] \vspace{-0.1cm} \end{equation}
	for $h\!=\!1,2,\ldots,H$,
	where $s' \!\sim p(\cdot\,|\,s,a)$, $V_h^{\pi}\!: \mathcal{S} \!\rightarrow\! \mathbb{R}$ is the value function defined as $V_h^{\pi}(s) := \int_{\mathcal{A}} Q_h^{\pi}(s,a) \pi(a \, | \, s) {\rm d}a$. 
	
	For the given target policy $\pi$, we apply regression recursively by letting $\widehat Q_{H+1}^{\pi} := 0$ and
	for $h=H,H-1,\ldots,0$, \vspace{-0.2cm}
	\begin{equation} \label{Q_recursion} 
	\widehat{Q}_h^{\pi} := \arg\min_{f \in \mathcal{Q}} \Bigg\{  \sum_{n=1}^N \bigg( f(s_n,a_n) - r_n' - \int_{\mathcal{A}} \!\widehat{Q}_{h+1}^{\pi}(s_n',a) \pi(a \, | \, s_n') {\rm d}a  \bigg)^2 + \lambda \rho (f) \Bigg\}, 
	 \vspace{-0.05cm}
	 \end{equation}
	where $ \lambda \geq 0$ and $\rho(\cdot)$ is a regularization function.
	The scheme above provides a recursive way to evaluate $\widehat{Q}_{H}^{\pi}, \widehat{Q}_{H-1}^{\pi}, \ldots, \widehat{Q}_0^{\pi}$ and $v^{\pi}$ by regression using empirical data. It is essentially a fitted Q-iteration. The full algorithm is summarized in Algorithm \ref{FQI}.
	
	
	\vspace{-0.2cm}
	\begin{algorithm}[H]
		\caption{Fitted Q-iteration for Off-Policy Evaluation (FQI-OPE)} \label{FQI}
		\begin{algorithmic}
		\STATE \hspace{-0.273cm} \begin{tabular}{p{0.8cm}l} {\bf Input:} & initial distribution $\xi_0$, target policy $\pi$, horizon $H$, function class $\mathcal{Q}$, \\ & sample transitions $\mathcal{D} = \{ (s_n,a_n,s_n',r_n') \}_{n=1}^N$ \end{tabular} 
		\vspace{0.005cm}
		\STATE Let $\widehat{Q}_{H+1}^{\pi} := 0$;
		\FOR {$h = H, H-1, \ldots, 1$}
		\STATE Calculate $\widehat{Q}_h$ by solving \eqref{Q_recursion};
		\ENDFOR
		\vspace{0.07cm} \\
		{\bf Output:} $\widehat{v}^{\pi}_{\mathsf{FQI}}:= \int_{\mathcal{X}} \widehat{Q}_0^{\pi}(s,a) \xi_0(s) \pi(a \, | \, s) {\rm d}s{\rm d}a$
		\end{algorithmic}
	\end{algorithm}
	
	\vspace{-0.3cm}
	\subsection{An equivalent model-based method using conditional mean operator}
	
	The preceding FQI method can be equivalently viewed as a model-based plug-in estimator.
	Recall the {\it conditional transition operator} $\mathcal{P}^{\pi} : \mathbb{R}^{\mathcal{X}} \rightarrow \mathbb{R}^{\mathcal{X}}$ is
	\vspace{-0.1cm}
	\[ \mathcal{P}^{\pi} f(s,a) := \mathbb{E}^{\pi} \big[ f(s',a') \, \big| \, s, a \big] \quad \text{for any } f: \mathcal{X} \rightarrow \mathbb{R}. \vspace{-0.1cm} \]
	Under Assumption \ref{Q_class}, it always holds that $\mathcal{P}^{\pi} Q_h^{\pi}  \in \mathcal{Q}$. To this end, we are only interested in a ``projection'' of ground-truth $\mathcal{P}^{\pi}$ onto $\mathcal{Q}$. 
	We estimate the conditional transition operator by $\widehat{\mathcal{P}}^{\pi}$: for any $f\!:\! \mathcal{X} \rightarrow \mathbb{R}$, let \vspace{-0.2cm}
	\begin{equation} \label{hatP_operator}
	\widehat{\mathcal{P}}^{\pi} f\!: =
	\arg\min_{g \in \mathcal{Q}} \! \Bigg\{ \! \sum_{n=1}^N \! \Big( g(s_n,a_n) -  \int_{\mathcal{A}} \!\! f(s_n', a) \pi(a \, | \, s_n'){\rm d}a \Big)\!\Big.^2 \!+\! \lambda \rho (g) \Bigg\}. \vspace{-0.1cm}
	\vspace{-0.1cm} \end{equation}
	We can see that, if $N\to\infty$, $\widehat{\mathcal{P}}^{\pi}$ converges to a projected version of ${\mathcal{P}}^{\pi}$ onto $\mathcal{Q}$.
	Denote $\phi(\cdot):= [ \phi_1(\cdot), \ldots, \phi_d(\cdot) ]^{\top}: \mathcal{X} \rightarrow \mathbb{R}^d$.
In the case where $\mathcal{Q}$ is a linear space given by $\mathcal{Q}=\big\{\phi(\cdot)^{\top}w \, \big| \, w\in\mathbb{R}^d\big\}$ and $\rho(\cdot)$ is taken as \vspace{-0.3cm} \begin{equation} \label{def_rho} \rho(f) := \|w\|_2^2 \qquad  \text{for }f(\cdot) = \phi(\cdot)^{\top} w, \vspace{-0.2cm} \end{equation} the constructed $\widehat{\mathcal{P}}^{\pi}$ in \eqref{hatP_operator} corresponds to an estimated $\widehat p$ {of the form} 
\vspace{-0.1cm}
\[
\widehat{p}(\cdot \, | \, s,a) := \phi(s,a)^{\top} \widehat{\Sigma}^{-1} \left(\,\sum^N_{n=1} \phi(s_n,a_n) \delta_{s_n'}(\cdot)\right), \vspace{-0.18cm} \]
	where $\widehat{\Sigma} := \lambda I + \sum_{n=1}^N \phi(s_n,a_n) \phi(s_n,a_n)^{\top} $ is the empirical covariance matrix and
	$\delta_{s'}(\cdot)$ denotes the Dirac measure. Note this $\widehat p$ is not necessary a transition kernel.
	
	
We adopt a model-based approach and use $\widehat{\mathcal{P}}^{\pi}$ in the Bellman equation as a plug-in estimator. In particular, let \vspace{-0.2cm} \begin{equation} \label{def_hatr} \widehat{r} := \arg\min_{f \in \mathcal{Q}} \Bigg\{ \sum_{n=1}^N \big( f(s_n, a_n) - r_n' \big)^2 + \lambda \rho(f) \Bigg\}, \vspace{-0.28cm} \end{equation} and {$\widehat{Q}_{H+1}^{\pi}:=0$,}\vspace{-0.2cm}
	\[ \widehat{Q}_{h-1}^{\pi} := \widehat{r} + \widehat{\mathcal{P}}^{\pi} \widehat{Q}_h^{\pi},\qquad h=H+1,H,\ldots,1. \vspace{-0.2cm} \]
	Then we can estimate the policy value by \vspace{-0.15cm}
	$$\widehat{v}_{\mathsf{Plug\text{-}in}}^{\pi} := \int_{s,a} \widehat{Q}_0^{\pi}(s,a) \xi_0(s) \pi(a \, | \, s) {\rm d}s{\rm d}a. \vspace{-0.15cm}$$
	It is easy to verify that this plug-in estimator is equivalent to the earlier FQI estimator. See the proof in Appendix \ref{section:equivalence}.
	\begin{theorem}[Equivalence between FQI and a model-based method] \label{theorem:equivalence}
	If $\mathcal{Q}$ is a linear space and $\rho$ is given by \eqref{def_rho}, Algorithm \ref{FQI} and the preceding plug-in approach generate identical policy value estimators, {\it i.e.}
		$\widehat{v}^{\pi}\!\!:=\widehat{v}_{\mathsf{FQI}}^{\pi}\!=\!\widehat{v}_{\mathsf{Plug\text{-}in}}^{\pi}$.
	\end{theorem} 

	

When $\mathcal{Q}$ is a $d$-dimensional linear space with the feature map $\phi$, under Assumption 1, there exists a matrix $M^{\pi} \in \mathbb{R}^{d\times d} $ such that \vspace{-0.15cm}
$$\phi(s,a)^{\top} M^{\pi}=\mathbb{E}\big[ \, \phi^{\pi}(s')^{\top} \, \big| \, s,a \,\big],\quad\forall (s,a) \in \mathcal{X}, \vspace{-0.15cm} $$
where $\phi^{\pi}(s) :=\int \phi(s,a)\pi(a|s)da$.
We refer to $M^{\pi}$ as the {\it matrix mean embedding} of the conditional transition operator $\cP^{\pi}$. We can implement Algorithm \ref{FQI} in simple vector forms.
	We embed the one-step reward function and conditional transition operator into a vector and a matrix, respectively: \vspace{-0.2cm}
	\begin{equation} \label{vector_recursion} \begin{aligned} & \widehat{r}(\cdot) = \phi(\cdot)^{\top} \widehat{R} \quad \text{with }\widehat{R} := \widehat{\Sigma}^{-1} \Bigg( \sum_{n=1}^N r_n' \phi(s_n,a_n) \Bigg), \qquad \widehat{M}^{\pi} := \widehat{\Sigma}^{-1} \Bigg( \sum_{n=1}^N \phi(s_n,a_n) \phi^{\pi}(s_n')^{\top} \Bigg). \end{aligned} \vspace{-0.2cm} \end{equation}
The corresponding conditional mean operator $\widehat{\mathcal{P}}^{\pi}$ is \vspace{-0.1cm}
	\begin{equation} \label{hatPpi} (\widehat{\mathcal{P}}^{\pi} f )(s,a) = \phi(s,a)^{\top} \widehat{M}^{\pi} w,\ \ \hbox{for }f(\cdot) =  \phi(\cdot)^{\top} w. \vspace{-0.1cm} \end{equation}
We represent $\widehat{Q}_h^{\pi}$ in the form of $\widehat{Q}_h^{\pi}(s,a) = \phi(s,a)^{\top}\widehat{w}_h^{\pi}$. In this way, we can easily compute $\widehat{Q}_h^{\pi}$ using recursive compact vector-matrix operations, as given in Algorithm \ref{Model-based}. 

	\vspace{-0.2cm}
	\begin{algorithm}[H]
		\caption{Conditional Mean Embedding for Off-Policy Evaluation (CME-OPE)} \label{Model-based}
		\begin{algorithmic}
			\STATE \hspace{-0.273cm} \begin{tabular}{p{0.8cm}l} {\bf Input:} & initial distribution $\xi_0$, target policy $\pi$, horizon $H$, a basis $\{\phi_1,\ldots,\phi_d\}$ of $\mathcal{Q}$, \\ & sample transitions $\mathcal{D} = \{ (s_n,a_n,s_n',r_n') \}_{n=1}^N$, \end{tabular} 
			\vspace{0.07cm}
			\STATE Estimate $\widehat{R}$ and $\widehat{M}^{\pi}$ according to \eqref{vector_recursion};
			\STATE Let $\widehat{w}_{H+1}^{\pi} := 0$;
			\STATE Let $\nu_0^{\pi} := \int_{\mathcal{X}} \phi(s,a) \xi_0(s) \pi(a \, | \, s) {\rm d}s {\rm d a}$;
			\FOR {$h = H, H-1, \ldots, 0$}
			\STATE Calculate $\widehat{w}_h^{\pi} := \widehat{R} + \widehat{M}^{\pi} \widehat{w}^{\pi}_{h+1}$;
			\ENDFOR
			\vspace{0.07cm} \\
			{\bf Output:} $\widehat{v}^{\pi} := (\nu_0^{\pi})^{\top}\widehat w_0^{\pi}$
		\end{algorithmic}
	\end{algorithm}

\subsection{Relations to other methods}

Our method turns out to be closely related to variants of importance sampling method for OPE. For examples: \vspace{-0.2cm}
\begin{itemize} \itemsep = -0.05cm
	\item {\it Marginalized importance sampling:} Our FQI estimator takes the form $\widehat{v}^{\pi} = \frac{1}{N} \sum_{n=1}^N \widehat{w}_{\pi/\mathcal{D}}(s_n,a_n) r_n'$ where $\widehat{w}_{\pi/\mathcal{D}}(s,a) \! := \! N \sum_{h=0}^H (\nu_0^{\pi})^{\top}(\widehat{M}^{\pi})^h \widehat{\Sigma}^{-1} \phi(s,a)$. By viewing $\widehat{w}_{\pi/\mathcal{D}}(s,a)$ as weights, our estimator can be obtained equivalently by importance sampling. In the special tabular case, our $\widehat{v}^{\pi}$ reduces to the marginalized importance sampling (MIS) estimator in \cite{yin2020asymptotically}.
	\item {\it DualDICE:} \citet{nachum2019dualdice} proposed a minimax formulation to find the stationary state occupancy measure and residue (weight for importance sampling) with function approximation. We observe that, if those function classes are taken to be $\mathcal{Q}$, a version of DualDICE produces the same estimator as the FQI estimator. The two methods can be viewed as dual to each other.
\end{itemize}
\vspace{-0.2cm}
See Appendix \ref{section:equivalence} for more discussions.

\section{Finite-Sample Error Bound} \label{section:UpperBound}

Recall that $\mathcal{D}$ is a collection of $K$ independent $H$-horizon trajectories. 
Let $\Sigma$ be the uncentered covariance matrix of the data distribution: \vspace{-0.15cm}
$$\Sigma = \mathbb{E} \Bigg[ \frac{1}{H} \sum_{h=0}^{H-1} \phi(s_{1,h}, a_{1,h}) \phi(s_{1,h},a_{1,h})^{\top} \Bigg],\vspace{-0.15cm}$$which is determined by the unknown behavior policy $\overline \pi$. Given a target policy $\pi$, let $\xi^{\pi}$ be an invariant distribution of the Markov chain with transition kernel $p^{\pi}(s' \, | \, s) = \int_{\mathcal{A}} p(s' \, | \, s,a) \pi(a \, | \, s) {\rm d} a$. Define \vspace{-0.1cm}
$$\Sigma^{\pi} := \mathbb{E}\big[ \phi^{\pi}(s) \phi^{\pi}(s)^{\top} \, \big| \, s \sim \xi^{\pi} \big].\vspace{-0.1cm}$$
We assume $\phi(s,a)^{\top} \Sigma^{-1} \phi(s,a) \leq C_1d$ without loss of generality.
Theorem \ref{UpperBound} provides an instance-dependent policy evaluation error upper bound. Its complete proof is given in Appendix \ref{appendix:proof:UpperBound}.

\begin{theorem}[Upper bound] \label{UpperBound}
	Let $\delta \in (0,1)$. Under Assumptions \ref{Q_class} and \ref{episode}, if $N \!\geq\! 20 \kappa_1(2+\kappa_2)^2 \ln(12dH/\delta)C_1dH^3$ and $\lambda \leq \ln(12dH/\delta)C_1dH \sigma_{\min}(\Sigma)$, then with probability at least $1 - \delta$, \vspace{-.2em}
	\begin{equation} \label{ub1} \big| v^{\pi} \! - \widehat{v}^{\pi} \big| \! \leq \! \sum_{h=0}^H \! (H-h+1) {\sup_{f \in \mathcal{Q}} \frac{\mathbb{E}^{\pi}\big[f(s_h,a_h) \, \big| \, s_0 \sim \xi_0 \big]}{\!\sqrt{\mathbb{E}\big[ \frac{1}{H} \! \sum_{h=0}^{H-1} \! f^2(s_{1,h},a_{1,h}) \big]}}} \cdot \sqrt{\frac{\ln(12/\delta)}{2N}}+ \frac{C\ln(12dH/\delta) dH^{3.5}}{N},  \end{equation}
	where
	$C := 15\kappa_1 C_1(3+\kappa_2) \sqrt{(\nu_0^{\pi})^{\top} \Sigma^{-1} \nu_0^{\pi}}$, $\kappa_1 := {\rm cond}\big(\Sigma^{-1/2} \Sigma^{\pi} \Sigma^{-1/2}\big)$, \\ $\kappa_2 := \big\|\Sigma^{-1/2} \mathbb{E}\big[ \frac{1}{H} \! \sum_{h=1}^H \! \phi^{\pi}(s_{1,h}) \phi^{\pi}(s_{1,h})^{\top} \big] \Sigma^{-1/2} \big\|_2 \!\vee\! 1$.
	
	Additionally, if either one of the following holds: \begin{itemize} \itemsep = -0.1cm
		\item $\phi(s,a)^{\top} \Sigma^{-1} \phi(s',a') \geq 0$ for any $(s,a), (s',a') \in \mathcal{X}$;
		\item the MDP is time-inhomogeneous,
	\end{itemize} \vspace{-0.3cm} the upper bound can be improved to \vspace{-.2em}
	\begin{equation} \label{ub2}  |v^{\pi} - \widehat{v}^{\pi}| \leq {\sup_{f \in \mathcal{Q}} \frac{\mathbb{E}^{\pi}\big[\sum_{h=0}^{H} (H-h+1) f(s_h,a_h) \, \big| \, s_0 \sim \xi_0 \big]}{\sqrt{\mathbb{E}\big[ \frac{1}{H} \sum_{h=0}^{H-1} f^2(s_{1,h},a_{1,h}) \big]}}} \cdot \sqrt{\frac{\ln(12/\delta)}{2N}} + \frac{C\ln(12dH/\delta) dH^{3.5}}{N}.  \end{equation}
\end{theorem}

\noindent {\bf Distributional mismatch as a $\mathcal{Q}$-$\chi^2$-divergence.}  \\	
Let $\overline{\mu}$ be the expected occupancy measure of observation $\{ (s_n, a_n) \}_{n=1}^N$. Let $\mu^{\pi}$ be the weighted occupancy distribution of $(s_h,a_h)$ under policy $\pi$ and $\xi_0$, given by \vspace{-0.12cm}
	{$$\mu^{\pi}(s,a):=\frac{\mathbb{E}^{\pi}\big[\sum^{H}_{h=0} (H-h+1) {\bf 1}(s_h=s, a_h=a)\big]}{\sum^{H}_{h=0} (H-h+1)}.$$	}
	The upper bound \eqref{ub2} can be simplified to \vspace{-0.12cm}
	$$|\widehat v^{\pi}- v^{\pi}| \leq C H^2 \sqrt{\frac{1+\chi_{\mathcal{Q}}^2 (\mu^{\pi}, \bar \mu)}{N}} +O(N^{-1}). \vspace{-0.05cm}$$

Moreover, each mismatch term in \eqref{ub1} has a vector form \vspace{-0.1cm} \[ {\frac{\mathbb{E}^{\pi}[f(s_h,a_h) \, \big| \, s_0 \sim \xi_0 ]}{\sqrt{\mathbb{E}[ \frac{1}{H} \sum_{h=0}^{H-1} f^2(s_{1,h},a_{1,h}) ]}}} = \sqrt{(\nu_h^{\pi})^{\top} \Sigma^{-1} \nu_h^{\pi}}, \vspace{-0.12cm} \] where $\nu_h^{\pi} \! := \mathbb{E}^{\pi}[ \phi(s_h,a_h) \, \big| \, s_0 \sim \xi_0]$, {so it can be estimated tractably.}

\noindent {\bf The case of tabular MDP.} \\
	In the tabular case, the condition $\phi(s,a)^{\top} \Sigma^{-1} \phi(s',a') \geq 0$ holds for all $(s,a), (s',a') \in \mathcal{X}$. It can be easily seen that the error bound \eqref{ub2} has a strong connection with the $\chi^2$-divergence between the state-action distributions under the behavior and target policies.

	\begin{corollary}[Upper bound in tabular case] \label{cor:tabular}
		In the tabular case with {$\mathcal{Q} = \mathbb{R}^{\mathcal{X}}$}, if $N$ is sufficiently large and $\lambda = 0$, then with probability at least $1-\delta$, \vspace{-0.2cm}
		\begin{equation} \label{tabular} \big| v^{\pi} - \widehat{v}^{\pi} \big| \leq 3H^2\sqrt{ 1 + \chi^2(\mu^{\pi}, \overline{\mu})} \sqrt{\frac{\ln(12/\delta)}{2N}} + O(N^{-1}),
		\vspace{-0.1cm} \end{equation}
		where $\chi^2(\cdot, \cdot)$ denotes the Pearson $\chi^2$-divergence. If the MDP is also time-inhomogeneous, then \vspace{-0.1cm}
		\begin{equation} \label{tabular&time} |v^{\pi} \!-\! \widehat{v}^{\pi}| \leq \sqrt{\! H \! \sum_{h=0}^H \sum_{s,a} \! \frac{\mu_h^{\pi}(s,a)^2}{\overline{\mu}_h(s,a)} {\rm Var} \big[ r' \!+\! V_{h+1}^{\pi}(s') \, \big| \, s, a \big]} \cdot \sqrt{\frac{2\ln(12/\delta)}{N}} + o(N^{-1/2}),  \end{equation}
		where $\overline{\mu}_h$ is the marginal distribution of $(s_{1,h}, a_{1,h})$ and $\mu_h^{\pi}$ is the marginal distribution of $(s_h,a_h)$ under policy $\pi$ and $\xi_0$.
	\end{corollary}
	The tabular-case upper bound \eqref{tabular&time} has the same form with Theorem 3.1 in \citet{yin2020asymptotically}. The proof of Corollary 1 is deferred to Appendix \ref{appendix:proof:cor}.
	
\vspace{-0.2cm}

\subsection{Proof Outline}

\def\cP{\mathcal{P}}

We decompose the error into three terms:
$v^{\pi} - \widehat{v}^{\pi} = E_1 + E_2 + E_3$, where $E_1$ is {a linear function of $\widehat{\cP}^{\pi}-{\cP}^{\pi}$}, $E_2$ is a high-order function of $\widehat{\cP}^{\pi}-{\cP}^{\pi}$ and $E_3=O(\lambda)$.
In the following, we outline the analysis of $E_1$ and $E_2$.

{\bf First-order term $E_1$: }
This linear error term takes the form $E_1 \!=\! \frac{1}{N} \! \sum_{n=1}^N \! e_n$, where \vspace{-0.15cm}
\[ e_n :=\sum_{h=0}^{H} \, (\nu_h^{\pi})^{\top} \Sigma^{-1} \phi(s_n,a_n) \Big( Q_h^{\pi}(s_n, a_n) - \big( r_n'+V_{h+1}^{\pi}(s_n') \big) \Big).\] 
Define a filtration $\{ \mathcal{F}_n\}_{n=1,\ldots,N}$ where $\mathcal{F}_n$ is generated by $(s_1,a_1,s_1',r_1'), \ldots, (s_{n-1},a_{n-1},s_{n-1}',r_{n-1}')$ and $(s_n,a_n)$. Then $e_1,e_2,\ldots,e_N$ is a martingale difference sequence with respect to $\{ \mathcal{F}_n \}_{n=1,\ldots,N}$. In what is next, we analyze ${\rm Var}[e_n \, | \, \mathcal{F}_n]$ and apply the Freedman's inequality \cite{freedman1975tail} to derive a finite sample upper bound for $E_1$.

Consider the conditional variance ${\rm Var}[e_n \, | \, \mathcal{F}_n]$.
By using the Cauchy-Schwarz inequality and the relation ${\rm Var}\big[r_n'+V_{h+1}^{\pi}(s_n') \, \big| \, s_n, a_n \big] \leq \frac{1}{4}(H-h+1)^2$, we have 
\begin{equation}\label{E1_1} {\rm Var}\big[e_n \, \big| \, \mathcal{F}_n\big] = \mathbb{E}\big[ e_n^2 \, | \, s_n, a_n \big] 
\leq \frac{1}{4} \Bigg( \sum_{h=0}^{H} (H-h+1) \sqrt{(\nu_h^{\pi})^{\top} \Sigma^{-1} \nu_h^{\pi}} \Bigg) \cdot \left( \sum_{h=0}^{H} \frac{H-h+1}{\sqrt{(\nu_h^{\pi})^{\top} \Sigma^{-1} \nu_h^{\pi}}} \big((\nu_h^{\pi})^{\top} \Sigma^{-1} \phi(s_n,a_n) \big)^2 \right). 
\end{equation}
We learn from the matrix-form Bernstein inequality that $\frac{1}{N} \sum_{n=1}^N \phi(s_n,a_n) \phi(s_n,a_n)^{\top}$ concentrates around $\Sigma$ with high probability. It follows that \vspace{-0.1cm}
\begin{equation}\label{E1_2} \sum_{n=1}^N \big((\nu_h^{\pi})^{\top} \Sigma^{-1} \phi(s_n,a_n) \big)^2 = (\nu_h^{\pi})^{\top}  \Sigma^{-1} \! \Bigg( \sum_{n=1}^N \phi(s_n,a_n) \phi(s_n,a_n)^{\top} \Bigg) \Sigma^{-1} \nu_h^{\pi} = (\nu_h^{\pi})^{\top} \Sigma^{-1} \nu_h^{\pi} \big( N +\sqrt{dH} \cdot O_{\mathbb{P}}(\sqrt{N}) \big). \vspace{-0.05cm} \end{equation}
Plugging \eqref{E1_2} into \eqref{E1_1} and taking the summation, we obtain \vspace{-0.15cm}
\[ \sum_{n=1}^N {\rm Var}[e_n \, | \, \mathcal{F}_n] \leq  \frac{1}{4} \bigg( \sum_{h=0}^{H} (H-h+1)  \sqrt{(\nu_h^{\pi})^{\top} \Sigma^{-1} \nu_h^{\pi}} \bigg)^2 \cdot \big( N + \sqrt{dH} \cdot O_{\mathbb{P}}(\sqrt{N}) \big). \] It follows from the Freedman's inequality that with high probability, \vspace{-0.1cm}
\[ |E_1| \lesssim \frac{1}{\sqrt{N}} \sum_{h=0}^{H} (H-h+1) \sqrt{(\nu_h^{\pi})^{\top} \Sigma^{-1} \nu_h^{\pi}} + \frac{\sqrt{dH}}{N}. \]

\paragraph{High-order term $E_2$ (bias-inducing term):} 
The high-order term $E_2$ involves powers of $\widehat{\cP}^{\pi}-{\cP}^{\pi}$. We use the contraction property of Markov process with respect to its invariant measure, in particular,
 \vspace{-0.1cm} \begin{equation} \label{M=1} \big\|(\Sigma^{\pi})^{1/2} M^{\pi} (\Sigma^{\pi})^{-1/2}\big\|_2\leq1. \end{equation}
where $\Sigma^{\pi} = \mathbb{E}\big[ \phi^{\pi}(s) \phi^{\pi}(s)^{\top} \, \big| \, s \sim \xi^{\pi} \big]$, $\xi^{\pi}$ is an invariant distribution under policy ${\pi}$. Assume $\Sigma^{\pi}$ has full rank for simplicity. 


By using the contraction property, we will see that the value error will not grow exponentially in $H$ for large $N$. We have: \vspace{-0.1cm}
\begin{equation} \label{E2decom} |E_2| \leq \sum_{h=0}^{H}  \sqrt{(\nu_0)^{\top} (\Sigma^{\pi})^{-1} \nu_0^{\pi}} \cdot Err(Q_h^{\pi})  \cdot \Big( \big(1+Err(\widehat{M}^{\pi})\big)^h\big(1+Err(N\widehat{\Sigma}^{-1})\big) - 1 \Big), \end{equation}
where the explicit definitions of errors $Err(\widehat{M}^{\pi})$, $Err(N \widehat{\Sigma}^{-1})$ and $Err(Q_h^{\pi})$ can be found in Lemma \ref{lemma:E2decompose}, Appendix \ref{section:E2}. 
By concentration arguments, we can show $Err(\widehat{M}^{\pi})$, $Err(N \widehat{\Sigma}^{-1}) \lesssim\sqrt{dH/N}$ and $Err(Q_h^{\pi}) \lesssim (H-h+1) \sqrt{d/N}$ with high probability. 
According to \eqref{E2decom}, as long as $ Err(\widehat{M}^{\pi}) \lesssim H^{-1}$, 
the policy evaluation error will not grow exponentially in $H$.
As a result, if $N \gtrsim dH^3$, we have $|E_2| \lesssim dH^{3.5}/N$.
\vspace{-0.6cm}
\begin{flushright}
	$\square$
\end{flushright}

\section{Minimax Lower Bound}

In this section, we establish a minimax lower bound that characterizes the hardness of off-policy evaluation using linear function approximators. Theorem \ref{theorem:lb} nearly matches the finite-sample upper bound given by Theorem \ref{UpperBound}. 
The complete proof of Theorem \ref{theorem:lb} is given in Appendix \ref{appendix:proof:LowerBound}.

\def\cS{\mathcal{S}}
\begin{theorem}[Minimax lower bound] \label{theorem:lb}
Suppose that an MDP instance $M=(p,r)$ satisfies:
\begin{itemize}
\item There exists a set of high-value states $\overline \cS\subseteq \cS$ and a set of low-value states $\underline \cS\subseteq \cS$ under the target policy $\pi$ such that $V_h^{\pi}(s) \geq \frac34 (H-h+1)$ if $s\in \overline \cS$ and $V_h^{\pi}(s) \leq \frac14 (H-h+1)$ if $s\in \underline \cS$;
\item $\overline{p} := \int_{\overline{\mathcal{S}}} \min_{s \in \mathcal{S}} p^{\overline{\pi}}(s' \, | \, s) {\rm d}s' \geq c $ and $\underline{p} := \int_{\underline{\mathcal{S}}} \min_{s \in \mathcal{S}} p^{\overline{\pi}}(s' \, | \, s) {\rm d}s' \geq c$ for $c>0$.\footnote{We assume the bahavior policy $\overline{\pi}$ is deterministic only for the sake of notational simplicity.}
\end{itemize}
For any behavior policy $\overline{\pi}$, when $N$ is sufficiently large,  one has \vspace{-0.1cm}
	\begin{equation} \label{lowerbound} \inf_{\widehat{v}^{\pi}} \sup_{M'\in \mathcal{N}(M)} \mathbb{P}_{M'} \left( \big| v^{\pi} - \widehat{v}^{\pi}(\mathcal{D}) \big| \geq \frac{\sqrt{c}}{24\sqrt{N}} \cdot \sup_{f \in \mathcal{Q}} \frac{\mathbb{E}_{M'}^{\pi} \big[ \sum_{h=0}^{H-1}(H-h)f(s_h,a_h) \, \big| \, s_0 \sim \xi_0 \big]}{\sqrt{\mathbb{E}_{M'}\big[ \frac{1}{H} \sum_{h=0}^{H-1} f^2(s_{1,h},a_{1,h}) \big]}} \right) \geq \frac{1}{6}, \vspace{-0.1cm} \end{equation}
	where $\mathcal{N}(M)$ is a small neighborhood of $M$ given by $\mathcal{N}(M) := \big\{ M' = (p',r) \ \big| \, \sup_{(s,a) \in \mathcal{X}} \big\| p'(\cdot \, | \, s,a) - p(\cdot \, | \, s,a) \big\|_{\rm TV} \leq \varepsilon \big\}$ ($\| \cdot \|_{\rm TV}$ denotes the total variation, $\varepsilon \gtrsim \sqrt{cd/N}$). $\mathbb{P}_{M'}$ is the probability space of $M'$, $\widehat{v}^{\pi}(\mathcal{D})$ is the output of some algorithm $\widehat v^{\pi}$ when $\mathcal{D}$ is given as the input.
\end{theorem}


{\bf Remark.} The minimax lower bound is a worst-case error lower bound that applies to {\it any estimator}, biased or unbiased. Typical minimax lower bound takes the form of $\inf_{\widehat v}\sup_{\mathcal{M}}$ where the sup is taken over the entire class of MDP instances $\mathcal{M}$. Our lower bound is much stronger and can be easily relaxed to the typical form. 

Compare Theorems \ref{UpperBound} and \ref{theorem:lb}. They nearly match each other, implying that the $\mathcal{Q}$-$\chi^2$-divergence term $\chi_{\mathcal{Q}}^2(\overline{\mu},\mu^{\pi})$ determines the statistical complexity of OPE.
	
\paragraph{An example.}

Suppose that there is a high-value state  $\overline{s}$ and a low-value state $\underline{s}$, which are two absorbing states under the target policy $\pi$, with rewards $1$ and $0$ respectively.

We construct $\phi$, $\pi$ and $\overline{\pi}$ such that
	$\phi^{ \overline{\pi}} (\overline{s}) = [z,1-z]^{\top}$, $\phi^{ \overline{\pi}} (\underline{s}) = [1-z,z]^{\top}$; and		$\phi^{\pi}(\overline{s}) = [1,0]^{\top}$, $\phi^{\pi}(\underline{s} ) = [0,1]^{\top}$.  Here $z\in[0,1]$ is a parameter. 
	We construct the transition model as: 
		\vspace{-0.1cm}
	\begin{figure}[H]
		\centering
		\begin{minipage}{0.25\linewidth}
			\centering 
			$p$ under behavior policy $\overline{\pi}$: \vspace{-0.22cm}
			\includegraphics[width = \linewidth]{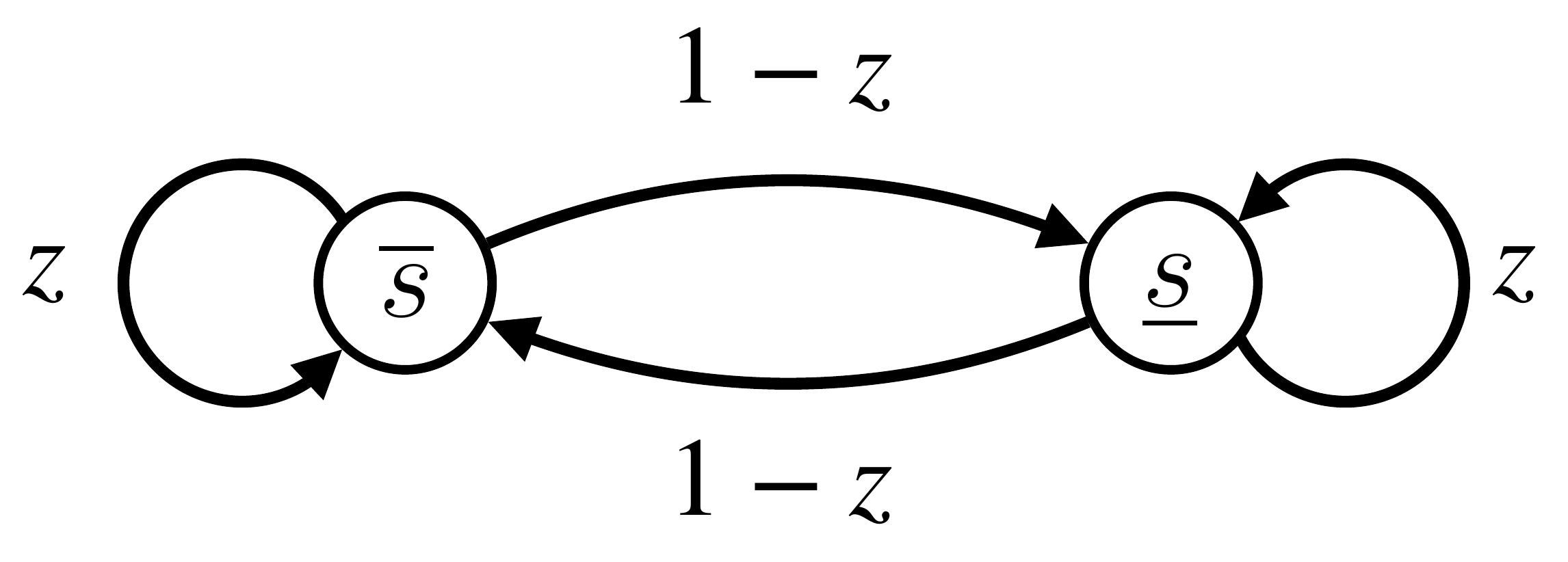}
		\end{minipage}
		\hspace{2cm}
		\begin{minipage}{0.25\linewidth}
			\centering
			$p$ under target policy $\pi$: \vspace{-0.22cm}
			\includegraphics[width = \linewidth]{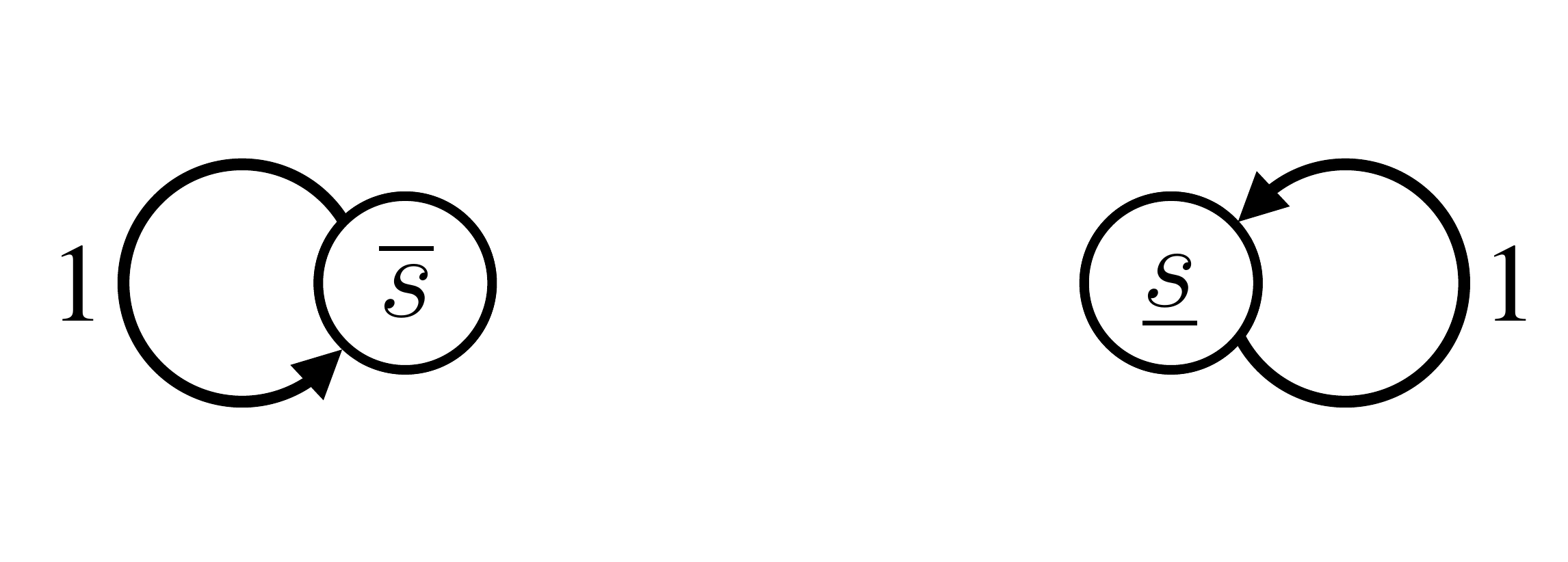}
		\end{minipage}
	\end{figure}
	\vspace{-0.2cm}
Suppose that the behavior policy $\overline \pi $ initiates at either one of the states with probability $1/2$, and the target policy $\pi$ always initiates at state $\overline{s}$. We can see that 
	\[ \Sigma = \bigg[ \begin{array}{cc} z^2-z+\frac{1}{2} & z(1-z) \\ z(1-z) & z^2-z+\frac{1}{2} \end{array} \bigg], \]
	and $\nu_0^{\pi} = \nu_1^{\pi} = \ldots = \nu_{H-1}^{\pi} = [1,0]^{\top}$.
	For $z \in [\frac{1}{4}, \frac{3}{4}]$, the distributional mismatch term controlling the lower bound becomes 
$$
\Theta(H^2) \sqrt{1 + \frac{1}{(2z-1)^2}} ,
$$
where $z$ quantifies how much one can tell apart the two states under the target policy $\pi$ using data generated by $\bar \pi$.
		When $z \approx 1/2$, one can not distinguish $\bar s$ and $\underline{s}$ from data generated by $\bar \pi$, where the lower bound becomes unbounded.


	
	

\def\bx{\mathbf{x}}
\subsection{Proof Outline}

 We start with an arbitrary MDP $M$ with transition kernel $p$ that satisfies the assumption. We will construct a perturbed instance $\widetilde p = p +\Delta p$ so that the two transition models are similar but have a gap in their policy values, denoted by $v^{\pi}$ and $\widetilde v^{\pi}$. 

Construct the perturbation $\Delta p$ such that $\Delta p(s' \, | \, s,a)\geq0$ if $s'\in \overline{\cS}$, $\Delta p(s' \, | \, s,a)\leq0$ if $s'\in \underline{\cS}$ and $\Delta p(s' \, | \, s,a)=0$ elsewhere.
In particular, we construct the perturbation as
\begin{equation}
	\Delta p(s' \, | \, s,a)  = \phi(s,a)^{\top} \Delta q(s'), \qquad \hbox{where  }\Delta q(s') := {\bf x} \cdot \min_{s \in \mathcal{S}} p^{\overline{\pi}} ( s' \, | \, s) \cdot \big( \underline{p} \mathbbm{1}_{\overline{\mathcal{S}}}(s') - \overline{p} \mathbbm{1}_{\underline{\mathcal{S}}}(s') \big),
\end{equation}
where $ \overline{p}$ and $\underline{p}$ are picked such that $\int_{\mathcal{S}} \Delta p(s' \, | \, s,a) {\rm d}s' =0$ for any $s,a$, $\bx$ is a vector to be picked later.

	\paragraph{Reduction to likelihood test}
	We define likelihood functions $\mathcal{L}(\mathcal{D})$ and $\widetilde{\mathcal{L}}(\mathcal{D})$ of transition kernels $p$ and $\widetilde{p}$. The likelihood ratio $\frac{\widetilde{\mathcal{L}}(\mathcal{D})}{\mathcal{L}(\mathcal{D})} = \prod_{n=1}^N \frac{\widetilde{p}(s_n' \, | \, s_n,a_n)}{p(s_n' \, | \, s_n,a_n)}$ reflects how likely the observation $\mathcal{D}$ comes from model $\widetilde{p}$ rather than $p$. When $p \approx \widetilde{p}$, with high probability, the dataset $\mathcal{D}$ generated by model $p$ has a relatively large likelihood ratio, so that it is hard to distinguish $p$ and $\widetilde{p}$ based on observation $\mathcal{D}$.
	We prove by a martingale concentration argument that, when $N$ is sufficiently large, 
	$$
		\ln \bigg( \frac{\widetilde{\mathcal{L}}(\mathcal{D})}{\mathcal{L}(\mathcal{D})} \bigg) \gtrsim - \sqrt{N} \sqrt{{\bf x}^{\top} \Sigma {\bf x}} - N \cdot {\bf x}^{\top} \Sigma {\bf x}
	$$
	with high probability.
	In particular, we have 
	\vspace{-0.2cm} \begin{equation} \label{likelihood} \mathbb{P} \bigg( \frac{\widetilde{\mathcal{L}}(\mathcal{D})}{\mathcal{L}(\mathcal{D})} \geq \frac{1}{2} \bigg) \geq \frac{1}{2}. \end{equation}
	when $\sqrt{{\bf x}^{\top} \Sigma {\bf x}} \lesssim N^{-1/2}$.
	If we further have $|v^{\pi} - \widetilde{v}^{\pi}| \geq \rho + \widetilde{\rho}$ for some constant gaps $\rho, \widetilde{\rho} \geq 0$, condition \eqref{likelihood} implies that for an arbitrary algorithm $\widehat{v}^{\pi}$, only one of the following must hold: either $\mathbb{P}\big( |v^{\pi} - \widehat{v}^{\pi}(\mathcal{D})| \geq \rho \big) \geq \frac{1}{6}$ or $\widetilde{\mathbb{P}}\big( |\widetilde{v}^{\pi} - \widehat{v}^{\pi}(\mathcal{D})| \geq \widetilde{\rho} \big) \geq \frac{1}{6}$. In other words, no algorithm can achieve small OPE error for both $p$ and $\widetilde p$.
	
	\paragraph{Constructing similar instances with a gap in values}
	
	We have
	\vspace{-0.27cm}
		\begin{equation}
		 \widetilde{v}^{\pi} - v^{\pi}  = \sum^H_{h=0} \xi_0^{\top} (\widetilde \cP^{\pi})^h (\widetilde \cP^{\pi}-\cP^{\pi})Q_{h+1}^{\pi}.  \end{equation}
		By first-order Taylor expansion and our construction, if the perturbation $\Delta p$ is sufficiently small, we have  
		\begin{equation} \label{v-v'}
		 \widetilde{v}^{\pi} - v^{\pi}  \approx \sum^H_{h=0} \xi_0^{\top} (\cP^{\pi})^h (\widetilde \cP^{\pi}-\cP^{\pi})Q_{h+1}^{\pi}\gtrsim \sum_{h=0}^{H-1} (H-h) (\nu_h^{\pi})^{\top} {\bf x}. \end{equation}	
	For a given $N$, we maximize the above value over $\bx$ under the constraint $\sqrt{{\bf x}^{\top} \Sigma {\bf x}} \lesssim N^{-1/2}$. Then we obtain ${\bf x}^* = \frac{c_0\bx_0}{\sqrt{N}\sqrt{\bx_0^{\top} \Sigma \bx_0}}$ where $c_0 > 0$ is a constant and $\bx_0= \Sigma^{-1} \sum_{h=0}^{H-1} (H-h) \nu_h^{\pi}$.
	In this way, we have shown that $ \widetilde v^{\pi} -{v}^{\pi} \gtrsim \frac{1}{\sqrt{N}} \big\| \sum_{h=0}^{H-1} (H-h) \nu_h^{\pi} \big\|_{\Sigma^{-1}}^2$ using the above construction of $\bx^*$. 
	
Similarly, one can show that for $N$ sufficiently large,
	$v^{\pi} - \widetilde{v}^{\pi} \geq \rho + \widetilde{\rho}$ for $\rho = \frac{\sqrt{c}}{24\sqrt{N}}\big\| \sum_{h=0}^{H-1} (H-h) \nu_h^{\pi} \big\|_{\Sigma^{-1}}^2$ and $\widetilde{\rho} = \frac{\sqrt{c}}{24\sqrt{N}}\big\| \sum_{h=0}^{H-1} (H-h) \widetilde{\nu}_h^{\pi} \big\|_{\widetilde{\Sigma}^{-1}}^2$, where $\widetilde \nu_h^{\pi}$ and $\widetilde \Sigma$ are counterparts of $\nu^{\pi}$ and $\Sigma$ under the perturbed model $\widetilde p$. Finally, we apply the result of the likelihood test and complete the proof.
	\vspace{-0.77cm}
	\begin{flushright}
		$\square$
	\end{flushright}



\section{A Computable Confidence Bound}

Next we study how to quantify the uncertainty in the policy evaluator given by Algorithm \ref{FQI}. In this section, we assume that the dataset is an arbitrary set of experiences, not necessarily independent episodes. We only assume that the transition samples $\mathcal{D} = \{ (s_n,a_n,s_n',r_n') \}_{n=1,\ldots,N}$ are collected in time order. 

\begin{assumption} \label{filtration}
	The dataset $\mathcal{D}$ consists of sample transitions $\{ (s_t,a_t,s_t',r_t') \}_{t=1}^N$ generated in time order, {\it i.e.} adapted to a filtration $\{ \mathcal{F}_t \}_{t=1}^N$, where $\{ (s_{\tau}, a_{\tau}, s_{\tau}', r_{\tau}') \}_{\tau = 1}^t$ are $\mathcal{F}_t$-measurable. 
\end{assumption}
Assumption \ref{filtration} is much weaker than Assumption \ref{episode}. It allows the samples to be generated from a long single path possibly under a nonstationary adaptive policy, as is typical in online reinforcement learning. 

Under this mildest assumption, we provide a confidence bound for the policy evaluation error $|v^{\pi} - \widehat{v}^{\pi}|$, which can be analytically computed from the data $\mathcal{D}$.  \vspace{-0.1cm}

\begin{theorem}[{\bf Computable confidence bound}] \label{UpperBound'} Let Assumptions \ref{Q_class} and \ref{filtration} hold.
Let $\omega :=  \max \big\{ \|w\|_2 \, \big| \, 0 \leq \phi(s,a)^{\top} w \leq 1, \forall (s,a) \in \mathcal{X} \big\}$.\footnote{Such $\omega$ always exists and can be computed priorly.} 
Assume $\|\phi(s,a)\|_2 \leq 1$ for any $(s,a) \in \mathcal{X}$. For a target policy $\pi$, with probability at least $1-\delta$, we have\vspace{-0.25cm}
	\begin{equation} \label{eqn_UpperBound'} \begin{aligned} & \big| v^{\pi} - \widehat{v}^{\pi} \big| \leq \sum_{h=0}^{H} (H-h+1) \sqrt{(\widehat{\nu}_h^{\pi})^{\top} \widehat{\Sigma}^{-1} \widehat{\nu}_h^{\pi}}  \cdot \bigg(\sqrt{2\lambda} \omega + 2 \sqrt{2d \ln\Big( 1 + \frac{N}{\lambda d} \Big)  \ln\Big(\frac{3N^2H}{\delta}\Big)} + \frac{4}{3}\ln\Big(\frac{3N^2H}{\delta}\Big) \bigg), \end{aligned} \vspace{-0.3cm}
	\end{equation}
	
	where $\widehat{\nu}_h^{\pi}$ is given by $(\widehat{\nu}_h^{\pi})^{\top} := (\nu_0^{\pi})^{\top}(\widehat{M}^{\pi})^h$.
\end{theorem}

The proof begins with a decomposition of error given by 
$v^{\pi} - \widehat{v}^{\pi} = \sum_{h=0}^{H} (\widehat{\nu}_h^{\pi})^{\top} \big( w_h^{\pi} - (\widehat{R} + \widehat{M}^{\pi} w_{h+1}^{\pi}) \big)$, from which we derive \vspace{-0.1cm}
\begin{equation} \label{upperbound'} \big| v^{\pi} - \widehat{v}^{\pi} \big| \leq \sum_{h=0}^{H} \sqrt{(\widehat{\nu}_h^{\pi})^{\top} \widehat{\Sigma}^{-1} \widehat{\nu}_h^{\pi}} \cdot \big\| \widehat{\Sigma}^{1/2} \big( w_{h}^{\pi} - (\widehat{R} + \widehat{M}^{\pi} w_{h+1}^{\pi}) \big) \big\|_2. \end{equation}
We analyze the concentration of $\Theta_h := \big\| \widehat{\Sigma}^{1/2} \big( w_h^{\pi} - (\widehat{R} + \widehat{M}^{\pi} w_{h+1}^{\pi}) \big) \big\|_2^2$ using a martingale argument that is similar to the bandit literature (e.g., proof of Theorem 5 in \cite{dani2008stochastic}). The complete proof is given in Appendix \ref{appendix:lemma:UpperBound'}. 

The confidence bound given in Theorem \ref{UpperBound'} can be easily calculated as a byproduct of FQI-OPE (Algorithm \ref{FQI}), since $\widehat \nu_h^{\pi},\widehat\Sigma$ were already computed in the iterations. In practice, one can tune the value of $\lambda$ to get the smallest possible confidence bound.

\section{Extension to Infinite-Horizon Discounted MDP}

Our analysis can be extended to the infinite-horizon discounted MDP where the value of policy $\pi$ is defined as 
$$v^{\pi} := \mathbb{E}^{\pi}\Bigg[\sum^{\infty}_{h=0} \gamma^h r(s_h,a_h)  \, \Bigg| \, s_0\sim \xi_0 \Bigg],$$
where $\gamma\in(0,1)$ is a discount factor.
In this case, we can estimate the Q function by letting 
$$\widehat{Q}^{\pi}(\cdot) = \phi(\cdot)^{\top} \widehat{w}^{\pi}, $$
where 
\[ \widehat{w}^{\pi} := \big( I - \gamma \widehat{M}^{\pi} \big)^{-1} \widehat{R}. \]
We still assume that the data are collected episodically as in Assumption \ref{episode}.

Finally, we establish the minimax-optimal OPE error bound for discounted MDP. Its proof is similar to the proof in the finite-horizon case, and is deferred to Appendix \ref{appendix:proof:discountMDP}.

\begin{theorem}[{\bf Minimax-optimal error bounds for discounted MDP}] \label{thm:UpperBound_gamma}
\. \\
	1. (Finite-sample upper bound) Suppose Assumptions \ref{Q_class} and \ref{episode} hold, $\phi(s,a)^{\top} \Sigma^{-1} \phi(s,a) \leq C_1d$ for any $(s,a) \in \mathcal{X}$ and $H \leq (1-\gamma)^{-1}$ for data collection. Let $\delta \in (0,1)$. If $N \geq 80C_1\kappa_1(2\!+\!\kappa_2)^2 \cdot \frac{ \gamma^2 \ln(12d/\delta)d}{(1-\gamma)^3}$ and $\lambda \leq  \sigma_{\min}(\Sigma) \cdot \frac{\ln(12d/\delta) C_1d}{1-\gamma}$, then with probability at least $1-\delta$,
	\begin{equation} \label{UpperBound_gamma} \big| v^{\pi} \!-\! \widehat{v}^{\pi} \big| \leq \frac{1}{1 - \gamma} \!\cdot\! \sup_{f \in \mathcal{Q}} \frac{\mathbb{E}^{\pi} \big[ \sum_{h=0}^{\infty} \! \gamma^h f(s_h,a_h) \, \big| \, s_0 \!\sim\! \xi_0 \big]}{\sqrt{\mathbb{E}\big[ \frac{1}{H} \sum_{h=0}^{H-1} f^2(s_{1,h},a_{1,h}) \big]}} \cdot \sqrt{\frac{\ln(12/\delta)}{2N}} + \frac{ \gamma C\ln(12d/\delta)d}{N (1-\gamma)^{3.5}}, \end{equation}
	where $\kappa_1$, $\kappa_2$ and $C$ are parameters defined in Theorem \ref{UpperBound}.
	
	2. (Minimax lower bound) Suppose that an MDP instance $M=(p,r)$ satisfies:
	\begin{itemize}
		\item There exists a set of high-value states $\overline \cS\subseteq \cS$ and a set of low-value states $\underline \cS\subseteq \cS$ under the target policy $\pi$ such that $V^{\pi}(s) \!\geq\! \frac3{4(1-\gamma)}$ if $s\in \overline \cS$ and $V^{\pi}(s) \!\leq\! \frac1{4(1-\gamma)}$ if $s\in \underline \cS$;
		\item $\overline{p} := \int_{\overline{\mathcal{S}}} \min_{s \in \mathcal{S}} p^{\overline{\pi}}(s' \, |\,s ) {\rm d}s' \geq c $ and $\underline{p} := \int_{\underline{\mathcal{S}}} \min_{s \in \mathcal{S}} p^{\overline{\pi}}(s' \, | \, s) {\rm d}s' \geq c$ for $c>0$.
	\end{itemize}
	For any behavior policy $\overline{\pi}$, when $N$ is sufficiently large,  one has \vspace{-0.1cm}
	\begin{equation} \label{lowerbound_gamma} \begin{aligned} & \inf_{\widehat{v}^{\pi}} \sup_{M'\in \mathcal{N}(M)} \mathbb{P}_{M'} \left( \big| v^{\pi} - \widehat{v}^{\pi}(\mathcal{D}) \big| \geq \frac{\sqrt{c}}{24\sqrt{N}} \cdot \frac{\gamma}{1-\gamma} \cdot \sup_{f \in \mathcal{Q}} \frac{\mathbb{E}^{\pi} \big[ \sum_{h=0}^{\infty}\gamma^h f(s_h,a_h) \, \big| \, s_0 \sim \xi_0 \big]}{\sqrt{\mathbb{E}\big[ \frac{1}{H} \sum_{h=0}^{H-1} f^2(s_{1,h},a_{1,h}) \big]}} \right) \geq \frac{1}{6}, \end{aligned} \vspace{-0.1cm} \end{equation}
	where $\mathcal{N}(M)$ is a small neighborhood of $M$ defined in Theorem \ref{theorem:lb}.
	
	3. (Computable confidence bound) With probability $1-\delta$, \begin{equation} \label{UpperBound'_gamma} |v^{\pi} - \widehat{v}^{\pi}| \leq \frac{1}{1-\gamma} \sqrt{\Bigg(\sum_{h=0}^{\infty} \gamma^h\widehat{\nu}_h^{\pi} \Bigg)^{\top} \widehat{\Sigma}^{-1} \Bigg(\sum_{h=0}^{\infty} \gamma^h\widehat{\nu}_h^{\pi} \Bigg)} \cdot \! \bigg(\! \sqrt{2\lambda} \omega \!+\! 2\sqrt{2}\sqrt{\!d \ln\!\Big(\! 1 \!+\! \frac{n}{\lambda d} \Big)\!\ln(2N^2\!/\delta)} \! + \! \frac{4}{3}\ln(2N^2\!/\delta) \! \bigg). \end{equation}
	
	In particular, when the spectral radius $\rho(\widehat{M}^{\pi}) < \gamma^{-1}$, $\sum_{h=0}^{\infty} \gamma^h\widehat{\nu}_h^{\pi} = (I - \gamma \widehat{M}^{\pi})^{-\top} \nu_0^{\pi}$.
\end{theorem}

{\bf Remark:} Denote {$\mu^{\pi} := (1-\gamma)\mathbb{E}^{\pi}\big[ \sum^{\infty}_{t=0} \gamma^t \mathbf{1}(s_t,a_t) \big]$} as the normalized cumulative discounted occupancy measure (also known as flux) under policy $\pi$. Theorem \ref{thm:UpperBound_gamma} shows that
$$|\widehat v^{\pi} - v^{\pi}| \asymp \frac1{(1-\gamma)^2} \sqrt{\frac{1+\chi^2_{\mathcal{Q}}(\mu^{\pi}, \bar\mu)}{N}} + o(N^{-1/2}).$$

\section{Summary}

This paper studies the statistical limits of off-policy evaluation using linear function approximation. We establish a minimax error lower bound that depends on a function class-restricted $\chi^2$-divergence between data and the target policy. We prove that a regression-based FQI method, which is equivalent to a plug-in estimator, nearly achieves the minimax lower bound.
We also provide a computable confidence bound as a byproduct of the algorithm.


\bibliography{ref}

\begin{thebibliography}{20}
\providecommand{\natexlab}[1]{#1}
\providecommand{\url}[1]{\texttt{#1}}
\expandafter\ifx\csname urlstyle\endcsname\relax
  \providecommand{\doi}[1]{doi: #1}\else
  \providecommand{\doi}{doi: \begingroup \urlstyle{rm}\Url}\fi

\bibitem[Bertsekas et~al.(1995)Bertsekas, Bertsekas, Bertsekas, and
  Bertsekas]{bertsekas1995dynamic}
Bertsekas, D.~P., Bertsekas, D.~P., Bertsekas, D.~P., and Bertsekas, D.~P.
\newblock \emph{Dynamic programming and optimal control}, volume~1.
\newblock Athena scientific Belmont, MA, 1995.

\bibitem[Dani et~al.(2008)Dani, Hayes, and Kakade]{dani2008stochastic}
Dani, V., Hayes, T.~P., and Kakade, S.~M.
\newblock Stochastic linear optimization under bandit feedback.
\newblock 2008.

\bibitem[Dann et~al.(2019)Dann, Li, Wei, and Brunskill]{dann2018policy}
Dann, C., Li, L., Wei, W., and Brunskill, E.
\newblock Policy certificates: Towards accountable reinforcement learning.
\newblock In \emph{International Conference on Machine Learning}, 2019.

\bibitem[Fonteneau et~al.(2013)Fonteneau, Murphy, Wehenkel, and
  Ernst]{fonteneau2013batch}
Fonteneau, R., Murphy, S.~A., Wehenkel, L., and Ernst, D.
\newblock Batch mode reinforcement learning based on the synthesis of
  artificial trajectories.
\newblock \emph{Annals of operations research}, 208\penalty0 (1):\penalty0
  383--416, 2013.

\bibitem[Freedman(1975)]{freedman1975tail}
Freedman, D.~A.
\newblock On tail probabilities for martingales.
\newblock \emph{The Annals of Probability}, 3\penalty0 (1):\penalty0 100--118,
  1975.

\bibitem[Grunewalder et~al.(2012)Grunewalder, Lever, Baldassarre, Pontil, and
  Gretton]{grunewalder2012modelling}
Grunewalder, S., Lever, G., Baldassarre, L., Pontil, M., and Gretton, A.
\newblock Modelling transition dynamics in mdps with rkhs embeddings.
\newblock 2012.

\bibitem[Jiang \& Li(2016)Jiang and Li]{jiang2015doubly}
Jiang, N. and Li, L.
\newblock Doubly robust off-policy value evaluation for reinforcement learning.
\newblock In \emph{International Conference on Machine Learning}, 2016.

\bibitem[Jong \& Stone(2007)Jong and Stone]{jong2007model}
Jong, N.~K. and Stone, P.
\newblock Model-based function approximation in reinforcement learning.
\newblock In \emph{Proceedings of the 6th international joint conference on
  Autonomous agents and multiagent systems}, pp.\  1--8, 2007.

\bibitem[Lagoudakis \& Parr(2003)Lagoudakis and Parr]{lagoudakis2003least}
Lagoudakis, M.~G. and Parr, R.
\newblock Least-squares policy iteration.
\newblock \emph{Journal of machine learning research}, 4\penalty0
  (Dec):\penalty0 1107--1149, 2003.

\bibitem[Liu et~al.(2018)Liu, Li, Tang, and Zhou]{liu2018breaking}
Liu, Q., Li, L., Tang, Z., and Zhou, D.
\newblock Breaking the curse of horizon: Infinite-horizon off-policy
  estimation.
\newblock In \emph{Advances in Neural Information Processing Systems}, pp.\
  5356--5366, 2018.

\bibitem[Liu et~al.(2019)Liu, Swaminathan, Agarwal, and Brunskill]{liu2019off}
Liu, Y., Swaminathan, A., Agarwal, A., and Brunskill, E.
\newblock Off-policy policy gradient with state distribution correction.
\newblock In \emph{Conference on Uncertainty in Artificial Intelligence}, 2019.

\bibitem[Mannor et~al.(2004)Mannor, Simester, Sun, and
  Tsitsiklis]{mannor2004bias}
Mannor, S., Simester, D., Sun, P., and Tsitsiklis, J.~N.
\newblock Bias and variance in value function estimation.
\newblock In \emph{Proceedings of the twenty-first international conference on
  Machine learning}, pp.\ ~72, 2004.

\bibitem[Nachum et~al.(2019)Nachum, Chow, Dai, and Li]{nachum2019dualdice}
Nachum, O., Chow, Y., Dai, B., and Li, L.
\newblock {DualDICE}: Behavior-agnostic estimation of discounted stationary
  distribution corrections.
\newblock In \emph{Advances in Neural Information Processing Systems 32}. 2019.

\bibitem[Precup(2000)]{precup2000eligibility}
Precup, D.
\newblock Eligibility traces for off-policy policy evaluation.
\newblock \emph{Computer Science Department Faculty Publication Series}, pp.\
  ~80, 2000.

\bibitem[Sutton \& Barto(2018)Sutton and Barto]{sutton2018reinforcement}
Sutton, R.~S. and Barto, A.~G.
\newblock \emph{Reinforcement learning: An introduction}.
\newblock MIT press, 2018.

\bibitem[Thomas \& Brunskill(2016)Thomas and Brunskill]{thomas2016data}
Thomas, P. and Brunskill, E.
\newblock Data-efficient off-policy policy evaluation for reinforcement
  learning.
\newblock In \emph{International Conference on Machine Learning}, 2016.

\bibitem[Tropp et~al.(2011)]{tropp2011freedman}
Tropp, J. et~al.
\newblock Freedman's inequality for matrix martingales.
\newblock \emph{Electronic Communications in Probability}, 16:\penalty0
  262--270, 2011.

\bibitem[Xie et~al.(2019)Xie, Ma, and Wang]{xie2019towards}
Xie, T., Ma, Y., and Wang, Y.-X.
\newblock Towards optimal off-policy evaluation for reinforcement learning with
  marginalized importance sampling.
\newblock In \emph{Advances in Neural Information Processing Systems}, pp.\
  9665--9675, 2019.

\bibitem[Yang \& Wang(2019)Yang and Wang]{yang2019reinforcement}
Yang, L.~F. and Wang, M.
\newblock Reinforcement leaning in feature space: Matrix bandit, kernels, and
  regret bound.
\newblock \emph{arXiv preprint arXiv:1905.10389}, 2019.

\bibitem[Yin \& Wang(2020)Yin and Wang]{yin2020asymptotically}
Yin, M. and Wang, Y.-X.
\newblock Asymptotically efficient off-policy evaluation for tabular
  reinforcement learning.
\newblock \emph{arXiv preprint arXiv:2001.10742}, 2020.

\end{thebibliography}
\bibliographystyle{icml2020}

\newpage

\onecolumn

\appendix

\addtocontents{toc}{\protect\setcounter{tocdepth}{3}}
~ \vspace{-2em} 
\begin{center} \bf \LARGE Appendices \end{center}



\numberwithin{equation}{section}	


\part{}

\section{Discussions in Section \ref{section:estimator}} \label{section:equivalence}
\subsection{Proof of Theorem \ref{theorem:equivalence}}

\begin{proof}[Proof of Theorem \ref{theorem:equivalence}]
		{Suppose we are provided with $\widehat{Q}_{h+1}^{\pi}(\cdot)$ at the beginning of an iteration, and $\widehat{Q}_{h+1}^{\pi}(\cdot) = \phi(\cdot)^{\top} \widehat{w}_{h+1}^{\pi}$ for some vector $\widehat{w}_{h+1}^{\pi} \in \mathbb{R}^d$.
		In FQI \eqref{Q_recursion}, we replace $g(\cdot)$ by $\phi(\cdot)^{\top}w$ and obtain $\widehat{Q}_h^{\pi} = \phi(\cdot)^{\top} w^*$, where
		\begin{equation*}
			\begin{aligned}
				w^* := \arg\min_{{\bf x} \in \mathbb{R}^{d}} \Bigg\{ \sum_{n=1}^N \Big( \phi(s_n,a_n)^{\top} {\bf x} - r_n' - \phi^{\pi}(s_n')^{\top} \widehat{w}_{h+1}^{\pi} \Big)^2 +\lambda \| {\bf x} \|_2^2 \Bigg\} = \widehat{\Sigma}^{-1} \sum_{n=1}^N \phi(s_n,a_n) \Big( r_n' + \phi^{\pi}(s_n')^{\top} \widehat{w}_{h+1}^{\pi} \Big).
			\end{aligned}
		\end{equation*}
		Recalling the definitions of $\widehat{R}$ and $\widehat{M}^{\pi}$ in \eqref{vector_recursion}, we have $w^* = \widehat{R} + \widehat{M}^{\pi} \widehat{w}_{h+1}^{\pi}$. Since $\widehat{r}(\cdot) = \phi(\cdot)^{\top} \widehat{R}$ and $\widehat{\mathcal{P}}^{\pi} \widehat{Q}_{h+1}^{\pi}(\cdot) = \phi(\cdot)^{\top} \widehat{M}^{\pi} \widehat{w}_{h+1}^{\pi}$ according to \eqref{vector_recursion} and \eqref{hatPpi},
		it holds that $\widehat{Q}_h^{\pi} = \widehat{r} + \widehat{\mathcal{P}}^{\pi} \widehat{Q}_{h+1}^{\pi}$.
		These two algorithms therefore output the same $\widehat{Q}_{h}^{\pi}$ based on the same $\widehat{Q}_{h+1}^{\pi}$. It follows that $\widehat{v}_{\mathsf{FQI}}^{\pi} = \widehat{v}_{\mathsf{Plug\text{-}in}}^{\pi}$.}
\end{proof}
{{\bf Remark:} Theorem \ref{theorem:equivalence} concerns the linearity of regression. We restrict $\mathcal{Q}$ to be finite-dimensional in this proof only for notational simplicity. The result can also apply to an infinite-dimensional linear space $\mathcal{Q}$.}

\subsection{Relations to Other Methods}
%
%
%
%

\subsubsection*{Marginalized importance sampling (MIS)}
Algorithm \ref{Model-based} suggests that $\widehat{v}^{\pi} = \sum_{h=0}^H (\widehat{\nu}_h^{\pi})^{\top} \widehat{R}$. Subsitituting $\widehat{R}$ with its definition in \eqref{vector_recursion} yields
\[ \widehat{v}^{\pi} = \frac{1}{N} \sum_{n=1}^N \widehat{w}_{\pi/\mathcal{D}}(s_n,a_n) r_n', \qquad \text{where $\widehat{w}_{\pi/\mathcal{D}}(s,a) := N \sum_{h=0}^H (\widehat{\nu}_h^{\pi})^{\top} \widehat{\Sigma}^{-1} \phi(s,a)$}. \]
In this way, we can interpret our algorithm as an importance sampling method with importance weight $\widehat{w}_{\pi/\mathcal{D}}(s,a)$.

In tabular case, if $\lambda = 0$, then the importance weight \begin{equation} \label{MIS} \widehat{w}_{\pi/\mathcal{D}}(s,a) = \frac{\sum_{h=0}^H \widehat{\xi}_h^{\pi}(s) \pi(a \, | \, s)}{\frac{1}{N} \sum_{n=1}^N {\bf 1}((s_n,a_n) = (s,a))}, \end{equation}
where $\widehat{\xi}_h^{\pi}$ is the marginal distribution of $s_h$ under policy $\pi$, initial distribution $\xi_0$ and the empirical transition kernel $\widehat{\mathbb{P}}(s' \, | \, s,a) := \frac{\sum_{n=1}^N {\bf 1}(s_n=s, a_n=a, s_n'=s')}{\sum_{n=1}^N {\bf 1}(s_n=s, a_n=a)}$. In this special case, our estimator reduces to the marginalized importance sampling method (MIS) in \cite{yin2020asymptotically}.

\subsubsection*{DualDICE}
Consider an infinite-horizon MDP with discounted factor $\gamma \in (0,1)$. 
In this case, FQI-OPE estimator has an equivalent form
\[ \widehat{v}^{\pi} = \frac{1}{N} \sum_{n=1}^N \widehat{w}_{\pi/\mathcal{D}}(s,a) r_n', \qquad \widehat{w}_{\pi/\mathcal{D}}(s,a) := N \sum_{h=0}^{\infty} \gamma^h (\widehat{\nu}_h^{\pi})^{\top} \widehat{\Sigma}^{-1} \phi(s,a) = N (\nu_0^{\pi})^{\top} \big(I - \gamma \widehat{M}^{\pi}\big)^{-1} \widehat{\Sigma}^{-1} \phi(s,a). \]
In the following, we will show that FQI-OPE is equivalent to DualDICE algorithm \cite{nachum2019dualdice} when the parameterization families are properly chosen. DualDICE algorithm solves the following minimax saddle-point optimization problem:
\begin{equation} \label{DualDICE} \begin{aligned} \min_{f\in \mathbb{R}^{\mathcal{X}}} \max_{g \in \mathbb{R}^{\mathcal{X}}} J(f,g) := \frac{1}{N} \sum_{n=1}^N \Bigg( \bigg( f(s_n,a_n) & - \gamma \int_{\mathcal{A}} f(s_n',a') \pi(a \, | \, s_n') {\rm d}a' \bigg) g(s_n,a_n) - \frac{1}{2} g(s_n,a_n)^2 \Bigg) \\ & \qquad \qquad \qquad - \mathbb{E} \big[ f(s_0,a_0) \, \big| \, s_0 \sim \xi_0, a_0 \sim \pi(\cdot \, | \, s_0) \big]. \end{aligned} \end{equation}
The solution $g^*$ serves as the {\it discounted stationary distribution correction}. One can estimate $v^{\pi}$ by
\begin{equation} \label{v_DualDICE} \widehat{v}_{\mathsf{DualDICE}}^{\pi} := \frac{1}{N} \sum_{n=1}^N g^*(s_n,a_n) r_n'. \end{equation}
We have the following equivalence result.
\begin{theorem}[Equivalence between FQI-OPE and DualDICE]
We take $f, g \in \mathcal{Q}$ in the optimization problem \eqref{DualDICE}. Then $\widehat{v}_{\mathsf{DualDICE}}^{\pi} = \widehat{v}^{\pi}$, where $\widehat{v}^{\pi}$ is an estimator provided by FQI-OPE with $\lambda = 0$.
\end{theorem}

\begin{proof}
We substitute $f$ and $g$ in $J(f,g)$ by $f(\cdot) = \phi(\cdot)^{\top} {\bf x}$ and $g(\cdot) = \phi(\cdot)^{\top} {\bf y}$, respectively, and obtain
\[ \begin{aligned} J(f,g) = & \frac{1}{N} \sum_{n=1}^N \bigg( \Big( {\bf x}^{\top}\phi(s_n,a_n) - \gamma {\bf x}^{\top} \phi^{\pi}(s_n') \Big) \phi(s_n,a_n)^{\top} {\bf y} - \frac{1}{2} \big(\phi(s_n,a_n)^{\top} {\bf y} \big)^2 \bigg) - (\nu_0^{\pi})^{\top} {\bf x} \\ = & \frac{1}{N} \Big( {\bf x}^{\top} \widehat{\Sigma} {\bf y} - \gamma {\bf x}^{\top} (\widehat{M}^{\pi})^{\top} \widehat{\Sigma} {\bf y} - \frac{1}{2} {\bf y}^{\top} \widehat{\Sigma} {\bf y} \Big) - (\nu_0^{\pi})^{\top} {\bf x} = - \frac{1}{2N} {\bf y}^{\top} \widehat{\Sigma} {\bf y} + \frac{1}{N} {\bf x}^{\top} \big( I - \gamma \widehat{M}^{\pi} \big)^{\top} \widehat{\Sigma} {\bf y} - (\nu_0^{\pi})^{\top} {\bf x}, \end{aligned} \]
where we have used the relations $\sum_{n=1}^N \phi(s_n, a_n) \phi(s_n, a_n)^{\top} = \widehat{\Sigma}$ and $\sum_{n=1}^N \phi(s_n,a_n) \phi^{\pi}(s_n')^{\top} = \widehat{\Sigma} \widehat{M}^{\pi}$. The optimization problem $\min_{{\bf x} \in \mathbb{R}^d} \max_{{\bf y} \in \mathbb{R}^d} J(f,g)$ has the solution
\[ {\bf x}^* := N \big( I - \gamma \widehat{M}^{\pi} \big)^{-1} \widehat{\Sigma}^{-1} \big( I - \gamma \widehat{M}^{\pi} \big)^{-\top} \nu_0^{\pi} \qquad \text{and} \qquad {\bf y}^* := N \widehat{\Sigma}^{-1} \big( I - \gamma \widehat{M}^{\pi} \big)^{-\top} \nu_0^{\pi}, \]
{\it i.e.}, $g^*(s,a) = N (\nu_0^{\pi})^{\top} \big( I - \gamma \widehat{M}^{\pi} \big)^{-1} \widehat{\Sigma}^{-1} \phi(s,a)$. To this end, $\widehat{v}_{\mathsf{DualDICE}}^{\pi} = (\nu_0^{\pi})^{\top} \big( I - \gamma \widehat{M}^{\pi} \big)^{-1} \widehat{\Sigma}^{-1} \sum_{n=1}^N r_n' \phi(s_n,a_n) = (\nu_0^{\pi})^{\top} \big( I - \gamma \widehat{M}^{\pi} \big)^{-1} \widehat{R} = \widehat{v}^{\pi}$, which finishes the proof.
\end{proof}

\section{Proof of Finite-Sample Upper Bound} \label{appendix:proof:UpperBound}


\subsection{Preliminaries} \label{appendix:theorem:Theorem}



\subsubsection*{Contraction of Markov chains}

In order to control the estimation errors in the powers of $\widehat{\mathcal{P}}^{\pi}$, we need to leverage the contraction property of a Markov process. In particular, under Assumption \ref{Q_class}, we are only concerned about a low-dimensional embedding of $\mathcal{P}^{\pi}$. Let $M^{\pi} \in \mathbb{R}^{d \times d}$ be a population counterpart to $\widehat{M}^{\pi}$ in \eqref{vector_recursion}, {\it i.e.} $M^{\pi} \in \mathbb{R}^{d \times d}$ is the matrix that satisfies $\mathbb{E}\big[ \phi^{\pi}(s')^{\top} \, \big| \, s,a \big] = \phi(s,a)^{\top} M^{\pi}$ for any $(s,a) \in \mathcal{X}$. 
By properties of the Markov process, the spectral radius of $M^{\pi}$ is at most $1$, therefore $M^{\pi}$ is nonexpansive with respect to some matrix norm. In particular, we provide the following Lemma \ref{SMS} about the nonexpansiveness. Its proof is defered to Appendix \ref{appendix:SMS}.

\begin{lemma}[Contraction of Markov chain] \label{SMS}
	Suppose $({s}_0,{s}_1,\ldots)$ is a general Markov chain defined on $\mathcal{S}$ with transition kernel ${p}(s' \, | \, s)$ and some initial distribution ${\xi}_0(s)$. Assume that for a feature mapping ${\psi}: \mathcal{S} \rightarrow \mathbb{R}^d$, there exists a matrix $M \in \mathbb{R}^{d \times d}$ such that
	\[ {\mathbb{E}}\big[ {\psi}(s')^{\top} \, \big| \, s \big] : = \int_{\mathcal{S}} {\psi}(s')^{\top} {p}(s' \, | \, s) {\rm d}s' = {\psi}(s)^{\top} M, \qquad \forall s \in \mathcal{S}. \]
	Take ${\Sigma}_t := {\mathbb{E}}\big[ {\psi}({s}_t) {\psi}({s}_t)^{\top} \, \big| \, {s}_0 \sim {\xi}_0 \big]$. We have
	\begin{equation} \big\| {\Sigma}_t^{1/2} M {\Sigma}_{t+1}^{-1/2} \big\|_2 \leq 1, \qquad \text{for }t = 0,1,\ldots. \end{equation}
\end{lemma}

The target policy $\pi$ defines a Markov process on $\mathcal{S}$ with transition kernel $p^{\pi}(s' \, | \, s) = \int_{\mathcal{A}} \pi(a \, | \, s) p(s' \, | \, s,a) {\rm d} a$. 
Under Assumption \ref{Q_class}, $M^{\pi}$ satisfies $\phi^{\pi}(s)^{\top} M^{\pi} \! = \! \mathbb{E}\big[ \phi^{\pi}(s')^{\top} \, \big| \, s' \!\sim\! p^{\pi}(\cdot \, | \, s) \big]$ for all $s \in \mathcal{S}$.
Suppose $\xi^{\pi}$ is an invariant distribution of $p^{\pi}$, {\it i.e.} $\xi^{\pi}(s') = \int_{\mathcal{S}} \xi^{\pi}(s) p^{\pi}(s' \, | \, s) {\rm d}s$ for any $s' \in \mathcal{S}$. Let
\begin{equation} \label{Sigmapi}
	\Sigma^{\pi} := \mathbb{E} \big[ \phi^{\pi}(s) \phi^{\pi}(s)^{\top} \, \big| \, s \sim \xi^{\pi} \big] \in \mathbb{R}^{d \times d}.
\end{equation}
Assume $\Sigma^{\pi}$ is full-rank for simplicity. We learn from Lemma \ref{SMS} that
\begin{equation} \label{SMS<1}
\big\| (\Sigma^{\pi})^{1/2} M^{\pi} (\Sigma^{\pi})^{-1/2} \big\|_2 \leq 1.
\end{equation}

In the case where the Markov decision process is time-inhomogeneous, one can instead define $\Sigma_h^{\pi} := \mathbb{E}^{\pi}\big[ \phi^{\pi}(s_h) \phi^{\pi}(s_h) \, \big| \, s_0 \sim \xi_0 \big] \in \mathbb{R}^{d \times d}$ and use the property $\big\| (\Sigma_h^{\pi})^{1/2} M^{\pi} (\Sigma_{h+1}^{\pi})^{-1/2} \big\|_2 \leq 1$ in the analysis below.

\subsubsection*{Equivalent vector-form representations}

For the convenience of our analysis, we reform the key quantities in Theorem \ref{UpperBound} with vector-form representations. See Lemma \ref{Equivalence}, of which the proof is defered to Appendix \ref{appendix: equivalence}. 

\begin{lemma} \label{Equivalence}
	We take the vector-form representation of functions under basis $\{ \phi_1, \phi_2, \ldots, \phi_d \}$. 
	Let $\nu_h^{\pi} := \mathbb{E}^{\pi} \big[ \phi(s_h,a_h) \, \big| \, s_0 \sim \xi_0 \big] = \big((M^{\pi})^{\top}\big)^h \nu_0^{\pi}$, $\Sigma^{\pi} := \mathbb{E} \big[ \phi^{\pi}(s) \phi^{\pi}(s)^{\top} \, \big| \, s \sim \xi^{\pi} \big]$. We have
		\begin{align} 
			& \sup_{f \in \mathcal{Q}} \frac{\mathbb{E}^{\pi}\big[f(s_h,a_h) \, \big| \, s_0 \sim \xi_0 \big]}{\sqrt{\mathbb{E} \big[ \frac{1}{H} \sum_{h=0}^{H-1} f^2(s_{1,h},a_{1,h})\big]}} = \sqrt{(\nu_h^{\pi})^{\top} \Sigma^{-1} \nu_h^{\pi}}, \qquad h = 0,1,\ldots,H-1; \label{nuS} \\
			& \sup_{f \in \mathcal{Q}} \frac{\mathbb{E}^{\pi}\big[\sum_{h=0}^{H} (H-h+1) f(s_h,a_h) \, \big| \, s_0 \sim \xi_0 \big]}{\sqrt{\mathbb{E}\big[ \frac{1}{H} \sum_{h=0}^{H-1} f^2(s_{1,h},a_{1,h}) \big]}} \! = \! \sqrt{ \! \Bigg( \!\sum_{h=0}^{H} \! (H-h+1)\nu_h^{\pi} \! \Bigg)\!\!\Bigg.^{\top} \! \Sigma^{-1} \Bigg( \! \sum_{h=0}^{H} \! (H-h+1)\nu_h^{\pi} \! \Bigg)}; \label{equvalence_1.2}
		\end{align}
\end{lemma}

\subsection{Error Decomposition}
According to the Bellman equation, we have 
\begin{equation} \label{Q_power} Q_0^{\pi} = \big( \mathcal{I} + \mathcal{P}^{\pi} + (\mathcal{P}^{\pi})^2 + \ldots + (\mathcal{P}^{\pi})^H \big) r \qquad \text{and} \qquad \widehat{Q}_0^{\pi} = \big( \mathcal{I} + \widehat{\mathcal{P}}^{\pi} + (\widehat{\mathcal{P}}^{\pi})^2 + \ldots + (\widehat{\mathcal{P}}^{\pi})^H \big) \widehat{r}. \end{equation} 
Note the relation
\begin{equation} \label{diff_power} \big( \mathcal{P}^{\pi} \big)^h - \big( \widehat{\mathcal{P}}^{\pi} \big)^h = \sum_{h'=1}^h \big(\widehat{\mathcal{P}}^{\pi}\big)^{h'-1}\big( \mathcal{P}^{\pi} - \widehat{\mathcal{P}}^{\pi} \big) \big(\mathcal{P}^{\pi}\big)^{h-h'}. \end{equation}
Combining \eqref{Q_power} and \eqref{diff_power}, 
we have
\begin{equation} \label{Q_decompose} \begin{aligned} Q_0^{\pi} - \widehat{Q}_0^{\pi} = & \sum_{h=0}^H \big( (\mathcal{P}^{\pi})^h - (\widehat{\mathcal{P}}^{\pi})^h \big) r + \sum_{h=0}^H (\widehat{\mathcal{P}}^{\pi})^h (r - \widehat{r}) \\ = & \sum_{h=0}^{H-1} (\widehat{\mathcal{P}}^{\pi})^h (\mathcal{P}^{\pi} - \widehat{\mathcal{P}}^{\pi}) Q_{h+1}^{\pi} + \sum_{h=0}^H (\widehat{\mathcal{P}}^{\pi})^h (r - \widehat{r}) \\ = & \sum_{h=0}^{H} ( \widehat{\mathcal{P}}^{\pi} )^h \big( Q_h^{\pi} - (\widehat{r} + \widehat{\mathcal{P}}^{\pi} Q_{h+1}^{\pi}) \big). \end{aligned} \end{equation}
Further, we have the following error decomposition into three terms: a first-order function of $\big( \mathcal{P}^{\pi} - \widehat{\mathcal{P}}^{\pi} \big) $, a high-order function of $\big( \mathcal{P}^{\pi} - \widehat{\mathcal{P}}^{\pi} \big) $, and a bias term due to $\lambda$.

\begin{lemma}\label{lemma:decompose}It always holds that
	\begin{equation} \label{Edecompose} \begin{aligned} 
	v^{\pi} - \widehat{v}^{\pi} = E_1 + E_2 + E_3,
	\end{aligned} \end{equation}
	where
	\begin{align}
	& E_1 := \sum_{h=0}^H ( \nu_h^{\pi} )^{\top} \Sigma^{-1} \Bigg( \frac{1}{N} \sum_{n=1}^N \phi(s_n,a_n) \Big( Q_h^{\pi}(s_n,a_n) - \big( r_n' + V_{h+1}^{\pi}(s_n') \big) \Big) \Bigg), \label{E1} \\
	& E_2 := \sum_{h=0}^{H} \Big( N ( \widehat{\nu}_h^{\pi} )^{\top} \widehat{\Sigma}^{-1} -  ( \nu_{h}^{\pi} )^{\top} \Sigma^{-1} \Big) \Bigg( \! \frac{1}{N} \! \sum_{n=1}^N \! \phi(s_n,a_n) \Big( Q_h^{\pi}(s_n, a_n) - \big( r_n' + V_{h+1}^{\pi}(s_n') \big) \Big) \! \Bigg), \label{def_E2} \\
	& E_3 := \lambda \sum_{h=0}^{H} ( \widehat{\nu}_h^{\pi} )^{\top} \widehat{\Sigma}^{-1} w_h^{\pi}. \label{def_E3}
	\end{align}
	Here, $(\widehat{\nu}_h^{\pi})^{\top} = ( \nu_0^{\pi} )^{\top} \big(\widehat{M}^{\pi}\big)^h$, $w_h^{\pi} \in \mathbb{R}^{d}$ satisfies $Q_h^{\pi}(\cdot) = \phi(\cdot)^{\top} w_h^{\pi}$.
\end{lemma}

\begin{proof}
	Note that
	\begin{equation} \label{v-v_operator} v^{\pi} - \widehat{v}^{\pi} = \int_{\mathcal{X}} \big( Q_0^{\pi}(s,a) - \widehat{Q}_0^{\pi}(s,a) \big) \xi_0(s) \pi(a \, | \, s) {\rm d} s {\rm d} a. \end{equation} 
	In the following, we reform the expression of $Q_0^{\pi} - \widehat{Q}_0^{\pi}$ in \eqref{Q_decompose} with a vector form. 
	
	Consider $Q_h^{\pi} - (\widehat{r} + \widehat{\mathcal{P}}^{\pi} Q_{h+1}^{\pi})$.
	According to the definitions of $\widehat{r}$ and $\widehat{\mathcal{P}}^{\pi}$ in \eqref{def_hatr} and \eqref{hatP_operator},
	\begin{equation} \label{hatPQ} \big( \widehat{r} + \widehat{\mathcal{P}}^{\pi} Q_{h+1}^{\pi} \big)(s,a) = \phi(s,a)^{\top} \widehat{\Sigma}^{-1} \sum_{n=1}^N \phi(s_n,a_n) \big( r_n' + V_{h+1}^{\pi}(s_n') \big). \end{equation}
	Under Assumption \ref{Q_class}, there exists a vector $w_h^{\pi} \in \mathbb{R}^d$ such that \begin{equation} \label{w_h} Q_h^{\pi}(s,a) = \phi(s,a)^{\top} w_h^{\pi}. \end{equation} 
	We have
	\[ \begin{aligned} Q_h^{\pi}(s,a) = & \phi(s,a)^{\top} \widehat{\Sigma}^{-1} \widehat{\Sigma} w_h^{\pi} = \phi(s,a)^{\top} \widehat{\Sigma}^{-1} \Bigg( \lambda I + \sum_{n=1}^N \phi(s_n,a_n) \phi(s_n,a_n)^{\top} \Bigg) w_h^{\pi} \\ = & \lambda \phi(s,a)^{\top} \widehat{\Sigma}^{-1} w_h^{\pi} + \phi(s,a)^{\top} \widehat{\Sigma}^{-1} \sum_{n=1}^N \phi(s_n,a_n) \phi(s_n,a_n)^{\top} w_h^{\pi} \\ = & \lambda \phi(s,a)^{\top} \widehat{\Sigma}^{-1} w_h^{\pi} + \phi(s,a)^{\top} \widehat{\Sigma}^{-1} \sum_{n=1}^N \phi(s_n,a_n) Q_h^{\pi}(s_n, a_n). \end{aligned} \]
	It follows that
	\begin{equation} \label{diff_PQ}
			\Big( Q_h^{\pi} - (\widehat{r} + \widehat{\mathcal{P}}^{\pi} Q_{h+1}^{\pi}) \Big)(s,a) = \lambda \phi(s,a)^{\top} \widehat{\Sigma}^{-1} w_h^{\pi} + \sum_{n=1}^N \phi(s,a)^{\top} \widehat{\Sigma}^{-1} \phi(s_n,a_n) \Big( Q_h^{\pi}(s_n,a_n) - \big(r_n' + V_{h+1}^{\pi}(s_n') \big) \Big).
	\end{equation}
	
	Note that for any $f \in \mathcal{Q}$ with $f(s,a) = \phi(s,a)^{\top} \mu$, we have $\big( \widehat{\mathcal{P}}^{\pi} f \big)(s,a) = \phi(s,a)^{\top} \widehat{M}^{\pi} \mu$, therefore, \[ \big( \big( \widehat{\mathcal{P}}^{\pi}\big)^h f \big)(s,a) = \phi(s,a)^{\top} \big(\widehat{M}^{\pi}\big)^h \mu. \]
	Then \eqref{diff_PQ} implies
	\begin{equation} \label{diff_Q0term} \begin{aligned} & ( \widehat{\mathcal{P}}^{\pi} )^h \big( Q_h^{\pi} - (\widehat{r} + \widehat{\mathcal{P}}^{\pi} Q_{h+1}^{\pi}) \big)(s,a) \\ = & \lambda \phi(s,a)^{\top} \! \big(\widehat{M}^{\pi}\big)^h \widehat{\Sigma}^{-1} w_h^{\pi} + \sum_{n=1}^N \phi(s,a)^{\top} \! \big( \widehat{M}^{\pi} \big)^h \widehat{\Sigma}^{-1} \phi(s_n,a_n) \Big( Q_h^{\pi}(s_n, a_n) - \big( r_n'+V_{h+1}^{\pi}(s_n') \big) \Big). \end{aligned} \end{equation}
	
	Plugging \eqref{diff_Q0term} into \eqref{v-v_operator} yields
	\begin{equation} \label{v-v_expression} v^{\pi} - \widehat{v}^{\pi} = \sum_{h=0}^{H} \Bigg( \lambda ( \widehat{\nu}_h^{\pi} )^{\top} \widehat{\Sigma}^{-1} w_h^{\pi} + \underbrace{\sum_{n=1}^N ( \widehat{\nu}_h^{\pi} )^{\top} \widehat{\Sigma}^{-1} \phi(s_n,a_n) \Big( Q_h^{\pi}(s_n, a_n) - \big( r_n'+V_{h+1}^{\pi}(s_n') \big) \Big)}_{Err_h} \Bigg), \end{equation}
	where we have used the definitions $\nu_0^{\pi} = \mathbb{E}\big[ \phi(s,a) \, \big| \, s \sim \xi_0, a \sim \pi(\cdot \, | \, s) \big]$ and $(\widehat{\nu}_h^{\pi})^{\top} = (\nu_0^{\pi})^{\top} \big( \widehat{M}^{\pi} \big)^h$. In \eqref{v-v_expression}, $\lambda (\widehat{\nu}_h^{\pi})^{\top} \widehat{\Sigma}^{-1} w_h^{\pi}$ is the bias term induced by the ridge penalty $\lambda \rho(\cdot)$ in \eqref{Q_recursion} and \eqref{hatP_operator}. 
	As for $Err_h$, we replace the data-dependent terms $(\widehat{\nu}_h^{\pi})^{\top}$ and $\widehat{\Sigma}^{-1}$ with their population counterparts $(\nu_h^{\pi})^{\top} = (\nu_0^{\pi})^{\top} M^{\pi}$ and $N^{-1}\Sigma^{-1}$. $Err_h$ is then the sum of first-order approximation
	\[ \frac{1}{N} \sum_{n=1}^N ( \nu_h^{\pi} )^{\top} \Sigma^{-1} \phi(s_n,a_n) \Big( Q_h^{\pi}(s_n, a_n) - \big( r_n'+V_{h+1}^{\pi}(s_n') \big) \Big) \]
	and high-order remainder
	\[  \frac{1}{N}  \sum_{n=1}^N \Big( N (\widehat{\nu}_h^{\pi})^{\top} \widehat{\Sigma}^{-1} - ( \nu_h^{\pi} )^{\top} \Sigma^{-1} \Big) \phi(s_n,a_n) \Big( Q_h^{\pi}(s_n, a_n) - \big( r_n'+V_{h+1}^{\pi}(s_n') \big) \Big). \]
	In this way, we propose the decomposition $v^{\pi} - \widehat{v}^{\pi} = E_1 + E_2 + E_3$, where the first-order error $E_1$, high-order error $E_2$ and bias $E_3$ are given in \eqref{E1}, \eqref{def_E2} and \eqref{def_E3}.
\end{proof}

In the following, we analyze $E_1$, $E_2$ and $E_3$ separately in Sections \ref{section:E1}, \ref{section:E2} and \ref{section:E3}, and integrate the results in Section \ref{section:sum}.

\subsection{First-Order Term $E_1$} \label{section:E1}

Note that $E_1 = \frac{1}{N} \sum_{n=1}^N e_n$, where
\begin{equation} \label{e_t} e_n := \sum_{h=0}^{H-1} (\nu_h^{\pi})^{\top} \Sigma^{-1} \phi(s_n,a_n) \Big( Q_h^{\pi}(s_n, a_n) - \big( r_n'+V_{h+1}^{\pi}(s_n') \big) \Big), \quad n = 1,2,\ldots,N. \end{equation}
Define a filtration $\big\{ \mathcal{F}_n \big\}_{n=1,\ldots,N}$ with $\mathcal{F}_n$ generated by $(s_1, a_1, s_1'), \ldots, (s_{n-1}, a_{n-1}, s_{n-1}')$ and $(s_n,a_n)$.
The identity
$\mathbb{E}\big[ e_n \, \big| \, \mathcal{F}_n \big] = 0$
implies that $\{e_n\}_{n=1,\ldots,N}$ is a martingale difference sequence.
In the following, we analyze the large-deviation behavior of $E_1$ with Freedman's inequality \cite{freedman1975tail}.

\begin{lemma}[Error in the first-order term, $E_1$] \label{lemma:E1}
	Under the assumption $\phi(s,a)^{\top} \Sigma^{-1} \phi(s,a) \leq C_1 d$ for all $(s,a) \in \mathcal{X}$, with probability at least $1 - \delta$,
	\begin{equation} \label{eqnE1} |E_1| \leq \sum_{h=0}^{H} (H-h+1) \sqrt{( \nu_h^{\pi} )^{\top} \Sigma^{-1} \nu_h^{\pi}} \cdot \sqrt{\frac{\ln(4/\delta)}{2N}} + \Delta E_1, \end{equation}
	where $\Delta E_1$ is a high-order term given by
	\[ \Delta E_1 = \sum_{h=0}^{H} (H-h+1) \sqrt{( \nu_h^{\pi} )^{\top} \Sigma^{-1} \nu_h^{\pi}} \cdot \bigg( \frac{7\ln(4d/\delta)\sqrt{C_1dH}}{6N} + \frac{\big(\ln(4d/\delta)\big)^{3/2} C_1 d H}{3\sqrt{2}N^{3/2}} \bigg). \]
	If we further have $\phi(s,a)^{\top} \Sigma^{-1} \phi(s',a') \geq 0$ for any $(s,a), (s',a') \in \mathcal{X}$ or the MDP is time-inhomogeneous, the upper bound \eqref{eqnE1} can be improved to
	\begin{equation} \label{eqnE1_new} |E_1| \leq \sqrt{\Bigg( \sum_{h=0}^{H} (H-h+1) \nu_h^{\pi} \Bigg)^{\top} \Sigma^{-1} \Bigg( \sum_{h=0}^{H} (H-h+1) \nu_h^{\pi} \Bigg)} \cdot \sqrt{\frac{\ln(4/\delta)}{2N}} + \Delta E_1. \end{equation}
\end{lemma}

We only present the proof of \eqref{eqnE1} here. The proof of \eqref{eqnE1_new} when $\phi(s,a)^{\top} \Sigma^{-1} \phi(s',a') \geq 0$ is similar and we defer it to Appendix \ref{appendix:eqnE1_new}.
We will use the following Lemma \ref{lemma:Term1} regarding the concentration of uncentered sample covariance matrix $\frac{1}{N} \sum_{n=1}^N \phi(s_n,a_n) \phi(s_n,a_n)^{\top}$. 
See Appendix \ref{appendix:Term1} for the proof of Lemma \ref{lemma:Term1}.
\begin{lemma} \label{lemma:Term1}
	Under the assumption $\phi(s,a)^{\top} \Sigma^{-1} \phi(s,a) \leq C_1 d$ for all $(s,a) \in \mathcal{X}$, with probability at least $1 - \delta$,
	\begin{equation} \label{Term1} \Bigg\| \Sigma^{-1/2} \bigg( \frac{1}{N} \sum_{n=1}^N \phi(s_n,a_n) \phi(s_n,a_n)^{\top} \bigg) \Sigma^{-1/2} - I \Bigg\|_2 \leq \sqrt{\frac{2\ln(2d/\delta)C_1dH}{N}} + \frac{2\ln(2d/\delta)C_1 d H}{3N}. \end{equation}
\end{lemma}
We are now ready to prove \eqref{eqnE1}.
\begin{proof}[Proof of \eqref{eqnE1}]
	When $\phi(s,a)^{\top} \Sigma^{-1} \phi(s,a) \leq C_1 d$ for all $(s,a) \in \mathcal{X}$, the difference sequence $\{ e_n \}_{n=1}^N$ is uniformly bounded.
	In fact, since $r \in [0,1]$, we have $r_n'+V_{h+1}^{\pi}(s_n') \in [0,H-h]$ and
	\[ \begin{aligned} |e_n| \leq & \sum_{h=0}^{H} \big| (\nu_h^{\pi})^{\top} \Sigma^{-1} \phi(s_n,a_n) \big| \cdot \Big| Q_h^{\pi}(s_n, a_n) - \big( r_n'+V_{h+1}^{\pi}(s_n') \big) \Big| \\ \leq & \sum_{h=0}^H (H-h+1) \big| (\nu_h^{\pi})^{\top} \Sigma^{-1} \phi(s_n,a_n) \big| \leq \sqrt{C_1d} \sum_{h=0}^{H} (H-h+1) \sqrt{ (\nu_h^{\pi})^{\top} \Sigma^{-1} \nu_h^{\pi} }, \end{aligned} \]
	where we have used $\big| \mu^{\top} \Sigma^{-1} \phi(s, a) \big| \leq \sqrt{C_1d} \sqrt{\mu^{\top} \Sigma^{-1} \mu}$ for any $\mu \in \mathbb{R}^d$, $(s,a) \in \mathcal{X}$.
	For simplicity, we denote
	\begin{equation} \label{B}
	B := \sum_{h=0}^{H} (H-h+1) \sqrt{ (\nu_h^{\pi})^{\top} \Sigma^{-1} \nu_h^{\pi} }.
	\end{equation}
	
	Next, we consider $\sum_{n=1}^N {\rm Var}\big[ e_n \, \big| \, \mathcal{F}_n \big]$.
	By Cauchy-Schwarz inequality,
	\[ \begin{aligned} {\rm Var} \big[ e_n \, \big| \, \mathcal{F}_n \big] = & \mathbb{E} \Bigg[ \bigg( \sum_{h=0}^H (\nu_h^{\pi})^{\top} \Sigma^{-1} \phi(s_n,a_n) \Big( Q_h^{\pi}(s_n, a_n) - \big( r_n'+V_{h+1}^{\pi}(s_n') \big) \Big) \! \bigg)^2 \, \Bigg| \, \mathcal{F}_n \Bigg] \\ \leq & \left( \sum_{h=0}^H \frac{\sqrt{(\nu_h^{\pi})^{\top} \Sigma^{-1} \nu_h^{\pi}}}{H-h+1} {\rm Var}\big[r_n' + V_{h+1}^{\pi}(s_n') \, \big| \, s_n, a_n \big] \! \right) \!\! \left( \sum_{h=0}^H \frac{H-h+1}{\sqrt{(\nu_h^{\pi})^{\top} \Sigma^{-1} \nu_h^{\pi}}} \Big( (\nu_h^{\pi})^{\top} \Sigma^{-1} \phi(s_n,a_n) \Big)^2 \right). \end{aligned} \]
	Since $r_n' + V_{h+1}^{\pi}(s_n') \in [0, H-h+1]$, the conditional variance ${\rm Var}\big[r_n' + V_{h+1}^{\pi}(s_n') \, \big| \, s_n, a_n\big] \leq \frac{1}{4} (H-h+1)^2$.
	It follows that
	\[ {\rm Var} \big[ e_n \, \big| \, \mathcal{F}_n \big] \leq \frac{B}{4} \sum_{h=0}^H \frac{H-h+1}{\sqrt{(\nu_h^{\pi})^{\top} \Sigma^{-1} \nu_h^{\pi}}} \Big( \big(\nu_h^{\pi}\big)^{\top} \Sigma^{-1} \phi(s_n,a_n) \Big)^2, \]
	and
	\[ \begin{aligned} \sum_{n=1}^N {\rm Var}\big[ e_n \, \big| \, \mathcal{F}_n \big] \leq & \frac{BN}{4} \sum_{h=0}^{H} \frac{H-h+1}{\sqrt{(\nu_h^{\pi})^{\top} \Sigma^{-1} \nu_h^{\pi}}} (\nu_h^{\pi})^{\top} \Sigma^{-1} \bigg( \frac{1}{N} \sum_{n=1}^N \phi(s_n,a_n) \phi(s_n,a_n)^{\top} \bigg) \Sigma^{-1} \nu_h^{\pi} \\ \leq & \frac{BN}{4} \sum_{h=0}^{H}  \frac{H-h+1}{\sqrt{(\nu_h^{\pi})^{\top} \Sigma^{-1} \nu_h^{\pi}}} \cdot (\nu_h^{\pi})^{\top} \Sigma^{-1} \nu_h^{\pi} \Bigg\| \Sigma^{-1/2} \bigg( \frac{1}{N} \sum_{n=1}^N \phi(s_n,a_n) \phi(s_n,a_n)^{\top} \bigg) \Sigma^{-1/2} \Bigg\|_2 \\ \leq & \frac{B^2N}{4} \Bigg\| \Sigma^{-1/2} \bigg( \frac{1}{N} \sum_{n=1}^N \phi(s_n,a_n) \phi(s_n,a_n)^{\top} \bigg) \Sigma^{-1/2} \Bigg\|_2, \end{aligned} \]
	where we have used $\mu^{\top} \Sigma^{-1} X \Sigma^{-1} \mu \leq \mu^{\top} \Sigma^{-1} \mu \cdot \big\| \Sigma^{-1/2} X \Sigma^{-1/2} \big\|_2$ for any $\mu \in \mathbb{R}^d$, $X \in \mathbb{R}^{d \times d}$.
	We take
	\begin{equation} \label{sigma} \sigma^2 := N \Bigg( 1 + \sqrt{\frac{2\ln(4d/\delta)C_1dH}{N}} + \frac{2\ln(4d/\delta) C_1 d H}{3N} \Bigg) \cdot \frac{B^2}{4}. \end{equation}
	According to Lemma \ref{lemma:Term1}, it holds that
	\begin{equation} \label{P1} \mathbb{P} \Bigg( \sum_{n=1}^N {\rm Var}\big[ e_n \, \big| \, \mathcal{F}_n \big] \geq \sigma^2 \Bigg) \leq  \delta/2. \end{equation}
	
	The Freedman's inequality implies that for any $\varepsilon \in \mathbb{R}$,
	\[ \mathbb{P}\Bigg( \bigg|\sum_{n=1}^N e_n\bigg| \geq \varepsilon, \sum_{n=1}^N {\rm Var}\big[ e_n \, \big| \, \mathcal{F}_n \big] \leq \sigma^2  \Bigg) \leq 2 \exp \bigg( - \frac{\varepsilon^2/2}{\sigma^2 + \sqrt{C_1d}B \varepsilon / 3} \bigg), \]
	where $B$ is given in \eqref{B} and $\sigma^2$ is defined in \eqref{sigma}.
	By taking
	\begin{equation} \label{varepsilon} \varepsilon := \sqrt{2 \ln(4/\delta)} \sigma + 2 \ln(4/\delta) \sqrt{C_1 d} B / 3,  \end{equation}
	one has
	\begin{equation} \label{P2} \mathbb{P}\Bigg( \bigg|\sum_{n=1}^N e_n\bigg| \geq \varepsilon, \, \sum_{n=1}^N {\rm Var}\big[ e_n \, \big| \, \mathcal{F}_n \big] \leq \sigma^2  \Bigg) \leq \delta/2. \end{equation}
	Combining \eqref{P1} and \eqref{P2}, we obtain
	\[ \mathbb{P}\Bigg( \bigg|\sum_{n=1}^N e_n\bigg| \geq \varepsilon \Bigg) \leq \mathbb{P}\Bigg( \bigg|\sum_{n=1}^N e_n\bigg| \geq \varepsilon, \sum_{n=1}^N {\rm Var}\big[ e_n \, \big| \, \mathcal{F}_n \big] \leq \sigma^2  \Bigg) + \mathbb{P} \Bigg( \sum_{n=1}^N {\rm Var}\big[ e_n \, \big| \, \mathcal{F}_n \big] \geq \sigma^2 \Bigg) \leq \delta. \]
	Using the inequality $\sqrt{1+x} \leq 1 + \frac{x}{2}$, $\forall x \geq 0$, we derive from \eqref{varepsilon} that
	\[ \frac{\varepsilon}{N} \leq B \Bigg( \sqrt{\frac{\ln(4/\delta)}{2N}} + \frac{7\ln(4d/\delta)\sqrt{C_1dH}}{6N} + \frac{\big(\ln(4d/\delta)\big)^{3/2} C_1 d H}{3\sqrt{2}N^{3/2}} \Bigg), \]
	which completes the proof of \eqref{eqnE1}.
\end{proof}

\subsection{High-Order Term $E_2$} \label{section:E2}

Recall that
\[ E_2 = \sum_{h=0}^H \Big( N \big( \widehat{\nu}_h^{\pi} \big)^{\top} \widehat{\Sigma}^{-1} - ( \nu_h^{\pi} )^{\top} \Sigma^{-1} \Big) \Bigg( \frac{1}{N} \sum_{n=1}^N \phi(s_n,a_n) \Big( Q_h^{\pi}(s_n, a_n) - \big( r_n'+V_{h+1}^{\pi}(s_n') \big) \Big) \Bigg). \]

\begin{lemma}[High-Order Term $E_2$] \label{lemma:E2}
	Suppose $\phi(s,a)^{\top} \Sigma^{-1} \phi(s,a) \leq C_1 d$ for all $(s,a) \in \mathcal{X}$. For any $\delta \in (0,1)$, if $N \geq 20 \kappa_1(2+\kappa_2)^2 \ln(8dH/\delta)C_1dH^3$ and $\lambda \leq \ln(8dH/\delta)C_1dH \sigma_{\min}(\Sigma)$,
	then there exists an event $\mathcal{E}_{\delta}$ such that $\mathbb{P}\big(\mathcal{E}_{\delta}\big) \geq 1 - \delta$ and $\mathcal{E}_{\delta}$ implies
	\begin{equation} \label{eqnE2} |E_2| \leq 15 \sqrt{(\nu_0^{\pi})^{\top} (\Sigma^{\pi})^{-1}  \nu_0^{\pi}} \cdot \big\| (\Sigma^{\pi})^{1/2} \Sigma^{-1/2} \big\|_2 \cdot \sqrt{C_1\kappa_1}(2+\kappa_2) \cdot \frac{\ln(8dH/\delta)dH^{3.5}}{N}. \end{equation}
	Here, $\kappa_1$ and $\kappa_2$ are defined in Theorem \ref{UpperBound}.
\end{lemma}

In order to prove Lemma \ref{lemma:E2}, we first decompose $E_2$ into terms that are tractable to control. 
In the following, we begin with a preliminary Lemma \ref{SMS}. We leverage the contraction property \eqref{SMS<1} and propose a decomposition of $E_2$ in Lemma \ref{lemma:E2decompose}. The upper bound \eqref{E2decompose} is a deterministic result. It does not grow exponentially with the horizon $H$. The proofs of Lemma \ref{lemma:E2decompose} is deferred to Appendix \ref{appendix:E2decompose}.


\begin{lemma}[Decomposition of $E_2$] \label{lemma:E2decompose}
	\begin{enumerate}
		\item It always holds that \begin{equation} \label{E2decompose} \begin{aligned} |E_2| \leq & \sum_{h=0}^H \sqrt{(\nu_0^{\pi})^{\top} (\Sigma^{\pi})^{-1}  \nu_0^{\pi}} \cdot \big\| (\Sigma^{\pi})^{1/2} \Sigma^{-1/2} \big\|_2 \big\| \Sigma^{-1/2} \Delta W_h^{\pi} \big\|_2 \\ & \qquad \cdot \Big( \big( 1 +  \big\| (\Sigma^{\pi})^{1/2} \Delta M^{\pi} (\Sigma^{\pi})^{-1/2} \big\|_2 \big)^h \big( 1 + \big\| \Sigma^{1/2} (\Delta X) \Sigma^{1/2} \big\|_2\big) - 1 \Big), \end{aligned} \end{equation}
		where $\Delta X := N \widehat{\Sigma}^{-1} - \Sigma^{-1}$, $\Delta M^{\pi} := \widehat{M}^{\pi} - M^{\pi}$,
		\begin{equation} \label{def_DeltaWh} \Delta W_h^{\pi} := \frac{1}{N} \sum_{n=1}^N \phi(s_n,a_n) \Big( Q_h^{\pi}(s_n, a_n) - \big( r_n'+V_{h+1}^{\pi}(s_n') \big) \Big). \end{equation}
		\item Given $\widehat{M}^{\pi}$ in \eqref{vector_recursion}, one further has
		\begin{equation} \label{DeltaM<}
			\big\| (\Sigma^{\pi})^{1/2} \Delta M^{\pi} (\Sigma^{\pi})^{-1/2} \big\|_2 \leq \sqrt{\kappa_1} \Big( \big( 1 + \big\| \Sigma^{1/2} (\Delta X) \Sigma^{1/2} \big\|_2 \big) \big( 1 + \big\| \Sigma^{-1/2} (\Delta Y^{\pi}) \Sigma^{-1/2} \big\|_2 \big) - 1 \Big),
		\end{equation}
		where $\Delta Y^{\pi} \! := \! \frac{1}{N} \sum_{n=1}^N \phi(s_n,a_n) \phi^{\pi}(s_n')\!^{\top} \! -  \Sigma M^{\pi}$, $\kappa_1$ is the condition number defined in Theorem \ref{UpperBound}.
		\item
		If $\big\| N^{-1} \Sigma^{-1/2} \widehat{\Sigma} \Sigma^{-1/2} - I \big\|_2 \leq \frac{1}{2}$, then \begin{equation} \label{inv<2Dx} \big\| \Sigma^{1/2}(\Delta X) \Sigma^{1/2} \big\|_2 \leq 2 \big\| N^{-1} \Sigma^{-1/2} \widehat{\Sigma} \Sigma^{-1/2} - I \big\|_2. \end{equation}
	\end{enumerate}
\end{lemma}

Lemma \ref{E2decompose} shows that the problem is now reduced to estimating \[ \big\| N^{-1} \Sigma^{-1/2} \widehat{\Sigma} \Sigma^{-1/2} - I \big\|_2, \ \big\| \Sigma^{-1/2} (\Delta Y^{\pi}) \Sigma^{-1/2} \big\|_2 \quad \text{and} \quad \big\| \Sigma^{-1/2} \Delta W_h^{\pi} \big\|_2. \] We present the upper bounds in \eqref{DeltatildeX}, Lemmas \ref{lemma:Term2} and \ref{lemma:Term3}. The proofs of the Lemmas are defered to Appendices \ref{appendix:Term2} and \ref{appendix:Term3}.

We learn from Lemma \ref{lemma:Term1} that, with probability at least $1 - \delta$,
\begin{equation} \label{DeltatildeX} \big\| N^{-1} \Sigma^{-1/2} \widehat{\Sigma} \Sigma^{-1/2} - I \big\|_2 \leq \sqrt{\frac{2\ln(2d/\delta)C_1dH}{N}} + \frac{2\ln(2d/\delta)C_1 dH}{3N} + \frac{\lambda \big\|\Sigma^{-1}\big\|_2}{N}. \end{equation}

\begin{lemma} \label{lemma:Term2}
	Under the assumption $\phi(s,a)^{\top} \Sigma^{-1} \phi(s,a) \leq C_1 d$ for all $(s,a) \in \mathcal{X}$, with probability at least $1 - \delta$,
	\begin{equation} \label{Term2} \begin{aligned} \big\| \Sigma^{-1/2} (\Delta Y^{\pi}) \Sigma^{-1/2} \big\|_2 \leq & \sqrt{\frac{2\ln(2d/\delta)C_1 d H}{N}} \cdot \kappa_2 + \frac{ 4 \ln(2d/\delta) C_1 dH}{3N}, \end{aligned} \end{equation}
	where $\kappa_1$ and $\kappa_2$  are defined in Theorem \ref{UpperBound}.
\end{lemma}


\begin{lemma} \label{lemma:Term3}
	Under the assumption $\phi(s,a)^{\top} \Sigma^{-1} \phi(s,a) \leq C_1 d$ for all $(s,a) \in \mathcal{X}$, for $h=1,2,\ldots,H$, with probability at least $1 - \delta$,
	\begin{equation} \label{Term3} 	\begin{aligned} \big\| \Sigma^{-1/2} \Delta W_h^{\pi} \big\|_2 \leq \sqrt{d} (H-h+1) \Bigg( \sqrt{\frac{ \ln\big((3d+1)/\delta\big)}{2N}} + \frac{7 \ln\big((3d+1)/\delta\big)\sqrt{C_1dH}}{6N} + \frac{\big(\ln\big((3d+1)/\delta\big)\big)^{3/2} C_1 d H}{3\sqrt{2}N^{3/2}} \Bigg). \end{aligned} \end{equation}
\end{lemma}

We now prove Lemma \ref{lemma:E2}.
\begin{proof}[Proof of Lemma \ref{lemma:E2}]
	We plug \eqref{DeltatildeX}, \eqref{Term2} and \eqref{Term3} into Lemma \ref{lemma:E2decompose}.
	Suppose that 
	\begin{equation} \label{E2condition} N \geq 18 \ln(8dH/\delta)C_1dH \quad \text{and} \quad \lambda \leq \ln(8dH/\delta)C_1dH \sigma_{\min}(\Sigma). \end{equation}
	According to \eqref{DeltatildeX}, $\big\| N^{-1} \Sigma^{-1/2} \widehat{\Sigma} \Sigma^{-1/2} - I \big\|_2 \leq \frac{1}{2}$ with probability at least $1 - \delta/4$. Then it follows from \eqref{inv<2Dx} that
	\begin{equation} \label{Term1'} \big\| \Sigma^{1/2}(\Delta X) \Sigma^{1/2} \big\|_2 \leq 2 \big\| N^{-1} \Sigma^{-1/2} \widehat{\Sigma} \Sigma^{-1/2} - I \big\|_2 \leq 4  \sqrt{\frac{\ln(8dH/\delta)C_1dH}{N}}. \end{equation}
	Lemmas \ref{lemma:Term2} and \ref{lemma:Term3} show that under \eqref{E2condition}, with probability at least $1-\delta/4$,
	\begin{equation} \label{Term2'} \big\| \Sigma^{-1/2} (\Delta Y^{\pi}) \Sigma^{-1/2} \big\|_2 \leq 2 \sqrt{\frac{\ln(8dH/\delta)C_1 d H}{N}} \cdot \kappa_2, \end{equation}
	and by union bound, with probability at least $1-\delta/2$,
	\begin{equation} \label{Term3'} \big\| \Sigma^{-1/2} \Delta W_h^{\pi} \big\|_2 \leq \sqrt{d} (H-h+1) \sqrt{\frac{ \ln(8dH/\delta)}{N}} \qquad \text{for $H=1,2,\ldots,H$}. \end{equation}
	Define
	\begin{equation} \label{event} \mathcal{E}_{\delta} := \big\{ \text{\eqref{Term1'}, \eqref{Term2'} and \eqref{Term3'} hold simultaneously} \big\}. \end{equation}
	By union bound, $\mathbb{P}(\mathcal{E}_{\delta}) \geq 1 - \delta$ under condition \eqref{E2condition}.
	
	Suppose \eqref{E2condition} and $\mathcal{E}_{\delta}$ hold. We apply \eqref{Term1'} and \eqref{Term2'} to \eqref{DeltaM<}. Under \eqref{E2condition}, $\big\| \Sigma^{1/2}(\Delta X) \Sigma^{1/2} \big\|_2 \leq \frac{2\sqrt{2}}{3}$, $\big\| \Sigma^{-1/2} (\Delta Y^{\pi}) \Sigma^{-1/2} \big\|_2 \leq \frac{\sqrt{2}}{3}$. For any $x \in \big[0,\frac{2\sqrt{2}}{3}\big]$, $y \in \big[0,\frac{\sqrt{2}}{3}\big]$, since $1+x \leq e^x$, $1+y \leq e^y$ and $\frac{e^{x+y} - 1}{x+y} \leq \frac{e^{\sqrt{2}}-1}{\sqrt{2}} \leq \sqrt{5}$, it holds that $(1+x)(1+y) - 1 \leq e^x e^y - 1 = e^{x+y} - 1 \leq \sqrt{5}(x+y)$. It follows from \eqref{DeltaM<}, \eqref{Term1'} and \eqref{Term2'} that
	\begin{equation} \label{DeltaM<<} \begin{aligned} \big\| (\Sigma^{\pi})^{1/2} \Delta M^{\pi} (\Sigma^{\pi})^{-1/2} \big\|_2 \! \leq & \sqrt{5\kappa_1} \Big( \big\| \Sigma^{1/2} (\Delta X) \Sigma^{1/2} \big\|_2 \!\! + \! \big\| \Sigma^{-1/2} (\Delta Y^{\pi}) \Sigma^{-1/2} \big\|_2 \Big) \! \leq \! 2\sqrt{5\kappa_1}(2\!+\!\kappa_2) \sqrt{\frac{\ln(8dH/\delta)C_1dH}{N}}. \end{aligned} \end{equation}
	
	We plug \eqref{Term1'} and \eqref{DeltaM<<} into \eqref{E2decompose} to derive an estimate for $E_2$.
	For notational simplicity, denote
	\begin{equation} \label{def_alpha} \alpha := 2\sqrt{5\kappa_1}(2+\kappa_2) \sqrt{\frac{\ln(8dH/\delta)C_1dH}{N}}. \end{equation}
	Then \eqref{DeltaM<<} and \eqref{Term1'} show that $\big\| (\Sigma^{\pi})^{1/2} \Delta M^{\pi} (\Sigma^{\pi})^{-1/2} \big\|_2 \leq \alpha$ and $\big\| \Sigma^{1/2} (\Delta X) \Sigma^{1/2} \big\|_2 \leq \alpha$. To this end,
	\[ \big( 1 + \big\| (\Sigma^{\pi})^{1/2} \Delta M^{\pi} (\Sigma^{\pi})^{-1/2} \big\|_2 \big)^h \big( 1 + \big\| \Sigma^{1/2} (\Delta X) \Sigma^{1/2} \big\|_2 \big) \leq (1+\alpha)^{h+1}. \]
	In order that $(1 + \alpha)^{h+1}$ does not grow exponentially with $h$, we enforce $\alpha \leq \frac{1}{H}$,
	or equivalently,
	\begin{equation} \label{E2condition2} N \geq 20 \kappa_1(2+\kappa_2)^2 \ln(8dH/\delta)C_1dH^3. \end{equation}
	Under condition \eqref{E2condition2}, $(1+\alpha)^{h+1}-1 \leq e^{(h+1)\alpha} - 1 \leq \frac{e^{3/2}-1}{3/2} (h+1) \alpha \leq \frac{5}{2}(h+1) \alpha$ for $h=0,1,\ldots,H$ and $H \geq 2$. It follows from \eqref{E2decompose} that
	\[ |E_2| \leq \sqrt{(\nu_0^{\pi})^{\top} (\Sigma^{\pi})^{-1}  \nu_0^{\pi}} \cdot \big\| (\Sigma^{\pi})^{1/2} \Sigma^{-1/2} \big\|_2 \sum_{h=0}^H \frac{5}{2}(h+1) \alpha \cdot \big\| \Sigma^{-1/2} \Delta W_{h}^{\pi} \big\|_2. \]
	Substituting $\big\| \Sigma^{-1} \Delta W_h^{\pi} \big\|_2$ with its upper bound in \eqref{Term3'}, we learn that
	\[ \begin{aligned} |E_2| \leq & \sqrt{(\nu_0^{\pi})^{\top} (\Sigma^{\pi})^{-1}  \nu_0^{\pi}} \cdot \big\| (\Sigma^{\pi})^{1/2} \Sigma^{-1/2} \big\|_2 \cdot \sum_{h=0}^{H} \frac{5}{2}(h+1)\alpha \cdot \sqrt{d} (H-h+1) \sqrt{\frac{ \ln(8dH/\delta)}{N}} \\ \leq & \sqrt{(\nu_0^{\pi})^{\top} (\Sigma^{\pi})^{-1}  \nu_0^{\pi}} \cdot \big\| (\Sigma^{\pi})^{1/2} \Sigma^{-1/2} \big\|_2 \cdot \alpha \sqrt{\frac{ \ln(8dH/\delta)d}{N}}  \cdot \sum_{h=0}^{H} \frac{5}{2}(h+1)(H-h+1) \\ \leq & \sqrt{(\nu_0^{\pi})^{\top} (\Sigma^{\pi})^{-1}  \nu_0^{\pi}} \cdot \big\| (\Sigma^{\pi})^{1/2} \Sigma^{-1/2} \big\|_2 \cdot \alpha \sqrt{\frac{ \ln(8dH/\delta)d}{N}} \cdot 3.2 H^3, \end{aligned} \]
	where we have used $\sum_{h=0}^{H} (h+1)(H-h+1) \leq \frac{5}{4}H^3$ for $H \geq 2$.
	Using the definition of $\alpha$ in \eqref{def_alpha}, we further have
	\[ \begin{aligned} |E_2| \leq & \sqrt{(\nu_0^{\pi})^{\top} (\Sigma^{\pi})^{-1}  \nu_0^{\pi}} \cdot \big\| (\Sigma^{\pi})^{1/2} \Sigma^{-1/2} \big\|_2 \cdot 2\sqrt{5\kappa_1}(2+\kappa_2) \sqrt{\frac{\ln(8d/\delta)C_1dH}{N}} \cdot \sqrt{\frac{ \ln(8dH/\delta)d}{N}} \cdot 3.2 H^3 \\ \leq & \sqrt{(\nu_0^{\pi})^{\top} (\Sigma^{\pi})^{-1}  \nu_0^{\pi}} \cdot \big\| (\Sigma^{\pi})^{1/2} \Sigma^{-1/2} \big\|_2 \cdot 15\sqrt{C_1\kappa_1}(2+\kappa_2) \cdot \frac{\ln(8dH/\delta)dH^{3.5}}{N}. \end{aligned} \]
	
	In summary, we conclude that if $N \geq 20 \kappa_1(2+\kappa_2)^2 \ln(8dH/\delta)C_1dH^3$ and $\lambda \leq \ln(8dH/\delta)C_1dH \sigma_{\min}(\Sigma)$, then (i) $\mathcal{E}_{\delta}$ in \eqref{event} happens with probability at least $1-\delta$; (i) $\mathcal{E}_{\delta}$ implies \eqref{eqnE2} in Lemma \ref{lemma:E2}.

\end{proof}

\subsection{The Bias Term $E_3$} \label{section:E3}

If $\lambda=0$, we have $E_3=0$. 
In a way similar to Lemma \ref{lemma:E2}, we derive an error bound for the bias term $E_3$ in Lemma \ref{lemma:E3}. See Appendex \ref{appendix:E3} for the proof.
\begin{lemma}\label{lemma:E3}
	Suppose that $N \geq 20 \kappa_1(2+\kappa_2)^2 \ln(8dH/\delta)C_1dH^3 $ and $\lambda \leq \ln(8dH/\delta)C_1dH \sigma_{\min}(\Sigma)$. For any $\delta \in (0,1)$, let $\mathcal{E}_{\delta}$ be the event defined in Lemma \ref{lemma:E2}. Conditioned on $\mathcal{E}_{\delta}$, it holds that
	\begin{equation} \label{eqnE3} |E_3| \leq \sqrt{( \nu_0^{\pi} )^{\top} (\Sigma^{\pi})^{-1} \nu_0^{\pi}} \cdot \big\| (\Sigma^{\pi})^{1/2} \Sigma^{-1/2} \big\|_2 \cdot \frac{5\ln(8dH/\delta) C_1dH^2}{N}. \end{equation}
\end{lemma}

\subsection{Proof of Theorem \ref{UpperBound}} \label{section:sum}

We now integrate the pieces and prove the main Theorem \ref{UpperBound}.
\begin{proof}[Proof of Theorem \ref{UpperBound}]
	According to \eqref{Edecompose},
	\begin{equation} \label{temp1'} \big| v^{\pi} - \widehat{v}^{\pi} \big| \leq |E_1| + |E_2| + |E_3|. \end{equation}
	We learn from Lemma \ref{lemma:E1} that with probability at least $1 - \delta/3$,
	\begin{equation} \label{E1'} |E_1| \leq \sum_{h=0}^{H} (H-h+1) \sqrt{( \nu_h^{\pi} )^{\top} \Sigma^{-1} \nu_h^{\pi}} \cdot \sqrt{\frac{\ln(12/\delta)}{2N}} + \Delta E_1, \end{equation}
	where
	\begin{equation} \label{DeltaE1'} \Delta E_1 := \sum_{h=0}^{H} (H-h+1) \sqrt{( \nu_h^{\pi} )^{\top} \Sigma^{-1} \nu_h^{\pi}} \cdot \bigg( \frac{7\ln(12d/\delta)\sqrt{C_1dH}}{6N} + \frac{\big(\ln(12d/\delta)\big)^{3/2} C_1 d H}{3\sqrt{2}N^{3/2}} \bigg). \end{equation}
	Lemmas \ref{lemma:E2} and \ref{lemma:E3} suggest that if \[ N \geq 20 \kappa_1(2+\kappa_2)^2 \ln(12dH/\delta)C_1dH^3 \qquad \text{and} \qquad \lambda \leq \ln(12dH/\delta)C_1dH \sigma_{\min}(\Sigma), \] then with probability at least $1 - 2\delta/3$,
	\begin{equation} \label{E2''} |E_2| \leq 15 \sqrt{(\nu_0^{\pi})^{\top} (\Sigma^{\pi})^{-1}  \nu_0^{\pi}} \cdot \big\| (\Sigma^{\pi})^{1/2} \Sigma^{-1/2} \big\|_2 \cdot \sqrt{C_1\kappa_1}(2+\kappa_2) \cdot \frac{\ln(12dH/\delta)dH^{3.5}}{N} \end{equation}
	and
	\begin{equation} \label{E3''} |E_3| \leq \sqrt{( \nu_0^{\pi} )^{\top} (\Sigma^{\pi})^{-1} \nu_0^{\pi}} \cdot \big\| (\Sigma^{\pi})^{1/2} \Sigma^{-1/2} \big\|_2 \cdot \frac{5\ln(12dH/\delta) C_1dH^2}{N}. \end{equation}
	By union bound, \eqref{E1'}, \eqref{E2''} and \eqref{E3''} hold simultaneously with probability at least $1 - \delta$.
	
	We now recast $\Delta E_1$ in \eqref{DeltaE1'} so that it has a similar form to the right hand sides of \eqref{E2''} and \eqref{E3''}.
	Note that
	\[ ( \nu_h^{\pi} )^{\top} \Sigma^{-1/2} = ( \nu_0^{\pi} )^{\top} \big( M^{\pi} \big)^h \Sigma^{-1/2} = ( \nu_0^{\pi} )^{\top} ( \Sigma^{\pi} )^{-1/2} \big( ( \Sigma^{\pi} )^{1/2} M^{\pi} ( \Sigma^{\pi} )^{-1/2} \big)^h ( \Sigma^{\pi} )^{1/2} \Sigma^{-1/2}. \]
	Therefore,
	\[ \begin{aligned} \sqrt{(\nu_h^{\pi})^{\top} \Sigma^{-1} \nu_h^{\pi}} = \big\| \Sigma^{-1/2} \nu_h^{\pi} \big\|_2 \leq & \big\| \Sigma^{-1/2} ( \Sigma^{\pi} )^{1/2} \big\|_2 \big\| ( \Sigma^{\pi} )^{1/2} M^{\pi} ( \Sigma^{\pi} )^{-1/2} \big\|_2^h \big\| ( \Sigma^{\pi} )^{-1/2} \nu_0^{\pi} \big\|_2  \\ \leq & \sqrt{(\nu_0^{\pi})^{\top} (\Sigma^{\pi})^{-1} \nu_0^{\pi}} \cdot \big\| ( \Sigma^{\pi} )^{1/2}  \Sigma^{-1/2} \big\|_2, \end{aligned} \]
	and
	\[ \sum_{h=0}^{H} (H-h+1) \sqrt{ ( \nu_h^{\pi} )^{\top} \Sigma^{-1} \nu_h^{\pi} } \leq \frac{1}{2}(H+1)(H+2) \cdot \sqrt{ ( \nu_0^{\pi} )^{\top} ( \Sigma^{\pi} )^{-1} \nu_0^{\pi} } \cdot \big\| ( \Sigma^{\pi} )^{1/2} \Sigma^{-1/2} \big\|_2. \]
	It follows that under condition $N \geq 20 \kappa_1(2+\kappa_2)^2 \ln(12dH/\delta)C_1dH^3$,
	\begin{equation} \label{DeltaE1''} \Delta E_1 \leq \sqrt{(\nu_0^{\pi})^{\top} (\Sigma^{\pi})^{-1} \nu_0^{\pi}} \cdot \big\| ( \Sigma^{\pi} )^{1/2}  \Sigma^{-1/2} \big\|_2 \cdot \frac{3 \ln(12dH/\delta)\sqrt{C_1d}H^{2.5}}{N}. \end{equation}
	
	Note that \[ \begin{aligned} & \sqrt{(\nu_0^{\pi})^{\top} (\Sigma^{\pi})^{-1} \nu_0^{\pi}} \! \cdot \! \big\| (\Sigma^{\pi})^{1/2} \Sigma^{-1/2} \big\|_2 \leq \! \sqrt{(\nu_0^{\pi})^{\top} \Sigma^{-1} \nu_0^{\pi}} \! \cdot \! \big\| \Sigma^{1/2} (\Sigma^{\pi})^{-1/2} \big\|_2\big\| (\Sigma^{\pi})^{1/2} \Sigma^{-1/2} \big\|_2 = \! \sqrt{(\nu_0^{\pi})^{\top} \Sigma^{-1} \nu_0^{\pi}} \! \cdot \! \sqrt{\kappa_1}. \end{aligned} \]
	We plug \eqref{E1'}, \eqref{DeltaE1''}, \eqref{E2''} and \eqref{E3''} into \eqref{temp1'} and obtain
	\[ \begin{aligned} |v^{\pi} - \widehat{v}^{\pi}| \leq & \sum_{h=0}^{H} (H-h+1) \sqrt{( \nu_h^{\pi} )^{\top} \Sigma^{-1} \nu_h^{\pi}} \cdot \sqrt{\frac{\ln(12/\delta)}{2N}} \\ & + \sqrt{(\nu_0^{\pi})^{\top} (\Sigma^{\pi})^{-1}  \nu_0^{\pi}} \cdot \big\| (\Sigma^{\pi})^{1/2} \Sigma^{-1/2} \big\|_2 \cdot 15\sqrt{\kappa_1}(3+\kappa_2) \cdot \frac{\ln(12dH/\delta)C_1dH^{3.5}}{N} \\ \leq & \sum_{h=0}^{H} (H-h+1) \sqrt{( \nu_h^{\pi} )^{\top} \Sigma^{-1} \nu_h^{\pi}} \cdot \sqrt{\frac{\ln(12/\delta)}{2N}} + \sqrt{(\nu_0^{\pi})^{\top} \Sigma^{-1}  \nu_0^{\pi}} \cdot 15\kappa_1(3+\kappa_2) \cdot \frac{\ln(12dH/\delta)C_1dH^{3.5}}{N}. \end{aligned} \]
	Combining with Lemma \ref{Equivalence}, we finish the proof of \eqref{ub1}.
	
	Under condition $\phi(s,a)^{\top} \Sigma^{-1} \phi(s',a') \geq 0$ for any $(s,a), (s',a') \in \mathcal{X}$, one can apply \eqref{eqnE1_new} instead of \eqref{eqnE1} and derive a tighter upper bound for $|E_1|$ in \eqref{E1'}. We can then prove \eqref{ub2} in the same way.
\end{proof}



\subsection{Proof of Corollary \ref{cor:tabular}} \label{appendix:proof:cor}

\begin{proof}[Proof of Corollary \ref{cor:tabular}]
	1.
	In the tabular case, we have finite state space $\mathcal{S}$ and action space $\mathcal{A}$. The feature $\phi$ is the indicator function $\phi(s,a) = {\bf 1}_{s,a}$, where ${\bf 1}_{s,a}$ is a $(|\mathcal{S}||\mathcal{A}|)$-dimensional vector whose $(s,a)$-th entry is $1$ and others are $0$. In this case, $\Sigma$ is a diagonal matrix with nonnegative entries. Therefore, $\phi(s,a)^{\top} \Sigma^{-1} \phi(s',a') \geq 0$ for any $(s,a), (s',a') \in \mathcal{X}$. We can apply the upper bound \eqref{ub2} in Theorem \ref{UpperBound}.
	
	The mismatch term in \eqref{ub2} has a vector form,
	\[ (*) := \sup_{f \in \mathcal{Q}} \frac{\mathbb{E}^{\pi}\big[\sum_{h=0}^H (H-h+1) f(s_h,a_h) \, \big| \, s_0 \sim \xi_0 \big]}{\sqrt{\mathbb{E}\big[ \frac{1}{H} \sum_{h=0}^{H-1} f^2(s_{1,h},a_{1,h}) \big]}} = \sqrt{ \Bigg( \sum_{h=0}^{H} (H-h+1) \nu_h^{\pi} \Bigg)^{\top} \Sigma^{-1} \Bigg(\sum_{h=0}^{H} (H-h+1) \nu_h^{\pi} \Bigg)}, \]
	where $\Sigma\big((s,a), (s,a)\big) = \frac{1}{H} \sum_{h=0}^{H-1} \mathbb{P}(s_{1,h} = s, a_{1,h} = a) = \overline{\mu}(s,a)$ and $\nu_h^{\pi}\big((s,a)\big) = \mathbb{P}^{\pi}(s_h = s, a_h = a \, | \, s_0 \sim \xi_0)$. Note that by definition of $\mu^{\pi}$, $\sum_{h=0}^H (H-h+1) \nu_h^{\pi}\big((s,a)\big) = \sum_{h=0}^H (H-h+1) \mathbb{P}^{\pi}(s_h = s, a_h = a \, | \, s_0 \sim \xi_0) = \sum_{h=0}^H (H-h+1) \mu^{\pi} \big( (s,a) \big)$. Therefore, \[ (*) = \sqrt{\sum_{s,a} \frac{\big[\sum_{h=0}^H (H-h+1) \nu_h^{\pi}\big((s,a)\big) \big]^2}{\Sigma\big((s,a), (s,a)\big)}} = \sum_{h=0}^H (H-h+1) \sqrt{\sum_{s,a} \frac{ \mu^{\pi}\big((s,a)\big)^2}{\overline{\mu}(s,a)}} = \sum_{h=0}^H (H-h+1) \sqrt{ 1 + \chi^2(\mu^{\pi},\overline{\mu}) }, \] which implies \eqref{tabular}.
	
	2. When the tabular MDP is also time-inhomogeneous, 
	the first order error in \eqref{Edecompose} now has the form $E_1 = \frac{1}{K} \sum_{k=1}^K \sum_{h=0}^{H} e_{k,h}$, where
	\[ e_{k,h} := \frac{\mu_h^{\pi}(s_{k,h},a_{k,h})}{\overline{\mu}_h(s_{k,h},a_{k,h})} \Big( Q_h^{\pi}(s_{k,h}, a_{k,h}) - \big( r_{k,h}' + V_{h+1}^{\pi}(s_{k,h+1}) \big) \Big). \]
	
	Let $\mathcal{F}_{k,h}$ be the sigma algebra generated by $\boldsymbol{\tau}_1, \boldsymbol{\tau}_2, \ldots, \boldsymbol{\tau}_{k-1}$ and $\big(s_{k,0}, a_{k,0}, r_{k,0}', s_{k,1}, \ldots, s_{k,h-1}, a_{k,h-1}, r_{k,h-1}', s_{k,h}, a_{k,h}\big)$.
	Note that
	\[ {\rm Var}\big[ e_{k,h} \, \big| \, \mathcal{F}_{k,h} \big] = \bigg( \frac{\mu_h^{\pi}(s_{k,h},a_{k,h})}{\overline{\mu}_h(s_{k,h},a_{k,h})} \bigg)^2 {\rm Var}\big[ r_{k,h}' + V_{h+1}^{\pi}(s_{k,h+1}) \, \big| \, s_{k,h}, a_{k,h} \big], \]
	and
	\[ \sum_{h=0}^H \bigg( \frac{\mu_h^{\pi}(s_{k,h},a_{k,h})}{\overline{\mu}_h(s_{k,h},a_{k,h})} \bigg)^2 {\rm Var}\big[ r_{k,h}' + V_{h+1}^{\pi}(s_{k,h+1}) \, \big| \, s_{k,h}, a_{k,h} \big] \leq \sum_{h=0}^H \bigg( \max_{s,a} \frac{\mu_h^{\pi}(s,a)}{\overline{\mu}_h(s,a)} \bigg)^2 \cdot \frac{1}{4} (H-h+1)^2. \]
	We apply Hoeffding's inequality and derive that,
	with probability at least $1 - \delta/2$,
	\begin{equation} \label{var_hoeffding} \begin{aligned} \sum_{k=1}^K \sum_{h=0}^H {\rm Var}\big[ e_{k,h} \, \big| \, \mathcal{F}_{k,h} \big] \leq & K \cdot \mathbb{E} \Bigg[ \sum_{h=0}^H \bigg( \frac{\mu_h^{\pi}(s_{k,h},a_{k,h})}{\overline{\mu}_h(s_{k,h},a_{k,h})} \bigg)^2 {\rm Var}\big[ r_{k,h}' + V_{h+1}^{\pi}(s_{k,h+1}) \, \big| \, s_{k,h}, a_{k,h} \big] \Bigg] \\ & + \sqrt{\frac{K \ln(2/\delta)}{2}} \sum_{h=0}^H \bigg( \max_{s,a} \frac{\mu_h^{\pi}(s,a)}{\overline{\mu}_h(s,a)} \bigg)^2 \cdot \frac{1}{4} (H-h+1)^2. \end{aligned} \end{equation}
	Combining \eqref{var_hoeffding} with Freedman's inequality, we obtain that with probability at least $1-\delta$,
	\begin{equation} \label{E1_new} |E_1| \leq \sqrt{\frac{2\ln(4/\delta)}{K}} \sqrt{\mathbb{E} \Bigg[ \sum_{h=0}^H \bigg( \frac{\mu_h^{\pi}(s_{k,h},a_{k,h})}{\overline{\mu}_h(s_{k,h},a_{k,h})} \bigg)^2 {\rm Var}\big[ r_{k,h}' + V_{h+1}^{\pi}(s_{k,h+1}) \, \big| \, s_{k,h}, a_{k,h} \big] \Bigg]} + O(K^{-3/4}). \end{equation}
	We integrate \eqref{E1_new} with the existing results for $E_2$ and $E_3$, and obtain that with probability at least $1-\delta$,
	\begin{equation} \label{Err_new} |v^{\pi} - \widehat{v}^{\pi}| \leq \sqrt{\frac{2\ln(12/\delta)}{K}} \sqrt{\sum_{h=0}^H \sum_{s,a} \frac{\mu_h^{\pi}(s,a)^2}{\overline{\mu}_h(s,a)} {\rm Var}\big[ r' + V_{h+1}^{\pi}(s') \, \big| \, s,a \big]} + O(K^{-3/4}), \end{equation}
	which aligns with the result of Theorem 3.1 in \cite{yin2020asymptotically}.



\end{proof}

\section{Proof of Minimax Lower Bound} \label{appendix:proof:LowerBound}

\subsection{Preliminaries}

Given an MDP instance $M=(p,r)$, we construct an MDP instance $(\widetilde{p},r) \in \mathcal{N}(M)$ such that $p$ and $\tilde p$ are hard to distinguish based on $\mathcal{D}$ but have a gap in their values. Let
\begin{equation} \label{Deltaq}
	 \widetilde{p}(s' \, | \, s,a) := p(s' \, | \, s,a) - \phi(s,a)^{\top} \Delta q(s'), \qquad \Delta q(s') := {\bf x} \cdot \min_{s \in \mathcal{S}} p\big(s' \, \big| \, s, \overline{\pi}(s) \big) \cdot \big( \underline{p} \mathbbm{1}_{\overline{\mathcal{S}}}(s') - \overline{p} \mathbbm{1}_{\underline{\mathcal{S}}}(s') \big),
\end{equation}
where  ${\bf x} \in \mathbb{R}^d$ is a vector to be decided later. For any $(s,a) \in \mathcal{X}$, we have
\[ \begin{aligned} & \big\| \widetilde{p}(\cdot \, | \, s,a) - p(\cdot \, | \, s,a) \big\|_{\rm TV} = \frac{1}{2} \int_{\mathcal{S}} \big| \widetilde{p}(s' \, | \, s,a) - p(s' \, | \, s,a) \big| {\rm d}s' \\ = & \frac{1}{2} \bigg( \underline{p} \int_{\overline{\mathcal{S}}} \phi(s,a)^{\top} {\bf x} \cdot \min_{s \in \mathcal{S}} p\big(s' \, \big| \, s, \overline{\pi}(s) \big) {\rm d} s' + \overline{p} \int_{\underline{\mathcal{S}}} \phi(s,a)^{\top} {\bf x} \cdot \min_{s \in \mathcal{S}} p\big(s' \, \big| \, s, \overline{\pi}(s) \big) {\rm d} s' \bigg) \\ = & \phi(s,a)^{\top} {\bf x} \cdot \overline{p} \underline{p} \leq \sqrt{C_1d} \sqrt{{\bf x}^{\top} \Sigma {\bf x}} \cdot \overline{p} \underline{p}, \end{aligned} \]
where we have used $\phi(s,a)^{\top} \Sigma^{-1} \phi(s,a) \leq C_1 d$. If we take
\begin{equation} \label{<varepsilon} \sqrt{{\bf x}^{\top} \Sigma {\bf x}} \leq \frac{ \varepsilon}{\sqrt{C_1d} \cdot \overline{p} \underline{p}}, \end{equation}
then $(\widetilde{p},r) \in \mathcal{N}(M)$.
We denote by $\mathbb{P}$ (or $\widetilde{\mathbb{P}}$), $\mathbb{E}$ (or $\widetilde{\mathbb{E}}$) and $v^{\pi}$ (or $\widetilde{v}^{\pi}$) the probability, expectation and expected cumulative reward with respect to $p$ (or $\widetilde{p}$).

\subsection{Reduction to Likelihood Test}
If ${\bf x}$ is sufficiently small, it is hard for us to distinguish $\widetilde{p}$ and $p$ from observations $\mathcal{D}$. Recall that $\widehat{v}^{\pi}$ is an estimator based on $\mathcal{D}$. If $| v^{\pi} - \widetilde{v}^{\pi} | \geq \rho + \widetilde{\rho}$ for some $\rho, \widetilde{\rho} \geq 0$, then $\big|\widehat{v}^{\pi}(\mathcal{D}) - v^{\pi}\big| < \rho$ and $\big|\widehat{v}^{\pi}(\mathcal{D}) - \widetilde{v}^{\pi}\big| < \widetilde{\rho}$ cannot hold simultaneously. Therefore, $\widehat{v}^{\pi}$ has a large estimation error on either $p$ or $\widetilde{p}$. See Lemma \ref{outline} for a rigorous statement.

\begin{lemma} \label{outline}
	Define likelihood functions $\mathcal{L}(\mathcal{D}) = \prod_{k=1}^K \overline{\xi}_0(s_{k,0}) \prod_{h=0}^{H-1} \overline{\pi}(a_{k,h} \, | \, s_{k,h}) p( s_{k,h+1} \, | \, s_{k,h}, a_{k,h} )$ and $
	\widetilde{\mathcal{L}}(\mathcal{D}) = \prod_{k=1}^K \overline{\xi}_0(s_{k,0}) \prod_{h=0}^{H-1} \overline{\pi}(a_{k,h} \, | \, s_{k,h}) \widetilde{p}( s_{k,h+1} \, | \, s_{k,h}, a_{k,h} )$.
	If \[ \mathbb{P} \bigg( \frac{\widetilde{\mathcal{L}}(\mathcal{D})}{\mathcal{L}(\mathcal{D})} \geq \frac{1}{2} \bigg) \geq \frac{1}{2} \qquad \text{and} \qquad | v^{\pi} - \widetilde{v}^{\pi} | \geq \rho + \widetilde{\rho} \quad \text{for $\rho, \widetilde{\rho} \geq 0$}, \]
	then
	\[ \mathbb{P} \Big( \big| v^{\pi} - \widehat{v}^{\pi}(\mathcal{D}) \big| \geq \rho \Big) \geq \frac{1}{6} \qquad \text{or} \qquad \widetilde{\mathbb{P}} \Big( \big| \widetilde{v}^{\pi} - \widehat{v}^{\pi}(\mathcal{D}) \big| \geq \widetilde{\rho} \Big) \geq \frac{1}{6}. \]
\end{lemma}
\begin{proof}
	We prove Lemma \ref{outline} by a contradicton argument. We first assume
	\begin{equation} \label{prop} \mathbb{P} \Big( \big| v^{\pi} - \widehat{v}^{\pi}(\mathcal{D}) \big| < \rho \Big) > \frac{5}{6} \qquad \text{and} \qquad \widetilde{\mathbb{P}} \Big( \big| \widetilde{v}^{\pi} - \widehat{v}^{\pi}(\mathcal{D}) \big| < \widetilde{\rho} \Big) > \frac{5}{6} \end{equation} and show that this assumption leads to a contradiction. 
	
	When $| v^{\pi} - \widetilde{v}^{\pi} | \geq \rho + \widetilde{\rho}$, $| v^{\pi} - \widehat{v}^{\pi}(\mathcal{D}) | < \rho$ and $| \widetilde{v}^{\pi} - \widehat{v}^{\pi}(\mathcal{D}) | < \widetilde{\rho}$ cannot hold simultaneously for any $\mathcal{D}$.
	The assumption $\widetilde{\mathbb{P}}\big( |\widetilde{v}^{\pi} - \widehat{v}^{\pi}(\mathcal{D})| < \widetilde{\rho} \big) > \frac{5}{6}$ therefore implies \begin{equation} \label{l1} \widetilde{\mathbb{P}}\Big( \big|v^{\pi} - \widehat{v}^{\pi}(\mathcal{D})\big| < \rho \Big) < \frac{1}{6}. \end{equation} 
	We will see that \eqref{l1} is not compatible with the assumption $ \mathbb{P} \big( | v^{\pi} - \widehat{v}^{\pi}(\mathcal{D}) | < \rho \big) > \frac{5}{6}$ under condition $\mathbb{P} \Big( \frac{\widetilde{\mathcal{L}}(\mathcal{D})}{\mathcal{L}(\mathcal{D})} \geq \frac{1}{2} \Big) \geq \frac{1}{2}$.
	
	Define an event
	\[ \mathcal{E} = \bigg\{ \mathcal{D} \, \bigg| \, \big| v^{\pi} - \widehat{v}^{\pi}(\mathcal{D}) \big| < \rho, \frac{\widetilde{\mathcal{L}}(\mathcal{D})}{\mathcal{L}(\mathcal{D})} \geq \frac{1}{2} \bigg\}. \]
	The assumption $ \mathbb{P} \big( | v^{\pi} - \widehat{v}^{\pi}(\mathcal{D}) | < \rho \big) > \frac{5}{6}$ in \eqref{prop} and condition $\mathbb{P} \Big( \frac{\widetilde{\mathcal{L}}(\mathcal{D})}{\mathcal{L}(\mathcal{D})} \geq \frac{1}{2} \Big) \geq \frac{1}{2}$ ensures that
	\[ \mathbb{P} (\mathcal{E}) \geq \mathbb{P} \Big( \big| v^{\pi} - \widehat{v}^{\pi}(\mathcal{D}) \big| < \rho \Big) + \mathbb{P} \bigg(  \frac{\widetilde{\mathcal{L}}(\mathcal{D})}{\mathcal{L}(\mathcal{D})} \geq \frac{1}{2} \bigg) - 1 > \frac{1}{3}. \]
	We conduct a change of measure using the likelihood ratio inequality, and obtain $\widetilde{\mathbb{P}}(\mathcal{E}) \geq \frac{1}{2} \mathbb{P}(\mathcal{E}) > \frac{1}{6}$.
	It follows that $\widetilde{\mathbb{P}}\big( |v^{\pi} - \widehat{v}^{\pi}(\mathcal{D})| < \rho \big) \geq \widetilde{\mathbb{P}}(\mathcal{E}) > \frac{1}{6}$, which contradicts \eqref{l1}.
\end{proof}

In the following, we will analyze the likelihood ratio $\frac{\widetilde{\mathcal{L}}(\mathcal{D})}{\mathcal{L}(\mathcal{D})}$ and the difference $v^{\pi} - \widetilde{v}^{\pi}$, respectively.

\subsection{Concentration of the Likelihood Ratio}

We first present a preliminary result in Lemma \ref{lemma:term1} so as to simplify the analysis of likelihood ratio $\frac{\widetilde{\mathcal{L}}(\mathcal{D})}{\mathcal{L}(\mathcal{D})}$.
\begin{lemma} \label{lemma:term1}
	Suppose that {\it i.i.d.} random varibles $X_1, X_2, \ldots, X_K$ satisfies $0 \leq X_1 \leq C_0 \mathbb{E}[X_1]$. With probability at least $1 - \delta$, it holds that 
	\begin{equation} \label{term1} \frac{1}{K} \sum_{k=1}^K X_k \leq \mathbb{E}[X_1] \Bigg( 1 + \sqrt{\frac{2 \ln(1/\delta) C_0 }{K}} + \frac{2 \ln(1/\delta) C_0}{3K} \Bigg). \end{equation}
\end{lemma}
\begin{proof}
	We apply Bernstein's inequality to analyze $\frac{1}{K} \sum_{k=1}^K X_k$.
	Note that ${\rm Var}[X_1] \leq \mathbb{E}[X_1^2] \leq C_0 \mathbb{E}[X_1] \cdot \mathbb{E}[X_1] = C_0 \mathbb{E}[X_1]^2$.
	The Bernstein's inequality shows that for any $\varepsilon>0$,
	\[ \mathbb{P}\Bigg( \sum_{k=1}^K X_k \geq \varepsilon \Bigg) \leq \exp\bigg( - \frac{\varepsilon^2/2}{K \cdot C_0 \mathbb{E}[X_1]^2 + C_0 \mathbb{E}[X_1] \cdot \varepsilon /3} \bigg), \]
	which implies \eqref{term1}.
\end{proof}

Lemma \ref{lemma:lb1'} below shows that if we take an ${\bf x} \in \mathbb{R}^d$ in \eqref{Deltaq} such that $\sqrt{{\bf x}^{\top} \Sigma {\bf x}}$ is sufficiently small, then  $\mathbb{P} \Big( \frac{\widetilde{\mathcal{L}}(\mathcal{D})}{\mathcal{L}(\mathcal{D})} \geq \frac{1}{2} \Big) \leq \frac{1}{2}$. 

\begin{lemma}[Concentration of likelihood ratio] \label{lemma:lb1'}
	Suppose $\phi(s,a)^{\top} \Sigma^{-1} \phi(s,a) \leq C_1 d$ for all $(s,a) \in \mathcal{X}$. If $N \geq 12c^{-1}C_1dH$ and we take a vector ${\bf x} \in \mathbb{R}^d$ such that \begin{equation} \label{like<} \sqrt{{\bf x}^{\top} \Sigma {\bf x}} \leq \frac{1}{4 \sqrt{N} \sqrt{\overline{p} \underline{p} (\overline{p} + \underline{p})}}, \end{equation}
	then $\mathbb{P} \Big( \frac{\widetilde{\mathcal{L}}(\mathcal{D})}{\mathcal{L}(\mathcal{D})} \geq \frac{1}{2} \Big) \leq \frac{1}{2}$.
\end{lemma}
\begin{proof}
	We first calculate the log-likelihood ratio explicitly,
	\[ \begin{aligned} \ln \bigg( \frac{\widetilde{\mathcal{L}}(\mathcal{D})}{\mathcal{L}(\mathcal{D})} \bigg) = & \sum_{k=1}^K \sum_{h=0}^{H-1} \big( \ln \widetilde{p} (s_{k,h+1} \, | \, s_{k,h}, a_{k,h} ) - \ln p (s_{k,h+1} \, | \, s_{k,h}, a_{k,h} ) \big) \\ = & \sum_{k=1}^K \sum_{h=0}^{H-1} \ln\bigg(1 + \frac{ \widetilde{p}( s_{k,h+1} \, | \, s_{k,h}, a_{k,h} ) - p( s_{k,h+1} \, | \, s_{k,h}, a_{k,h} ) }{ p(s_{k,h+1} \, | \, s_{k,h}, a_{k,h}) } \bigg) \\ = & \sum_{k=1}^K \sum_{h=0}^{H-1} \ln\bigg(1 - \frac{ \phi(s_{k,h}, a_{k,h})^{\top} \Delta q(s_{k,h+1}) }{ p(s_{k,h+1} \, | \, s_{k,h}, a_{k,h}) } \bigg). \end{aligned} \]
	For the notational simplicity, we take
	\[ \Lambda_{k,h} = \frac{ \phi( s_{k,h-1}, a_{k,h-1} )^{\top} \Delta q (s_{k,h} ) }{ p( s_{k,h} \, | \, s_{k,h-1}, a_{k,h-1} )} = \frac{p( s_{k,h+1} \, | \, s_{k,h}, a_{k,h} ) - \widetilde{p}( s_{k,h+1} \, | \, s_{k,h}, a_{k,h} )}{p(s_{k,h+1} \, | \, s_{k,h}, a_{k,h})}, \]
	and let $\Lambda_n = \Lambda_{k,h}$ for $n = (k-1)H + h$.
	
	Under the assumption $\phi(s,a)^{\top} \Sigma^{-1} \phi(s,a) \leq C_1 d$ for all $(s,a) \in \mathcal{X}$, if we take $\sqrt{{\bf x}^{\top} \Sigma {\bf x}} \leq \frac{1}{4 \sqrt{N} \sqrt{\overline{p} \underline{p} (\overline{p} + \underline{p})}}$ and $N \geq 12c^{-1}C_1dH$,
	then $|\Lambda_n| \leq \frac{1}{2}$ for $n=1,2,\ldots,N$. It holds that $\ln(1 - \Lambda_n) \geq - \Lambda_n - \Lambda_n^2$. The log-likelihood ratio has a lower bound
	\begin{equation} \label{Lambda_decompose} \ln \bigg( \frac{\widetilde{\mathcal{L}}(\mathcal{D})}{\mathcal{L}(\mathcal{D})} \bigg) \geq - \sum_{n=1}^N \Lambda_n - \sum_{n=1}^N \Lambda_n^2. \end{equation}
	In the following, we analyze these two terms in \eqref{Lambda_decompose} separately.
	
	Consider the first term. Let $\{ \mathcal{F}_n \}_{n=1,2,\ldots,N}$ be a filtration where $\mathcal{F}_n$ is generated by $(s_0, a_0, s_0')$, $(s_1, a_1, s_1')$, $\ldots$, $(s_{n-1}, a_{n-1}, s_{n-1}')$ and $(s_n,a_n)$.
	It is easy to see that $\mathbb{E}[\Lambda_n \, | \, \mathcal{F}_n] = 0$, therefore, $\Lambda_n$ is a martingale difference.
	We apply Freedman's inequality to analyze $\sum_{n=1}^N \Lambda_n$.
	According to \eqref{Deltaq}, the conditional variance ${\rm Var}[\Lambda_n \, | \, \mathcal{F}_n]$ has the form
	\begin{equation} \label{like6} \begin{aligned} {\rm Var}\big[ \Lambda_n \, \big| \, \mathcal{F}_n \big] = & \mathbb{E} \big[ \Lambda_n^2 \, \big| \, \mathcal{F}_n \big] = \big( \phi(s_n,a_n)^{\top} {\bf x} \big)^2 \cdot \int_{\mathcal{S}} \frac{ \big( \min_{s \in \mathcal{S}} p\big(s' \, \big| \, s, \overline{\pi}(s) \big) \big)^2 }{ p( s_n' \, | \, s_n, a_n )} \cdot \big( \underline{p} \mathbbm{1}_{\overline{\mathcal{S}}}(s') - \overline{p} \mathbbm{1}_{\underline{\mathcal{S}}}(s') \big)^2 {\rm d} s' \\ \leq & \big( \phi(s_n,a_n)^{\top} {\bf x} \big)^2 \cdot \int_{\mathcal{S}} \min_{s \in \mathcal{S}} p\big(s' \, \big| \, s, \overline{\pi}(s) \big) \cdot \big( \underline{p}^2 \mathbbm{1}_{\overline{\mathcal{S}}}(s') + \overline{p}^2 \mathbbm{1}_{\underline{\mathcal{S}}}(s') \big) {\rm d} s' = \big( \phi(s_n,a_n)^{\top} {\bf x} \big)^2 \cdot \overline{p} \underline{p} (\overline{p} + \underline{p}). \end{aligned} \end{equation}
	It also holds that
	\[ \begin{aligned} |\Lambda_n| = & \bigg| \frac{ p(s_n' \, | \, s_n, a_n) - \widetilde{p}(s_n' \, | \, s_n, a_n) }{ p( s_n' \, | \, s_n, a_n )} \bigg| = \bigg| \frac{ \phi( s_n, a_n )^{\top} \Delta q (s_n' ) }{ p( s_n' \, | \, s_n, a_n )} \bigg| \\ = & \bigg| \phi( s_n, a_n )^{\top} {\bf x} \cdot \frac{ \min_{s \in \mathcal{S}} p\big(s' \, \big| \, s, \overline{\pi}(s) \big) }{ p( s_n' \, | \, s_n, a_n )} \cdot \big( \underline{p} \mathbbm{1}_{\overline{\mathcal{S}}}(s_n') - \overline{p} \mathbbm{1}_{\underline{\mathcal{S}}}(s_n') \big) \bigg| \\ \leq & \big| \phi( s_n, a_n )^{\top} {\bf x} \big| \cdot (\overline{p} \vee \underline{p}) \leq \sqrt{C_1 d} \cdot \sqrt{{\bf x}^{\top} \Sigma {\bf x}} \cdot (\overline{p} \vee \underline{p}). \end{aligned} \]
	Under assumption $\overline{p} \wedge \underline{p} \geq c$, we have $(\overline{p} \vee \underline{p})^2  \leq c^{-1} (\overline{p} \wedge \underline{p})(\overline{p} \vee \underline{p})^2 \leq c^{-1} (\overline{p} \wedge \underline{p})(\overline{p} \vee \underline{p}) \big( \overline{p} \wedge \underline{p} + \overline{p} \vee \underline{p} \big) = c^{-1} \overline{p} \underline{p}( \overline{p} + \underline{p} )$, therefore,
	\begin{equation} \label{like4}
	|\Lambda_n| \leq \sqrt{c^{-1} C_1 d} \cdot \sqrt{{\bf x}^{\top} \Sigma {\bf x}} \cdot \sqrt{\overline{p} \underline{p} (\overline{p} + \underline{p})}.
	\end{equation}
	
	Based on the estimations in \eqref{like6} and \eqref{like4}, we analyze the concentration of $\sum_{n=1}^N {\rm Var} [ \Lambda_n \,  \, \mathcal{F}_n ]$, and next derive an upper bound for $\sum_{n=1}^N \Lambda_n$.
	Note that
	\[ \mathbb{E} \Bigg[ \frac{1}{H} \sum_{h=0}^{H-1} \big( \phi(s_{k,h}, a_{k,h})^{\top} {\bf x} \big)^2 \Bigg] = {\bf x}^{\top} \Sigma {\bf x} \quad \text{and} \quad \Bigg| \frac{1}{H} \sum_{h=0}^{H-1} \big( \phi(s_{k,h}, a_{k,h})^{\top} {\bf x} \big)^2 \Bigg| \leq C_1d \cdot {\bf x}^{\top} \Sigma {\bf x}. \]
	We learn from Lemma \ref{lemma:term1} that with probability at least $\frac{7}{8}$,
	\[ \big( \phi(s_n, a_n)^{\top} {\bf x} \big)^2 \leq N \cdot {\bf x}^{\top} \Sigma {\bf x} \cdot \Bigg( 1 + \sqrt{\frac{6 \ln 2 \cdot C_1dH }{N}} + \frac{2 \ln 2 \cdot C_1dH}{N} \Bigg). \]
	Therefore,
	\begin{equation} \label{like3} \mathbb{P}\Bigg( \sum_{n=1}^N \! {\rm Var}\big[ \Lambda_n \, \big| \, \mathcal{F} \big] \leq \sigma^2 \Bigg) \leq \frac{1}{8}, \quad \text{where } \sigma^2 \! := \! N \cdot {\bf x}^{\top} \Sigma {\bf x} \cdot \overline{p} \underline{p} (\overline{p} + \underline{p}) \cdot \Bigg( \! 1 + \sqrt{\frac{6 \ln 2 \! \cdot \! C_1dH }{N}} + \frac{2 \ln 2 \!\cdot\! C_1dH}{N} \! \Bigg). \end{equation}
	Since $N \geq 12 c^{-1} C_1dH$ and $c \leq \frac{1}{2}$, we have $\frac{C_1dH}{N} \leq \frac{1}{24}$. It follows that $\sigma^2 \leq N \cdot {\bf x}^{\top} \Sigma {\bf x} \cdot \overline{p} \underline{p}(\overline{p} + \underline{p}) \cdot 1.46^2 \ln 2$.
	Additionally, Freedman's inequality implies
	\begin{equation} \label{like5} \mathbb{P} \Bigg( \sum_{n=1}^N \Lambda_n \geq 2\sqrt{\ln 2} \cdot \sigma + \frac{4}{3}\ln 2 \sqrt{c^{-1} C_1 d} \sqrt{{\bf x}^{\top} \Sigma {\bf x}} \sqrt{\overline{p} \underline{p}(\overline{p} + \underline{p})}, \ \sum_{n=1}^N {\rm Var} \big[ \Lambda_n \, \big| \, \mathcal{F}_n \big] \leq \sigma^2  \Bigg) \leq \frac{1}{4}, \end{equation}
	where we have used \eqref{like4}.
	The condition $N \geq 12 c^{-1}C_1dH$ ensures $c^{-1} C_1 d \leq \frac{N}{12} $. 
	We combine \eqref{like3} and \eqref{like5} and derive that with probability at least $\frac{5}{8}$,
	\begin{equation} \label{like_term1} \sum_{n=1}^N \Lambda_n \leq 2\sqrt{\ln 2} \cdot \sigma + \frac{4}{3}\ln 2 \sqrt{c^{-1} C_1 d} \sqrt{{\bf x}^{\top} \Sigma {\bf x}} \sqrt{\overline{p} \underline{p}(\overline{p} + \underline{p})} 
	\leq \sqrt{N} \sqrt{{\bf x}^{\top} \Sigma {\bf x}}\sqrt{\overline{p} \underline{p}(\overline{p} + \underline{p})} \cdot 3.31 \ln 2. \end{equation}
	
	As for the second term $\sum_{n=1}^N \Lambda_n^2$ in \eqref{Lambda_decompose}, the estimations \eqref{like6} and \eqref{like4} suggest that
	\[ \mathbb{E}\Bigg[ \frac{1}{H} \sum_{h=0}^{H-1} \Lambda_{k,h}^2 \Bigg] \leq {\bf x}^{\top} \Sigma {\bf x} \cdot \overline{p} \underline{p} (\overline{p} + \underline{p}) \quad \text{and} \quad \Bigg| \frac{1}{H} \sum_{h=0}^{H-1} \Lambda_{k,h}^2 \Bigg| \leq c^{-1}C_1 d \cdot {\bf x}^{\top} \Sigma {\bf x} \cdot \overline{p} \underline{p} (\overline{p} + \underline{p}). \]
	It follows from Lemma \ref{lemma:term1} that with probability at least $\frac{7}{8}$,
	\begin{equation} \label{like2} \sum_{n=1}^N \Lambda_n^2 \leq N \cdot {\bf x}^{\top} \Sigma {\bf x} \cdot \overline{p} \underline{p} (\overline{p} + \underline{p}) \Bigg( 1 + \sqrt{\frac{6 \ln 2 \cdot c^{-1} C_1dH}{N}} + \frac{2 \ln 2 \cdot c^{-1} C_1 dH }{N} \Bigg). \end{equation}
	If $N \geq 12 c^{-1}C_1dH$, then we can reduce \eqref{like2} to
	\begin{equation} \label{like_term2} \sum_{n=1}^N \Lambda_n^2 \leq N \cdot {\bf x}^{\top} \Sigma {\bf x} \cdot \overline{p} \underline{p} (\overline{p} + \underline{p}) \cdot 2.46 \ln 2. \end{equation}
	
	We now use the condition $\sqrt{{\bf x}^{\top} \Sigma {\bf x}} \leq \frac{1}{4 \sqrt{N} \sqrt{\overline{p} \underline{p} (\overline{p} + \underline{p})}}$.
	By union bound, \eqref{like_term1} and \eqref{like_term2} imply that with probability at least $\frac{1}{2}$,
	\[ \ln \bigg( \frac{\widetilde{\mathcal{L}}(\mathcal{D})}{\mathcal{L}(\mathcal{D})} \bigg) \geq - \sum_{n=1}^N \Lambda_n - \sum_{n=1}^N \Lambda_n^2 \geq - \frac{3.31 \ln 2}{4} - \frac{2.46 \ln 2}{16} > - \ln 2, \]
	or equivalently, $\mathbb{P}\Big(  \frac{\widetilde{\mathcal{L}}(\mathcal{D})}{\mathcal{L}(\mathcal{D})} \geq \frac{1}{2} \Big) \geq \frac{1}{2}$.
	
\end{proof}

\subsection{Calculating the Gap between Values}

\begin{lemma} \label{lemma:v-v}
Let $\widetilde{\nu}_h^{\pi} = \widetilde{\mathbb{E}}^{\pi} \big[ \phi(s_h, a_h) \, \big| \, s_0 \sim \xi_0 \big] \in \mathbb{R}^d$ for $h=0,1,\ldots,H-1$.
\begin{equation} \label{v-v} v^{\pi} - \widetilde{v}^{\pi} \geq \frac{1}{2} \overline{p} \underline{p} \cdot \Bigg( \sum_{h=0}^{H-1} (H-h) \widetilde{\nu}_h^{\pi} \Bigg)^{\top} {\bf x}. \end{equation}
\end{lemma}

\begin{proof}
	Let $\mathcal{P}^{\pi}$ and $\widetilde{\mathcal{P}}^{\pi}$ be the conditional mean operators that correspond to transition kernels $p$ and $\widetilde{p}$. Similar to \eqref{Q_decompose}, we have
	\[ Q_0^{\pi} - \widetilde{Q}_0^{\pi} = \sum_{h=0}^{H-1} \big( \widetilde{\mathcal{P}}^{\pi} \big)^h \big( \mathcal{P}^{\pi} - \widetilde{\mathcal{P}}^{\pi} \big) Q_{h+1}^{\pi}. \]
	
	We first analyze $( \mathcal{P}^{\pi} - \widetilde{\mathcal{P}}^{\pi} ) Q_{h+1}^{\pi} $. Note that $\mathcal{P}^{\pi}Q_{h+1}^{\pi}(s,a) = \int_{\mathcal{S}} V_{h+1}^{\pi}(s') p(s' \, | \, s,a) {\rm d} s'$ and $\widetilde{\mathcal{P}}^{\pi}Q_{h+1}^{\pi}(s,a) = \int_{\mathcal{S}} V_{h+1}^{\pi}(s') \widetilde{p}(s' \, | \, s,a) {\rm d} s'$. Therefore,
	\[ \Big(\big( \mathcal{P}^{\pi} - \widetilde{\mathcal{P}}^{\pi} \big) Q_{h+1}^{\pi} \Big)(s,a) = \int_{\mathcal{S}} \big( p(s' \, | \, s,a) - \widetilde{p}(s' \, | \, s,a) \big) V_{h+1}^{\pi}(s') {\rm d}s'. \]
	According to \eqref{Deltaq},
	\begin{equation} \label{v-v2} \begin{aligned} & \Big(\big( \mathcal{P}^{\pi} - \widetilde{\mathcal{P}}^{\pi} \big) Q_{h+1}^{\pi} \Big)(s,a) = \int_{\mathcal{S}} \phi(s,a)^{\top} \Delta q(s') \cdot V_{h+1}^{\pi}(s') {\rm d} s' \\ = & \int_{\mathcal{S}} \phi(s,a)^{\top} {\bf x}  \cdot \min_{s \in \mathcal{S}} p\big(s' \, \big| \, s, \overline{\pi}(s) \big) \cdot \big( \underline{p} \mathbbm{1}_{\overline{\mathcal{S}}}(s') - \overline{p} \mathbbm{1}_{\underline{\mathcal{S}}}(s') \big) \cdot V_{h+1}^{\pi}(s') {\rm d} s' \\  = & \phi(s,a)^{\top} {\bf x} \cdot \bigg( \underline{p} \int_{\overline{\mathcal{S}}} \min_{s \in \mathcal{S}} p\big(s' \, \big| \, s, \overline{\pi}(s) \big) \cdot V_{h+1}^{\pi}(s') {\rm d}s' - \overline{p} \int_{\underline{\mathcal{S}}} \min_{s \in \mathcal{S}} p\big(s' \, \big| \, s, \overline{\pi}(s) \big) \cdot V_{h+1}^{\pi}(s') {\rm d}s' \bigg). \end{aligned} \end{equation}
	Since by definition, $V_h^{\pi}(s) \geq \frac{3}{4}(H-h+1)$ for $s \in \overline{\mathcal{S}}$, $V_h^{\pi}(s) \leq \frac{1}{4}(H-h+1)$ for $s \in \underline{\mathcal{S}}$, 
	we have
	\begin{equation} \label{v-v1} \begin{aligned} & \int_{\overline{\mathcal{S}}} \min_{s \in \mathcal{S}} p\big(s' \, \big| \, s, \overline{\pi}(s) \big) \cdot V_{h+1}^{\pi}(s') {\rm d}s' \geq \int_{\overline{\mathcal{S}}} \min_{s \in \mathcal{S}} p\big(s' \, \big| \, s, \overline{\pi}(s) \big) \cdot \frac{3}{4}(H-h) {\rm d}s' = \frac{3}{4} (H-h) \overline{p}, \\ & \int_{\underline{\mathcal{S}}} \min_{s \in \mathcal{S}} p\big(s' \, \big| \, s, \overline{\pi}(s) \big) \cdot V_{h+1}^{\pi}(s') {\rm d}s' \leq \int_{\overline{\mathcal{S}}} \min_{s \in \mathcal{S}} p\big(s' \, \big| \, s, \overline{\pi}(s) \big) \cdot \frac{1}{4}(H-h) {\rm d}s' = \frac{1}{4} (H-h) \underline{p}. \end{aligned} \end{equation}
	Plugging \eqref{v-v1} into \eqref{v-v2} yields
	\begin{equation} \label{v-v3} \Big(\big( \mathcal{P}^{\pi} - \widetilde{\mathcal{P}}^{\pi} \big) Q_{h+1}^{\pi} \Big)(s,a) \geq \phi(s,a)^{\top} {\bf x} \cdot \bigg( \frac{3}{4}(H-h) \overline{p} \underline{p} - \frac{1}{4}(H-h) \overline{p} \underline{p} \bigg) = \phi(s,a)^{\top} {\bf x} \cdot \frac{1}{2} (H-h) \overline{p} \underline{p}.  \end{equation}
	
	The inequality \eqref{v-v3} further implies that
	\begin{equation} \label{v-v4} \Big( \big( \widetilde{\mathcal{P}}^{\pi} \big)^h \big( \mathcal{P}^{\pi} - \widetilde{\mathcal{P}}^{\pi} \big) Q_{h+1}^{\pi} \Big)(s,a) \geq \widetilde{\mathbb{E}}^{\pi} \big[ \phi(s_h,a_h)^{\top} {\bf x} \, \big| \, s_0 = s, a_0 = a \big] \cdot \frac{1}{2} (H-h) \overline{p} \underline{p}. \end{equation}
	Since
	\[ v^{\pi} - \widetilde{v}^{\pi} = \sum_{h=0}^{H-1} \int_{\mathcal{X}} \Big( \big( \widetilde{\mathcal{P}}^{\pi} \big)^h \big( \mathcal{P}^{\pi} - \widetilde{\mathcal{P}}^{\pi} \big) Q_{h+1}^{\pi} \Big)(s,a) \cdot \xi_0(s) \pi(a \, | \, s) {\rm d}s {\rm d}a, \]
	we apply \eqref{v-v4} and derive
	\begin{equation} \begin{aligned} v^{\pi} - \widetilde{v}^{\pi} \geq & \sum_{h=0}^{H-1} \int_{\mathcal{X}} \widetilde{\mathbb{E}}^{\pi} \big[ \phi(s_h,a_h)^{\top} {\bf x} \, \big| \, s_0 = s, a_0 = a \big] \cdot \xi_0(s) \pi(a \, | \, s) {\rm d}s {\rm d}a \cdot \frac{1}{2} (H-h) \overline{p} \underline{p} \\ = & \sum_{h=0}^{H-1} (\widetilde{\nu}_h^{\pi})^{\top} {\bf x} \cdot \frac{1}{2} (H-h) \overline{p} \underline{p} = \frac{1}{2} \overline{p} \underline{p} \cdot \Bigg( \sum_{h=0}^{H-1} (H-h) \widetilde{\nu}_h^{\pi} \Bigg)^{\top} {\bf x}, \end{aligned} \end{equation}
	which completes the proof.
	
\end{proof}


\subsection{Completing the Proof of Theorem \ref{theorem:lb}}

For the notational convenience, let
\[ \boldsymbol{\nu}^{\pi} := \sum_{h=0}^{H-1} (H-h) \nu_h^{\pi} \qquad \text{and} \qquad \widetilde{\boldsymbol{\nu}}^{\pi} := \sum_{h=0}^{H-1} (H-h) \widetilde{\nu}_h^{\pi}. \]
When $\widetilde{p} \approx p$, we have $\widetilde{\nu}_h^{\pi} \approx \nu_h^{\pi}$ for $h=0,1,\ldots,H-1$. According to Lemma \ref{lemma:v-v}, the value gap in \eqref{like<} satisfies
\begin{equation} \label{gap} \frac{1}{2} \overline{p} \underline{p} \cdot (\widetilde{\boldsymbol{\nu}}^{\pi})^{\top} {\bf x} \approx \frac{1}{2} \overline{p} \underline{p} \cdot (\boldsymbol{\nu}^{\pi})^{\top} {\bf x}. \end{equation}
We construct ${\bf x} \in \mathbb{R}^d$ such that \eqref{like<} holds and the (approximate) value gap in \eqref{gap} is maximized. More explicitly, we take ${\bf x} = {\bf x}^*$ that solves the following optimization problem,
\[  \text{maximize}_{{\bf x} \in \mathbb{R}^d} \quad \frac{1}{2} \overline{p} \underline{p} \cdot (\boldsymbol{\nu}^{\pi})^{\top} {\bf x}, \qquad \text{subject to} \quad \sqrt{{\bf x}^{\top} \Sigma {\bf x}} \leq \frac{1}{4 \sqrt{N} \sqrt{\overline{p} \underline{p} (\overline{p} + \underline{p})}}. \]
${\bf x}^*$ has a closed form,
\begin{equation} \label{x} {\bf x}^* := \frac{\Sigma^{-1} \boldsymbol{\nu}^{\pi}}{4\sqrt{N} \sqrt{(\boldsymbol{\nu}^{\pi})^{\top} \Sigma^{-1} \boldsymbol{\nu}^{\pi}} \sqrt{\overline{p} \underline{p} (\overline{p} + \underline{p})}}. \end{equation}


	We integrate the pieces in Lemmas \ref{outline}, \ref{lemma:lb1'} and \ref{lemma:v-v} to complete the proof of Theorem \ref{theorem:lb}.

	\begin{proof}[Proof of Theorem \ref{theorem:lb}]
	We construct a perturbed instance $\widetilde{p}$ according to \eqref{Deltaq}, where ${\bf x}$ is chosen to be ${\bf x}^*$ in \eqref{x}.
	In this case, $\sqrt{({\bf x}^*)^{\top} \Sigma {\bf x}^*} = \frac{1}{4\sqrt{N}\sqrt{\overline{p}\underline{p}(\overline{p}+\underline{p})}}$. If we take $\varepsilon \geq \sqrt{\frac{cC_1d}{32N}}$, then \eqref{<varepsilon} holds and $\big\| \widetilde{p}(\cdot \, | \, s,a) - p(\cdot \, | \, s,a) \big\|_{\rm TV} \leq \varepsilon$ for all $s,a \in \mathcal{X}$. Therefore, the perturbed instance $(\widetilde{p},r) \in \mathcal{N}(M)$.
	
	Lemma \ref{lemma:lb1'} guarantees that when $N \geq 12c^{-1}C_1dH$, we have $\mathbb{P}\Big( \frac{\widetilde{\mathcal{L}}(\mathcal{D})}{\mathcal{L}(\mathcal{D})} \geq \frac{1}{2} \Big) \leq \frac{1}{2}$. 
	Additionally, according to Lemma \ref{lemma:v-v}, the value gap $v^{\pi} - \widetilde{v}^{\pi}$ satisfies
	\begin{equation} \label{GAP} v^{\pi} - \widetilde{v}^{\pi} \geq \frac{1}{2} \overline{p} \underline{p} \cdot (\widetilde{\boldsymbol{\nu}}^{\pi})^{\top} {\bf x}^* = \frac{1}{8\sqrt{N} } \cdot \frac{(\widetilde{\boldsymbol{\nu}}^{\pi})^{\top} \Sigma^{-1} \boldsymbol{\nu}^{\pi}}{\sqrt{(\boldsymbol{\nu}^{\pi})^{\top} \Sigma^{-1} \boldsymbol{\nu}^{\pi}}} \cdot \sqrt{\frac{\overline{p} \underline{p}}{\overline{p} + \underline{p}}} \geq \frac{1}{8\sqrt{N} } \cdot \frac{(\widetilde{\boldsymbol{\nu}}^{\pi})^{\top} \Sigma^{-1} \boldsymbol{\nu}^{\pi}}{\sqrt{(\boldsymbol{\nu}^{\pi})^{\top} \Sigma^{-1} \boldsymbol{\nu}^{\pi}}} \cdot \sqrt{\frac{c}{2}}. \end{equation}
	When $\widetilde{p}$ and $p$ are close enough, {\it i.e.} $N$ is sufficiently large in our instance, we have $\widetilde{\Sigma} \approx \Sigma$ for $\widetilde{\Sigma} := \widetilde{\mathbb{E}} \big[ \frac{1}{H} \sum_{h=0}^{H-1} \phi(s_{1,h}, a_{1,h}) \phi(s_{1,h}, a_{1,h})^{\top} \big]$, $\widetilde{\nu}_h^{\pi} \approx \nu_h^{\pi}$ for $h=0,1,\ldots,H-1$, and $\frac{(\widetilde{\boldsymbol{\nu}}^{\pi})^{\top} \Sigma^{-1} \boldsymbol{\nu}^{\pi}}{\sqrt{(\boldsymbol{\nu}^{\pi})^{\top} \Sigma^{-1} \boldsymbol{\nu}^{\pi}}} \approx \frac{1}{2} \big(\sqrt{(\boldsymbol{\nu}^{\pi})^{\top} \Sigma^{-1} \boldsymbol{\nu}^{\pi}} + \sqrt{(\widetilde{\boldsymbol{\nu}}^{\pi})^{\top} \widetilde{\Sigma}^{-1} \widetilde{\boldsymbol{\nu}}^{\pi}} \big)$.
	In particular, when $N$ is sufficiently large, it holds that
	\begin{equation} \label{v-v>} v^{\pi} - \widetilde{v}^{\pi} \geq \rho + \widetilde{\rho} \qquad \text{for }\rho := \frac{\sqrt{c}}{24\sqrt{N}} \sqrt{ (\boldsymbol{\nu}^{\pi})^{\top} \Sigma^{-1} \boldsymbol{\nu}^{\pi} }, \quad \widetilde{\rho} := \frac{\sqrt{c}}{24\sqrt{N}} \sqrt{ (\widetilde{\boldsymbol{\nu}}^{\pi})^{\top} \widetilde{\Sigma}^{-1} \widetilde{\boldsymbol{\nu}}^{\pi} }. \end{equation}
	One can then conclude from Lemma \ref{outline} that $\mathbb{P} \big( \big| v^{\pi} - \widehat{v}^{\pi}(\mathcal{D}) \big| \geq \rho \big) \geq \frac{1}{6}$ or $\widetilde{\mathbb{P}} \big( \big| \widetilde{v}^{\pi} - \widehat{v}^{\pi}(\mathcal{D}) \big| \geq \widetilde{\rho} \big) \geq \frac{1}{6}$, whicn further implies the minimax lower bound \eqref{lowerbound}.
	\end{proof}
	
	\vspace{0.3cm}
	
	\noindent{\bf Remark} (Requirement on sample size $N$){\bf.}  If $\widetilde{\Sigma} \succeq (1-c_1) \Sigma$ and $\sqrt{(\widetilde{\boldsymbol{\nu}}^{\pi} - \boldsymbol{\nu}^{\pi})^{\top} \Sigma^{-1} (\widetilde{\boldsymbol{\nu}}^{\pi} - \boldsymbol{\nu}^{\pi})} \leq c_2\sqrt{(\boldsymbol{\nu}^{\pi})^{\top} \Sigma^{-1} \boldsymbol{\nu}^{\pi}}$ for some constants $c_1, c_2 \in (0,1)$, then by routine calculations, one can show that $\frac{(\widetilde{\boldsymbol{\nu}}^{\pi})^{\top} \Sigma^{-1} \boldsymbol{\nu}^{\pi}}{\sqrt{(\boldsymbol{\nu}^{\pi})^{\top} \Sigma^{-1} \boldsymbol{\nu}^{\pi}}} \geq \frac{c_3}{2}\big(\sqrt{(\boldsymbol{\nu}^{\pi})^{\top} \Sigma^{-1} \boldsymbol{\nu}^{\pi}} + \sqrt{(\widetilde{\boldsymbol{\nu}}^{\pi})^{\top} \widetilde{\Sigma}^{-1} \widetilde{\boldsymbol{\nu}}^{\pi}} \big)$ for some $c_3 \in (0,1)$. We can analyze $\sqrt{(\widetilde{\boldsymbol{\nu}}^{\pi} - \boldsymbol{\nu}^{\pi})^{\top} \Sigma^{-1} (\widetilde{\boldsymbol{\nu}}^{\pi} - \boldsymbol{\nu}^{\pi})}$ in a way similar to the estimation of high-order term $E_2$ in the upper bound. In this way, we can show that \eqref{v-v>} holds when
		\[ N \geq 200 \kappa_1 c^{-3} C_1 d H^2 \cdot \frac{\sum_{h=0}^{H-1}(H-h)(\nu_0^{\pi})^{\top} \Sigma^{-1} \nu_0^{\pi}}{(\boldsymbol{\nu}^{\pi})^{\top} \Sigma^{-1} \boldsymbol{\nu}^{\pi}}. \]
	If we further propose a mild assumption that the all-one function ${\bf 1}(s,a) = 1, \forall (s,a) \in \mathcal{X}$ belongs to $\mathcal{Q}$, then $\sqrt{(\boldsymbol{\nu}^{\pi})^{\top} \Sigma^{-1} \boldsymbol{\nu}^{\pi}} \geq \sum_{h=0}^{H-1}(H-h)$. Therefore, it is sufficient to have
	\[ N \geq 200 \kappa_1 c^{-3} C_1 d H^2 \cdot (\nu_0^{\pi})^{\top} \Sigma^{-1} \nu_0^{\pi}. \]

\section{Proof of Data-Dependent Confidence Bound} \label{appendix:lemma:UpperBound'}

\begin{lemma} \label{lemma:Q_decompose0}
	Under Assumption \ref{Q_class}, it always holds that
	\begin{equation} \label{Q_decompose0} |v^{\pi} - \widehat{v}^{\pi}| \leq \sum_{h=0}^H \sqrt{(\widehat{\nu}_h^{\pi})^{\top} \widehat{\Sigma}^{-1} \widehat{\nu}_h^{\pi}} \sqrt{\big( \widehat{W}_h^{\pi} - w_h^{\pi} \big)^{\top} \widehat{\Sigma} \big( \widehat{W}_h^{\pi} - w_h^{\pi} \big)}, \end{equation}
	where $w_h^{\pi} \in \mathbb{R}^d$ satisfies $Q_h^{\pi}(s,a) = \phi(s,a)^{\top}w_h^{\pi}$ and $\widehat{W}_h^{\pi} := \widehat{R} + \widehat{M}^{\pi}w_{h+1}^{\pi}$.
\end{lemma}
\begin{proof}
Recall that \eqref{Q_decompose} shows
$Q_0^{\pi} - \widehat{Q}_0^{\pi} = \! \sum_{h=0}^{H} \! \big( \widehat{\mathcal{P}}^{\pi} \big)^h \big( Q_h^{\pi} - ( \widehat{r}+ \widehat{\mathcal{P}}^{\pi}Q_{h+1}^{\pi} ) \big)$.
We apply the definition of $\widehat{r}$ in \eqref{def_hatr} and the property of $\widehat{\mathcal{P}}^{\pi}$ in \eqref{hatPpi}, and derive
\begin{equation} \label{Q_decompose1} \big(\widehat{r} + \widehat{\mathcal{P}}^{\pi} Q_{h+1}^{\pi}  \big)(s,a) = \phi(s,a)^{\top} \big( \widehat{R} + \widehat{M}^{\pi} w_{h+1}^{\pi} \big) = \phi(s,a)^{\top} \widehat{W}_h^{\pi}. \end{equation}
By definition, we also have \begin{equation} \label{Q_decompose2} Q_h^{\pi}(s,a) = \phi(s,a)^{\top} w_h^{\pi}. \end{equation}
Plugging \eqref{Q_decompose1} and \eqref{Q_decompose2} into \eqref{Q_decompose} yields $Q_0^{\pi}(s,a) - \widehat{Q}_0^{\pi}(s,a) = \sum_{h=0}^{H} \big( \widehat{\mathcal{P}}^{\pi} \big)^h \phi (s,a)^{\top}\big(w_h^{\pi} - \widehat{W}_h^{\pi}\big)$. Since $\widehat{\mathcal{P}}^{\pi} \phi (s,a) = (\widehat{M}^{\pi})^{\top} \phi(s,a)$, we have $\big(\widehat{\mathcal{P}}^{\pi}\big)^h \phi (s,a) = \big( (\widehat{M}^{\pi})^{\top} \big)^h \phi(s,a)$, therefore, \[ \int_{\mathcal{X}} \big(\widehat{\mathcal{P}}^{\pi}\big)^h \phi (s,a) \xi_0(s) \pi(a \, | \, s) {\rm d}s {\rm d}a = \int_{\mathcal{X}} \big( (\widehat{M}^{\pi})^{\top} \big)^h \phi(s,a) \xi_0(s) \pi(a \, | \, s) {\rm d}s {\rm d}a = \big( (\widehat{M}^{\pi})^{\top} \big)^h \nu_0^{\pi} = \widehat{\nu}_h^{\pi}.\]
It follows that
\[ \begin{aligned} v^{\pi} - \widehat{v}^{\pi} = \int_{\mathcal{X}} \!\! \big(Q_0^{\pi}(s,a) - \widehat{Q}_0^{\pi}(s,a)\big) \xi_0(s) \pi(a \, | \, s) {\rm d} s {\rm d}a = \sum_{h=0}^{H-1} (\widehat{\nu}_h^{\pi})^{\top} \big( w_{h}^{\pi} - \widehat{W}_{h}^{\pi} \big), \end{aligned} \]
which further implies \eqref{Q_decompose0}.
\end{proof}

According to \eqref{vector_recursion},
\[ \begin{aligned} \widehat{W}_h^{\pi} = & \widehat{R} + \widehat{M}^{\pi} w_{h+1}^{\pi} = \widehat{\Sigma}^{-1} \sum_{n=1}^N r_n' \phi(s_n,a_n) + \widehat{\Sigma}^{-1} \sum_{n=1}^N \phi(s_n,a_n) \phi^{\pi}(s_n')^{\top} w_{h+1}^{\pi} \\ = & \widehat{\Sigma}^{-1} \sum_{n=1}^N \phi(s_n,a_n) \big( r_n' + \phi^{\pi}(s_n')^{\top} w_{h+1}^{\pi} \big) = \widehat{\Sigma}^{-1} \sum_{n=1}^N \phi(s_n,a_n) \big(r_n'+ V_{h+1}^{\pi}(s_n')\big). \end{aligned} \]
In the proof of Theorem \ref{UpperBound'}, we define
\[ \widehat{\Sigma}_n := \lambda I + \sum_{t = 1}^n \phi(s_t,a_t) \phi(s_t,a_t)^{\top} \qquad \text{and} \qquad \widehat{W}_{h,n}^{\pi} := \widehat{\Sigma}_n^{-1} \sum_{t=1}^n \phi(s_t,a_t) \big( r_t' + V_{h+1}^{\pi}(s_t') \big) \] 
for $n=0,1,\ldots,N$. Note that $\widehat{\Sigma} = \widehat{\Sigma}_N$ and $\widehat{W}_h^{\pi} = \widehat{W}_{h,N}^{\pi}$.
In the following, we analyze the concentration of
\[ \Theta_{h,n} := \big( \widehat{W}_{h,n}^{\pi} - w_h^{\pi} \big)^{\top} \widehat{\Sigma}_n \big( \widehat{W}_{h,n}^{\pi} - w_h^{\pi} \big), \qquad n = 0,1,\ldots,N. \]

Parallel to Lemma 12 in \cite{dani2008stochastic} and Lemma 11 in \cite{yang2019reinforcement}, we have the following Lemma \ref{Theta_decompose}.
\begin{lemma} \label{Theta_decompose}
	For all $n = 0,1,2,\ldots,N$,
	\[ \begin{aligned} \Theta_{h,n} \leq & \lambda \| w_h^{\pi} \|_2^2 + \sum_{t=1}^n  \underbrace{ 2 \big(r_t' + V_{h+1}^{\pi}(s_t')-Q_h^{\pi}(s_t, a_t) \big) \frac{\phi(s_t,a_t)^{\top}\big(\widehat{W}_{h,t-1}^{\pi}-w_h^{\pi}\big)}{1 + \phi(s_t,a_t)^{\top} \widehat{\Sigma}_{t-1}^{-1}\phi(s_t,a_t)}}_{\alpha_{h,t}} \\ & + \sum_{t=1}^n \underbrace{\big(r_t' + V_{h+1}^{\pi}(s_t')-Q_h^{\pi}(s_t, a_t) \big)^2 \frac{\phi(s_t,a_t)^{\top} \widehat{\Sigma}_{t-1}^{-1}\phi(s_t,a_t)}{1 + \phi(s_t,a_t)^{\top} \widehat{\Sigma}_{t-1}^{-1}\phi(s_t,a_t)}}_{\beta_{h,t}}. \end{aligned} \]
\end{lemma}


Define a filtration $\{ \mathcal{F}_n \}_{n=1}^N$ where $\mathcal{F}_n$ is generated by $(s_1, a_1, s_1',r_1'), (s_2, a_2, s_2',r_2'), \ldots, (s_{n-1}, a_{n-1}, s_{n-1}',r_{n-1}')$ and $(s_n,a_n)$. Lemma \ref{Theta_decompose} suggests that $\Theta_{h,n}$ is upper bounded by a martingale $\sum_{t=1}^n \alpha_{h,t}$ plus the sum of shift terms $\sum_{t=1}^n \beta_{h,t}$. Under the assumption \[ \text{$ \| \phi(s,a) \|_2 \leq 1$ for all $(s,a) \in \mathcal{X}$}, \] we utilize the following Lemma \ref{det} to control $\sum_{t=1}^n \beta_{h,t}$.
\begin{lemma} \label{det}
	$d \ln \lambda + \sum_{t=1}^n \ln\big(1+\phi(s_t,a_t)^{\top} \widehat{\Sigma}_{t-1}^{-1} \phi(s_t,a_t)\big) = \ln \det(\widehat{\Sigma}_n) \leq d \ln(\lambda+n/d)$.
\end{lemma}
\begin{proof}
	Identical to Lemma 9 in \cite{dani2008stochastic} and Lemma 10 in \cite{yang2019reinforcement}.
\end{proof}
Note that $r_n'+V_{h+1}^{\pi}(s_n') \in [0,H-h+1]$ for all $s \in \mathcal{S}$. Therefore,
\begin{equation} \label{shift} \begin{aligned} \sum_{t=1}^n \beta_{h,t} \leq & (H-h+1)^2 \sum_{t=1}^n \frac{\phi(s_t,a_t)^{\top} \widehat{\Sigma}_{t-1}^{-1}\phi(s_t,a_t)}{1 + \phi(s_t,a_t)^{\top} \widehat{\Sigma}_{t-1}^{-1}\phi(s_t,a_t)} \\ \leq & (H-h+1)^2 \sum_{t=1}^n \ln \big( 1 + \phi(s_t,a_t)^{\top} \widehat{\Sigma}_{t-1}^{-1} \phi(s_t,a_t) \big) \\ \leq & d(H-h+1)^2 \ln\Big(1 + \frac{n}{\lambda d}\Big), \end{aligned} \end{equation}
where we have used the inequality $\frac{x^2}{1+x^2} \leq \ln(1+x^2)$ for all $x \in \mathbb{R}$ and Lemma \ref{det}.

As for $\sum_{t=1}^n \alpha_{h,t}$, similar to \cite{dani2008stochastic} and \cite{yang2019reinforcement}, we first define its trancated version. By leveraging the concentration property of the trancated martingale, we derive a high probability upper bound for $\Theta_{h,n}$.
Take a sequence $0 \leq \vartheta_{h,0} \leq \vartheta_{h,1} \leq \ldots \leq \vartheta_{h,N}$. We consider a series of events
\[ \text{$\mathcal{E}_{h,0}^{\vartheta} := $ the whole sample space}, \qquad \mathcal{E}_{h,n}^{\vartheta} := \big\{ \Theta_{h,t} \leq \vartheta_{h,t} \text{ for $t = 0, 1,\ldots,n$} \big\}, \ n = 1,2,\ldots,N. \]
Define
\[ \alpha_{h,n}^{\star} := 2\big(r_n'+V_{h+1}^{\pi}(s_n')-Q_h^{\pi}(s_n, a_n) \big) \frac{\phi(s_n,a_n)^{\top}(\widehat{W}_{h,n-1}^{\pi}-w_h^{\pi})}{1 + \phi(s_n,a_n)^{\top} \widehat{\Sigma}_{n-1}^{-1}\phi(s_n,a_n)} \cdot \mathbbm{1}_{\mathcal{E}_{h,n-1}^{\vartheta}}. \]
Then, $\alpha_{h,n}^{\star}$ is a martingale difference with respect to $\mathcal{F}_n$. Similar to Lemma 14 in \cite{dani2008stochastic} and Lemma 13 in \cite{yang2019reinforcement}, we apply Freedman's inequality to show that when $\vartheta_{h,1}, \ldots, \vartheta_{h,N}$ are appropriately chosen, the truncated martingale $\big\{ \sum_{t=1}^n \alpha_{h,t}^{\star} \big\}_{n=1,2,\ldots,N}$ never grows too large.
\begin{lemma} \label{Theta1}
	Suppose for $n = 1,2,\ldots,N$,
	\begin{equation} \label{beta>} \sqrt{\vartheta_{h,n}} \geq 2\sqrt{2}(H-h+1) \sqrt{d \ln\Big( 1 + \frac{n}{\lambda d} \Big)\ln(2n^2/\delta)} + \frac{4}{3}(H-h+1)\ln(2n^2/\delta). \end{equation}
	Then with probability at least $1-\delta$, it holds for all $n = 1,2,\ldots,N$ that
	\begin{equation} \label{sumTheta1} \sum_{t=1}^n \alpha_{h,t}^{\star} \leq \frac{1}{2} \vartheta_{h,n}. \end{equation}
\end{lemma}
\begin{proof}
	Since $V_h^{\pi}(s) \in [0,H-h+1]$, 
	\begin{equation} \label{Freedman1} \begin{aligned} \big| \alpha_{h,n}^{\star} \big| \leq & 2\big|V_h^{\pi}(s_n')-\mathbb{E}[V_h^{\pi}(s_n') \, | \, s_n, a_n] \big|  \frac{\sqrt{\Theta_{h,n-1}} \sqrt{\phi(s_n,a_n)^{\top} \widehat{\Sigma}_{n-1}^{-1} \phi(s_n,a_n)}}{1 + \phi(s_n,a_n)^{\top} \widehat{\Sigma}_{n-1}^{-1}\phi(s_n,a_n)} \cdot \mathbbm{1}_{\mathcal{E}_{h,n-1}^{\vartheta}} \\ \leq & 2(H-h+1) \cdot \frac{\sqrt{\vartheta_{h,n-1}}}{2} \leq (H-h+1) \sqrt{\vartheta_{h,n}}, \end{aligned} \end{equation}
	where we have used the inequality $0 \leq \frac{x}{1+x^2} \leq \frac{1}{2}$ for all $x \geq 0$.
	Consider the conditional variance ${\rm Var}\big[ \alpha_{h,n}^{\star} \, \big| \, \mathcal{F}_n \big]$,
	\[ \begin{aligned} {\rm Var}\big[ \alpha_{h,n}^{\star} \, \big| \, \mathcal{F}_n \big] = & 4 {\rm Var}\big[V_h^{\pi}(s_n') \, \big| \, s_n, a_n\big] \Bigg( \frac{\phi(s_n,a_n)^{\top}\big(\widehat{W}_{h,n-1}^{\pi}-w_h^{\pi}\big)}{1 + \phi(s_n,a_n)^{\top} \widehat{\Sigma}_{n-1}^{-1}\phi(s_n,a_n)} \Bigg)^2 \cdot \mathbbm{1}_{\mathcal{E}_{h,n-1}^{\vartheta}} \\ \leq & (H-h+1)^2 \Theta_{h,n-1} \frac{ \phi(s_n,a_n)^{\top} \widehat{\Sigma}_{n-1}^{-1} \phi(s_n,a_n)}{\big( 1 + \phi(s_n,a_n)^{\top} \widehat{\Sigma}_{n-1}^{-1}\phi(s_n,a_n) \big)^2} \cdot \mathbbm{1}_{\mathcal{E}_{h,n-1}^{\vartheta}} \\ \leq & (H-h+1)^2 \vartheta_{h,n} \cdot \ln\big(1+ \phi(s_n,a_n)^{\top} \widehat{\Sigma}_{n-1}^{-1} \phi(s_n,a_n)\big), \end{aligned} \]
	where we have used ${\rm Var}\big[r_n'+V_{h+1}^{\pi}(s_n') \, \big| \, s_n, a_n\big] \leq \frac{1}{4}(H-h+1)^2$ and $\frac{x^2}{(1+x^2)^2} \leq \ln(1+x^2)$.
	Taking the summation and using the inequality $\vartheta_{h,t} \leq \vartheta_{h,n}$ for $t = 1,2,\ldots,n$ yields
	\begin{equation} \label{Freedman2'} \begin{aligned} \sum_{t=1}^n {\rm Var}\big[ \alpha_{h,t}^{\star} \, \big| \, \mathcal{F}_t \big] \leq & (H-h+1)^2 \vartheta_{h,n} \sum_{t=1}^n \ln\big(1+ \phi(s_t,a_t)^{\top} \widehat{\Sigma}_{t-1}^{-1} \phi(s_t,a_t)\big) \leq (H-h+1)^2 \vartheta_{h,n} d \ln\Big(1+\frac{n}{\lambda d} \Big). \end{aligned} \end{equation}
	We denote $\sigma_n^2 := (H-h+1)^2 \vartheta_{h,n} d \ln\big( 1 + \frac{n}{\lambda d} \big)$.
	
	According to \eqref{Freedman1} and \eqref{Freedman2'}, the Freedman's inequality implies that
	\[ \mathbb{P} \Bigg( \sum_{t=1}^n \alpha_{h,t}^{\star} \geq \vartheta_{h,n}/2, \sum_{t=1}^n {\rm Var}\big[ \alpha_{h,t}^{\star} \, \big| \, \mathcal{F}_t \big] \leq \sigma_n^2 \Bigg) \leq \exp \Bigg( - \frac{\vartheta_{h,n}^2/8}{\sigma_n^2+(H-h+1) \vartheta_{h,n}^{3/2}/6} \Bigg). \]
	When $\vartheta_{h,n}$ satisfies \eqref{beta>},
	\[ \begin{aligned} \mathbb{P} \Bigg( \sum_{t=1}^n \alpha_{h,t}^{\star} \geq \vartheta_{h,n}/2 \Bigg) = & \mathbb{P} \Bigg( \sum_{t=1}^n \alpha_{h,t}^{\star} \geq \vartheta_{h,n}/2, \sum_{t=1}^n {\rm Var}\big[ \alpha_{h,t}^{\star} \, \big| \, \mathcal{F}_t \big] \leq \sigma_n^2 \Bigg) \leq \frac{\delta}{2n^2}. \end{aligned} \]
	By union bound,
	\[ \mathbb{P} \Bigg( \exists n = 1,2,\ldots,N : \sum_{t=1}^n \alpha_{h,t}^{\star} \geq \vartheta_{h,n}/2 \Bigg) \leq \sum_{n=1}^N \frac{\delta}{2n^2} \leq \delta, \]
	which completes the proof.
\end{proof}

Based on the concentration inequalities in Lemma \ref{Theta1}, we now derive an upper bound for $\Theta_{h,n}$ by induction. 

\begin{lemma} \label{lemma:Theta<} If we take $\sqrt{\vartheta_{h,0}} = \sqrt{2\lambda} \| w_h^{\pi} \|_2$ and
	\begin{equation} \label{beta>>} \sqrt{\vartheta_{h,n}} = \sqrt{2\lambda} \| w_h^{\pi} \|_2 + 2\sqrt{2}(H-h+1) \sqrt{d \ln\Big( 1 + \frac{n}{\lambda d} \Big)\ln(2n^2/\delta)} + \frac{4}{3}(H-h+1)\ln(2n^2/\delta) \end{equation}
	for $n=1,2,\ldots,N$,
	then with probability at least $1-\delta$,
	\begin{equation} \label{Theta<} \Theta_{h,n} \leq \vartheta_{h,n}, \qquad \text{for $n=0,1,2,\ldots,N$}. \end{equation}
\end{lemma}
\begin{proof}
	Define an event $\mathcal{E}_h^{\vartheta} := \big\{ \sum_{t=1}^n \!\alpha_{h,t}^{\star} \leq \vartheta_{h,n}/2 \text{ for $n=1,\ldots,N$} \big\}$.
	Lemma \ref{Theta1} guarantees that $\mathbb{P}\big(\mathcal{E}_h^{\vartheta}\big) \!\geq\! 1 \!-\! \delta$ under condition \eqref{beta>>}.
	In the following, we prove by induction that $\mathcal{E}_h^{\vartheta}$ implies \eqref{Theta<}.
	
	Note that $\Theta_{h,0} = \lambda \|w_h^{\pi}\|_2^2 \leq \vartheta_{h,0}$.
	Suppose \begin{equation} \label{inductive} \Theta_{h,t} \leq \vartheta_{h,t}, \qquad \text{for $t\!=\!0, 1,\ldots,n-1$.} \end{equation} We now consider $\Theta_{h,n}$. Under the inductive condition \eqref{inductive}, we have $ \mathbbm{1}_{\mathcal{E}_{h,0}^{\vartheta}} = \ldots = \mathbbm{1}_{\mathcal{E}_{h,n-1}^{\vartheta}} = 1$, which ensures $\alpha_{h,t} = \alpha_{h,t}^{\star}$ for $t = 1,\ldots,n$. According to Lemma \ref{Theta_decompose}, $\Theta_{h,n}$ satisfies $\Theta_{h,n} \! \leq \! \lambda\|w_h^{\pi}\|_2^2 + \! \sum_{t=1}^n \alpha_{h,t}^{\star} \! + \! \sum_{t=1}^n \beta_{h,t}$. If $\mathcal{E}_h^{\vartheta}$ happens, then by \eqref{shift} we further have
	\[ \begin{aligned} \Theta_{h,n} \leq \lambda\|w_h^{\pi}\|_2^2 + \frac{1}{2} \vartheta_{h,n} + (H-h+1)^2 d \ln\Big( 1 + \frac{n}{\lambda d} \Big). \end{aligned} \] Condition \eqref{beta>>} further implies $\Theta_{h,n} \leq \vartheta_{h,n}$. By induction, we conclude that under $\mathcal{E}_h^{\vartheta}$, $\Theta_{h,n} \leq \vartheta_{h,n}$ for $n=0,1,\ldots,N$.
\end{proof}

We now complete the proof of Theorem \ref{UpperBound'}.
\begin{proof}[Proof of Theorem \ref{UpperBound'}]
By union bound, Lemma \ref{lemma:Theta<} implies that with probability at least $1-\delta$, it holds for all $h=0,1,2,\ldots,H$ that
\begin{equation} \label{Theta<<} \begin{aligned} & \sqrt{\Theta_h} =\! \sqrt{\Theta_{h,N}} \leq \! \sqrt{\vartheta_{h,N}} = \! \sqrt{2\lambda} \| w_h^{\pi} \|_2 + 2\sqrt{2}(H\!-\!h\!+\!1) \sqrt{d \ln\Big( 1 \!+\! \frac{n}{\lambda d} \Big)\ln(3N^2H/\delta)} + \frac{4}{3}(H\!-\!h\!+\!1)\ln(3N^2H/\delta). \end{aligned} \end{equation}
Here,
\begin{equation} \label{w<} \| w_h^{\pi} \|_2 \leq (H-h+1) \omega, \qquad\omega =  \max \big\{ \|w\|_2 \, \big| \, 0 \leq \phi(s,a)^{\top} w \leq 1 \text{ for all $(s,a) \in \mathcal{X}$} \big\}. \end{equation}
Plugging \eqref{Theta<<} and \eqref{w<} into \eqref{Q_decompose0} yields
\[ \begin{aligned} \big| v^{\pi} - \widehat{v}^{\pi} \big| \leq & \sum_{h=0}^{H} (H-h+1) \sqrt{(\widehat{\nu}_h^{\pi})^{\top}\widehat{\Sigma}^{-1} \widehat{\nu}_h^{\pi}} \cdot \bigg( \! \sqrt{2\lambda} \omega \!+\! 2\sqrt{2} \sqrt{d \ln\Big( 1 + \frac{n}{\lambda d} \Big)\ln(3N^2H/\delta)} + \frac{4}{3}\ln(3N^2H/\delta) \bigg), \end{aligned} \]
which completes the proof of Theorem \ref{UpperBound'}.
\end{proof}



\section{Proof of Infinite-Horizon Discounted MDP} \label{appendix:proof:discountMDP}

We first present some preliminary results in Lemma \ref{lemma:decompose_gamma}.
\begin{lemma} \label{lemma:decompose_gamma}
	\begin{enumerate}
		\item It always holds that \begin{equation}  \label{Q_decompose_gamma}  Q^{\pi} - \widehat{Q}^{\pi} = \sum_{h=0}^{\infty} \gamma^{h} (\widehat{\mathcal{P}}^{\pi})^{h}\Big( Q^{\pi} - \big( \widehat{r} + \gamma \widehat{\mathcal{P}}^{\pi} Q^{\pi} \big) \Big). \end{equation}
		\item $v^{\pi} - \widehat{v}^{\pi} = E_1 + E_2 + E_3$, where
		\[ \begin{aligned} E_1 := & \sum_{h=0}^{\infty} \gamma^h (\nu_h^{\pi})^{\top} \Sigma^{-1} \Bigg( \frac{1}{N} \sum_{n=1}^N \phi(s_n,a_n) \Big( Q^{\pi}(s_n, a_n) - \big( r_n' + \gamma V^{\pi}(s_n') \big) \Big) \Bigg), \\ E_2 := & \sum_{h=0}^{\infty} \gamma^h \Big( N(\widehat{\nu}_h^{\pi})^{\top} \widehat{\Sigma}^{-1} - (\nu_h^{\pi})^{\top} \Sigma^{-1} \Big) \Bigg( \frac{1}{N} \sum_{n=1}^N \phi(s_n,a_n) \Big( Q^{\pi}(s_n, a_n) - \big( r_n' + \gamma V^{\pi}(s_n') \big) \Big) \Bigg), \\ E_3 := & \lambda \sum_{h=0}^{\infty} \gamma^h (\widehat{\nu}_h^{\pi})^{\top} \widehat{\Sigma}^{-1} w^{\pi}. \end{aligned} \]
		Here, $w^{\pi} \in \mathbb{R}^d$ satisfies $\mathbb{E}[V^{\pi}(s') \, | \, s,a] = \phi(s,a)^{\top} w^{\pi}$.
		\item If $\big\| (\Sigma^{\pi})^{1/2} \Delta M^{\pi} (\Sigma^{\pi})^{-1/2} \big\|_2 \leq \frac{1-\gamma}{2\gamma}$, then \begin{align} \label{E2<=} & \begin{aligned} |E_2| \leq & \frac{1}{1-\gamma} \sqrt{(\nu_0^{\pi})^{\top} (\Sigma^{\pi})^{-1}  \nu_0^{\pi}} \cdot \big\| (\Sigma^{\pi})^{1/2} \Sigma^{-1/2} \big\|_2 \big\| \Sigma^{-1/2} \Delta W^{\pi} \big\|_2 \\ & \qquad \cdot \bigg( \frac{2\gamma}{1-\gamma} \big\| (\Sigma^{\pi})^{1/2} \Delta M^{\pi} (\Sigma^{\pi})^{-1/2} \big\|_2+ 2 \big\| \Sigma^{1/2} (\Delta X) \Sigma^{1/2} \big\|_2 \bigg), \end{aligned} \\ \label{E3<=} & \begin{aligned} |E_3| \leq & \frac{\lambda\|\Sigma^{-1}\|_2}{N} \cdot \frac{2}{(1-\gamma)^2} \sqrt{(\nu_0^{\pi})^{\top} (\Sigma^{\pi})^{-1}  \nu_0^{\pi}} \cdot \big\| (\Sigma^{\pi})^{1/2} \Sigma^{-1/2} \big\|_2 \big( 1 + \big\| \Sigma^{1/2} (\Delta X) \Sigma^{1/2} \big\|_2\big), \end{aligned} \end{align}
		where $\Delta W^{\pi} := \frac{1}{N} \sum_{n=1}^N \phi(s_n,a_n) \Big( Q(s_n, a_n) - \big( r_n' + \gamma V^{\pi}(s_n') \big) \Big)$.
	\end{enumerate}
\end{lemma}

\begin{proof}
	1. Note that for a discounted MDP, $Q^{\pi} = \sum_{h=0}^{\infty} \gamma^h (\mathcal{P}^{\pi})^h r$ and $\widehat{Q}^{\pi} = \sum_{h=0}^{\infty} \gamma^h (\widehat{\mathcal{P}}^{\pi})^h \widehat{r}$.
	By using \eqref{diff_power}, we derive that
	\[ \begin{aligned} Q^{\pi} - \widehat{Q}^{\pi} = & \sum_{h=0}^{\infty} \gamma^h \big( (\mathcal{P}^{\pi})^h - (\widehat{\mathcal{P}}^{\pi})^h \big) r + \sum_{h=0}^{\infty} \gamma^h (\widehat{\mathcal{P}}^{\pi})(r - \widehat{r}) \\ = & \sum_{h=0}^{\infty} \gamma^h \sum_{h'=1}^h \big(\widehat{\mathcal{P}}^{\pi}\big)^{h'-1}\big( \mathcal{P}^{\pi} - \widehat{\mathcal{P}}^{\pi} \big) \big(\mathcal{P}^{\pi}\big)^{h-h'} r + \sum_{h=0}^{\infty} \gamma^h (\widehat{\mathcal{P}}^{\pi})(r - \widehat{r}) \\ = & \sum_{h'=1}^{\infty} \gamma^{h'} (\widehat{\mathcal{P}}^{\pi})^{h'-1}\big( \mathcal{P}^{\pi} - \widehat{\mathcal{P}}^{\pi} \big) \sum_{h=0}^{\infty} \gamma^{h} (\mathcal{P}^{\pi})^{h} r + \sum_{h=0}^{\infty} \gamma^h (\widehat{\mathcal{P}}^{\pi})(r - \widehat{r}) \\ = & \sum_{h=0}^{\infty} \gamma^{h} (\widehat{\mathcal{P}}^{\pi})^{h}\Big( \gamma \big( \mathcal{P}^{\pi} - \widehat{\mathcal{P}}^{\pi} \big) Q^{\pi} + (r - \widehat{r}) \Big) = \sum_{h=0}^{\infty} \gamma^{h} (\widehat{\mathcal{P}}^{\pi})^{h}\Big( Q^{\pi} - \big(\widehat{r} + \gamma \widehat{\mathcal{P}}^{\pi} Q^{\pi}\big) \Big), \end{aligned} \]
	where we have used Bellman equation $Q^{\pi} = r + \gamma \mathcal{P}^{\pi} Q^{\pi}$.
	
	2. Based on \eqref{Q_decompose_gamma}, we can prove the decomposition in a way similar to Lemma \ref{Edecompose}.
	
	3. For notational convenience, define $\widehat{\boldsymbol{\nu}}^{\pi} := \sum_{h=0}^{\infty} \gamma^h \widehat{\nu}_h^{\pi}$, $\boldsymbol{\nu}^{\pi} := \sum_{h=0}^{\infty} \gamma^h \nu_h^{\pi}$ and $\Delta \boldsymbol{\nu}^{\pi} := \widehat{\boldsymbol{\nu}}^{\pi} - \boldsymbol{\nu}^{\pi}$. It is easy to see that
	$E_2 = \big( (\widehat{\boldsymbol{\nu}}^{\pi})^{\top} (N \widehat{\Sigma}^{-1}) - (\boldsymbol{\nu}^{\pi})^{\top} \Sigma^{-1} \big) \Delta W^{\pi}$. In the following, we analyze $\widehat{\boldsymbol{\nu}}^{\pi}$ and $\boldsymbol{\nu}^{\pi}$, and connect $\Delta \boldsymbol{\nu}^{\pi}$ to $\Delta M^{\pi}$.
	
	Since $\big\| (\Sigma^{\pi})^{1/2} M^{\pi} (\Sigma^{\pi})^{-1/2} \big\|_2 \leq 1$ and $\gamma \in (0,1)$, we have
	\[ \begin{aligned} (\boldsymbol{\nu}^{\pi})^{\top} = & \sum_{h=0}^{\infty} \gamma^h (\nu_h^{\pi})^{\top} = \sum_{h=0}^{\infty} \gamma^h (\nu_0^{\pi})^{\top} (M^{\pi})^h = (\nu_0^{\pi})^{\top} (\Sigma^{\pi})^{-1/2} \Bigg( \sum_{h=0}^{\infty} \big( \gamma (\Sigma^{\pi})^{1/2} M^{\pi} (\Sigma^{\pi})^{-1/2} \big)^h \Bigg) (\Sigma^{\pi})^{1/2} \\ = & (\nu_0^{\pi})^{\top} (\Sigma^{\pi})^{-1/2} \big( I - \gamma (\Sigma^{\pi})^{1/2} M^{\pi} (\Sigma^{\pi})^{-1/2} \big)^{-1} (\Sigma^{\pi})^{1/2}. \end{aligned} \]
	If $\big\| (\Sigma^{\pi})^{1/2} \Delta M^{\pi} (\Sigma^{\pi})^{-1/2} \big\|_2 < \frac{1-\gamma}{\gamma}$, then $\big\| \gamma (\Sigma^{\pi})^{1/2} \widehat{M}^{\pi} (\Sigma^{\pi})^{-1/2} \big\|_2 < 1$, which implies
	\[ (\widehat{\boldsymbol{\nu}}^{\pi})^{\top} = \sum_{h=0}^{\infty} \gamma^h (\widehat{\nu}_h^{\pi})^{\top} = (\nu_0^{\pi})^{\top} (\Sigma^{\pi})^{-1/2} \big( I - \gamma (\Sigma^{\pi})^{1/2} \widehat{M}^{\pi} (\Sigma^{\pi})^{-1/2} \big)^{-1} (\Sigma^{\pi})^{1/2}. \]
	Note that
	\begin{equation} \label{E2<} \begin{aligned} |E_2| \leq & ~ \Big( \big( \big\| (\Sigma^{\pi})^{-1/2} \boldsymbol{\nu}^{\pi} \big\|_2 + \big\| (\Sigma^{\pi})^{-1/2} \Delta \boldsymbol{\nu}^{\pi}  \big\|_2 \big)\big( 1 + \big\| \Sigma^{1/2} (\Delta X) \Sigma^{1/2} \big\|_2\big) - \big\| (\Sigma^{\pi})^{-1/2} \boldsymbol{\nu}^{\pi} \big\|_2 \Big) \\ & \cdot \big\| (\Sigma^{\pi})^{1/2} \Sigma^{-1/2} \big\|_2 \big\| \Sigma^{-1/2} \Delta W^{\pi} \big\|_2. \end{aligned} \end{equation}
	Since $\big\| \big( I \!-\! \gamma (\Sigma^{\pi})^{1/2} M^{\pi} (\Sigma^{\pi})^{-1/2} \big)^{-1} \big\|_2 \!\leq\! \sum_{h=0}^{\infty} \gamma^h \big\| (\Sigma^{\pi})^{1/2} M^{\pi}(\Sigma^{\pi})^{-1/2}\big\|_2^h \!\leq\! \frac{1}{1-\gamma}$, we have $\big\| (\Sigma^{\pi})^{-1/2} \boldsymbol{\nu}^{\pi} \big\|_2 \leq (1-\gamma)^{-1} \sqrt{(\nu_0^{\pi})^{\top} (\Sigma^{\pi})^{-1} \nu_0^{\pi}}$.
	As for $\big\| (\Sigma^{\pi})^{-1/2} \Delta \boldsymbol{\nu}^{\pi} \big\|_2$, in a way similar to the proof of \eqref{inv<2Dx}, we derive that if $\big\| (\Sigma^{\pi})^{1/2} \Delta M^{\pi} (\Sigma^{\pi})^{-1/2} \big\|_2 \leq \frac{1-\gamma}{2\gamma}$, then
	\begin{equation} \label{diff_inv'} \Big\| \big( I - \gamma (\Sigma^{\pi})^{1/2} \widehat{M}^{\pi} (\Sigma^{\pi})^{-1/2} \big)^{-1} - \big( I - \gamma (\Sigma^{\pi})^{1/2} M^{\pi} (\Sigma^{\pi})^{-1/2} \big)^{-1} \Big\|_2 \leq \frac{2\gamma}{(1-\gamma)^2} \big\| (\Sigma^{\pi})^{1/2} \Delta M^{\pi} (\Sigma^{\pi})^{-1/2} \big\|_2. \end{equation}
	It follows that
	\begin{equation} \label{S_Delta_nu_pi} \big\| (\Sigma^{\pi})^{-1/2} \Delta \boldsymbol{\nu}^{\pi}  \big\|_2 \leq \frac{2\gamma}{(1-\gamma)^2} \sqrt{(\nu_0^{\pi})^{\top} (\Sigma^{\pi})^{-1} \nu_0^{\pi}} \big\| (\Sigma^{\pi})^{1/2} \Delta M^{\pi} (\Sigma^{\pi})^{-1/2} \big\|_2 \leq \big\| (\Sigma^{\pi})^{-1/2} \boldsymbol{\nu}^{\pi}  \big\|_2. \end{equation}
	Plugging \eqref{S_Delta_nu_pi} into \eqref{E2<}, we finish the proof of \eqref{E2<=}. One can show \eqref{E3<=} in the same way.
\end{proof}

We are now ready to prove Theorem \ref{thm:UpperBound_gamma}.

\begin{proof}[Proof of Theorem \ref{thm:UpperBound_gamma}]
	1. Parallel to Lemma \ref{lemma:E1}, we define martingale differences
	\[ e_n := \sum_{h=0}^{\infty} \gamma^h (\nu_h^{\pi})^{\top} \Sigma^{-1} \phi(s_n,a_n) \Big( Q^{\pi}(s_n, a_n) - \big( r_n' + \gamma V^{\pi}(s_n') \big) \Big) \]
	such that $E_1 = \frac{1}{N} \sum_{n=1}^N e_n$. Since $V^{\pi} \in [0,\frac{1}{1-\gamma}]$, we have $|e_n| \leq \frac{1}{1-\gamma} \sqrt{\big(\sum_{h=0}^{\infty} \gamma^h \nu_h^{\pi}\big)^{\top} \Sigma^{-1} \big(\sum_{h=0}^{\infty} \gamma^h \nu_h^{\pi} \big)}\sqrt{C_1d}$ and
	\[ {\rm Var}[e_n \, | \, \mathcal{F}_n] = \bigg( \sum_{h=0}^{\infty} \gamma^h (\nu_h^{\pi})^{\top} \Sigma^{-1} \phi(s_n,a_n) \bigg)^2 {\rm Var}\big[ V^{\pi}(s_n') \, \big| \, s_n, a_n \big] \leq \bigg( \frac{1}{2(1-\gamma)} \sum_{h=0}^{\infty} \gamma^h (\nu_h^{\pi})^{\top} \Sigma^{-1} \phi(s_n,a_n) \bigg)^2. \]
	Following the same analysis as Lemma \ref{lemma:E1} and using $H \leq \frac{C_2}{1 - \gamma}$, we can prove that with probability at least $1 - \delta$,
	\begin{equation} \label{E1_gamma} |E_1| \leq \frac{1}{1 - \gamma}\sqrt{\Bigg(\sum_{h=0}^{\infty} \gamma^h \nu_h^{\pi}\bigg)^{\top} \Sigma^{-1} \Bigg(\sum_{h=0}^{\infty} \gamma^h \nu_h^{\pi} \Bigg)} \cdot \Bigg( \sqrt{\frac{\ln(4/\delta)}{2N}} + \frac{7\ln(4d/\delta)\sqrt{C_1C_2d}}{6N \sqrt{1-\gamma}} + \frac{\big(\ln(4d/\delta)\big)^{3/2} C_1 C_2 d}{3\sqrt{2}N^{3/2}(1-\gamma)} \Bigg). \end{equation}
	As for $E_2$ and $E_3$, according to Lemma \ref{lemma:decompose_gamma}, it only remains to analyze $\big\| \Sigma^{1/2}(\Delta X) \Sigma^{1/2} \big\|_2$, $\big\| (\Sigma^{\pi})^{1/2} \Delta M^{\pi} (\Sigma^{\pi})^{-1/2} \big\|_2$ and $\big\| \Sigma^{-1/2} \Delta W^{\pi} \big\|_2$. We apply the existing concentration inequalities in Appendix \ref{section:E2}. Note that the result for $\big\| \Sigma^{-1/2} \Delta W^{\pi} \big\|_2$ is analogous to $\big\| \Sigma^{-1/2} \Delta W_h^{\pi} \big\|_2$ in Lemma \ref{lemma:E3}. 
	We combine the estimations of $|E_2|$ and $|E_3|$ with \eqref{E1_gamma}, and obtain \eqref{UpperBound_gamma}. 
	
	2. We only need to adapt Lemma \ref{lemma:v-v} to the discouted MDP. We have the decomposition $v^{\pi} - \widetilde{v}^{\pi} = \gamma \sum_{h=0}^{\infty} \gamma^{h} (\widetilde{\mathcal{P}}^{\pi})^h (\mathcal{P}^{\pi} - \widetilde{\mathcal{P}}^{\pi}) Q^{\pi}$ , which yields a lower bound in \eqref{lowerbound_gamma}.
	
	3. Similar to Lemma \ref{lemma:Q_decompose0}, we have
		\begin{equation} \label{Q_decompose0_gamma} |v^{\pi} - \widehat{v}^{\pi}| \leq \sqrt{\Bigg(\sum_{h=0}^{\infty} \gamma^h\widehat{\nu}_h^{\pi} \Bigg)^{\top} \widehat{\Sigma}^{-1} \Bigg(\sum_{h=0}^{\infty} \gamma^h\widehat{\nu}_h^{\pi} \Bigg)} \cdot \sqrt{\big( \widehat{W}^{\pi} - w^{\pi} \big)^{\top} \widehat{\Sigma} \big( \widehat{W}^{\pi} - w^{\pi} \big)} \quad \text{with $\widehat{W}^{\pi} := \widehat{R} + \gamma \widehat{M}^{\pi} w^{\pi}$}. \end{equation}
		We reform $\widehat{W}^{\pi}$ into $ \widehat{W}^{\pi} = \widehat{\Sigma}^{-1} \sum_{n=1}^N \phi(s_n,a_n) \big(r_n' + \gamma V^{\pi}(s_n')\big)$, where $r_n' + \gamma V^{\pi}(s_n') \in [0, \frac{1}{1-\gamma}]$. Using the same arguments as the proof of Theorem \ref{UpperBound'}, we can analyze $\sqrt{\big( \widehat{W}^{\pi} - w^{\pi} \big)^{\top} \widehat{\Sigma} \big( \widehat{W}^{\pi} - w^{\pi} \big)}$. Plugging the result into \eqref{Q_decompose0_gamma}, we obtain \eqref{UpperBound'_gamma}.
\end{proof}




\part{}

\section{Proofs of Lemmas in Appendix \ref{appendix:proof:UpperBound}} \label{appendix:concentration}

\subsection{Proof of Lemma \ref{SMS}} \label{appendix:SMS}

\begin{proof}[Proof of Lemma \ref{SMS}]
	For any $\mu \in \mathbb{R}^{d}$, we define a function $f: \mathcal{S} \rightarrow \mathbb{R}$ such that $f(s) := \mu^{\top} {\psi}(s)$.
	By Jensen's inequality
	\begin{equation} \label{SMS1} {\mathbb{E}} \big[ f^2({s}_{t+1}) \, \big| \, {s}_0 \sim {\xi}_0 \big] = {\mathbb{E}} \Big[ {\mathbb{E}} \big[f^2({s}_{t+1}) \, \big| \, {s}_t \big] \, \Big| \, {s}_0 \sim {\xi}_0 \Big] \geq {\mathbb{E}} \Big[ {\mathbb{E}} \big[f({s}_{t+1}) \, \big| \, {s}_t \big]^2 \, \Big| \, {s}_0 \sim {\xi}_0 \Big]. \end{equation}
	The left hand side of \eqref{SMS1} satisfies
	\begin{equation} \label{SMS2} \begin{aligned} {\mathbb{E}} \big[ f^2({s}_{t+1}) \, \big| \, {s}_0 \sim {\xi}_0 \big] = {\mathbb{E}} \big[ \big( \mu^{\top} {\psi}({s}_{t+1}) \big)^2 \, \big| \, {s}_0 \sim {\xi}_0 \big] = \mu^{\top} {\mathbb{E}} \big[ {\psi}({s}_{t+1}) {\psi}({s}_{t+1})^{\top} \, \big| \, {s}_0 \sim {\xi}_0 \big] \mu = \mu^{\top} {\Sigma}_{t+1} \mu. \end{aligned} \end{equation}
	We also have
	\[ {\mathbb{E}} \big[f({s}_{t+1}) \, \big| \, {s}_t \big] = {\mathbb{E}} \big[{\psi}({s}_{t+1})^{\top} \, \big| \, {s}_t \big] \mu = {\psi}({s}_t)^{\top} M \mu. \]
	Therefore, the right hand side of \eqref{SMS1} equals to
	\begin{equation} \label{SMS3} \begin{aligned} & {\mathbb{E}} \Big[ {\mathbb{E}} \big[f({s}_{t+1}) \, \big| \, {s}_t \big]^2 \, \Big| \, {s}_0 \!\sim\! {\xi}_0 \Big] = {\mathbb{E}} \Big[ \big( {\psi}({s}_t)^{\top} M \mu \big)^2 \, \Big| \, {s}_0 \!\sim\! {\xi}_0 \Big] = \mu^{\top} M^{\top} {\mathbb{E}} \big[ {\psi}({s}_t) {\psi}({s}_t)^{\top} \, \big| \, {s}_0 \!\sim\! {\xi}_0 \big] M \mu = \mu^{\top} M^{\top} {\Sigma}_t M \mu. \end{aligned} \end{equation}
	Plugging \eqref{SMS2} and \eqref{SMS3} into \eqref{SMS1}, we have
	\[ \mu^{\top} M^{\top} {\Sigma}_t M \mu \leq \mu^{\top} {\Sigma}_{t+1} \mu, \qquad \text{for all $\mu \in \mathbb{R}^{d}$}. \]
	It follows that ${\Sigma}_{t+1}^{-1/2} M^{\top} {\Sigma}_t M {\Sigma}_{t+1}^{-1/2} \preceq I$.
	Hence, $\big\| {\Sigma}_t^{1/2} M {\Sigma}_{t+1}^{-1/2} \big\|_2 \leq 1$.
\end{proof}

\subsection{Proof of Lemma \ref{Equivalence}} \label{appendix: equivalence}

\begin{proof}[Proof of Lemma \ref{Equivalence}]
	Note that
	\[ \sqrt{(\nu_h^{\pi})^{\top} \Sigma^{-1} \nu_h^{\pi}} = \sup_{\mu \in \mathbb{R}^d} \frac{(\nu_h^{\pi})^{\top} \mu}{\sqrt{\mu^{\top} \Sigma \mu}}. \]
	For any $\mu \in \mathbb{R}^d$, we take $f \in \mathcal{Q}$ such that $f(s,a) = \phi(s,a)^{\top} \mu$ for all $(s,a) \in \mathcal{X}$. Then $(\nu_h^{\pi})^{\top} \mu = \mathbb{E}^{\pi}\big[ f(s_h,a_h) \, \big| \, s_0 \sim \xi_0 \big]$ according to the definition of $\nu_h^{\pi}$. We can also rewrite $\sqrt{\mu^{\top} \Sigma \mu}$ with the use of function $f$. The definition of $\Sigma$ suggests that \vspace{-0.1cm}
	\[ \begin{aligned} \mu^{\top} \Sigma \mu = & \mu^{\top} \mathbb{E} \bigg[ \frac{1}{H} \sum_{h=0}^{H-1} \phi(s_{1,h},a_{1,h}) \phi(s_{1,h},a_{1,h})^{\top} \bigg] \mu = \, \frac{1}{H} \sum_{h=0}^{H-1} \mathbb{E}\Big[ \big( \phi(s_{1,h},a_{1,h})^{\top} \mu \big)^2 \Big] = \frac{1}{H} \sum_{h=0}^{H-1} \mathbb{E} \big[ f^2(s_{1,h}, a_{1,h}) \big]. \end{aligned} \]
	Since $\mathcal{Q}$ is isomorphic to $\mathbb{R}^d$, we have
	\[ \sqrt{(\nu_h^{\pi})^{\top} \Sigma^{-1} \nu_h^{\pi}} = \sup_{\mu \in \mathbb{R}^d} \frac{(\nu_h^{\pi})^{\top}\mu}{\sqrt{\mu^{\top} \Sigma \mu}} = \sup_{f \in \mathcal{Q}} \frac{ \mathbb{E}^{\pi}\big[ f(s_h,a_h) \, \big| \, s_0 \sim \xi_0 \big]}{\sqrt{\mathbb{E} \big[ \frac{1}{H} \sum_{h=0}^{H-1} f^2(s_{1,h},a_{1,h}) \big]}}. \]
	One can prove \eqref{equvalence_1.2} in a similar way. 
\end{proof}

\subsection{Proof of Lemma \ref{lemma:Term1}} \label{appendix:Term1}
\begin{proof}[Proof of Lemma \ref{lemma:Term1}]
	For each episode $\boldsymbol{\tau}_k = \big(s_{k,0}, a_{k,0}, s_{k,1}, a_{k,1}, \ldots, s_{H-1}, a_{H-1}, s_H\big)$, we define
	\[ X_k := \frac{1}{H}\sum_{h=0}^{H-1} \Sigma^{-1/2} \phi(s_{k,h},a_{k,h}) \phi(s_{k,h},a_{k,h})^{\top} \Sigma^{-1/2} \in \mathbb{R}^{d \times d}. \]
	Then,
	\begin{equation} \label{1_0} N^{-1} \Sigma^{-1/2} \widehat{\Sigma} \Sigma^{-1/2} - I = \frac{1}{K} \sum_{k=1}^K ( X_k - I )  + \frac{\lambda}{N} \Sigma^{-1}. \end{equation}
	It is easy to see that $X_1, X_2, \ldots, X_K$ are independent and $\mathbb{E} [ X_k ] = I$.
	In the following, we apply the matrix-form Bernstein inequality to analyze the concentration of $\frac{1}{K} \sum_{k=1}^K X_k$.
	
	We first consider the matrix-valued variance ${\rm Var}(X_k)  = \mathbb{E} \big[ (X_k - I)^2 \big] = \mathbb{E} \big[ X_kX_k \big] - I$. 
	Denote
	\begin{equation} \label{Phik} \Phi_k := \Big[ \phi(s_{k,0},a_{k,0}), \ldots, \phi(s_{k,H-1},a_{k,H-1}) \Big] \in \mathbb{R}^{d \times H}. \end{equation}
	Then $X_k = \frac{1}{H} \Sigma^{-1/2} \Phi_k \Phi_k^{\top} \Sigma^{-1/2}$.
	For any vector $\mu \in \mathbb{R}^{d}$,
	\begin{equation} \label{1_2} \begin{aligned} \mu^{\top} \mathbb{E}\big[ X_k^2 \big] \mu = & \mathbb{E} \Big[ \big\| X_k \mu \big\|_2^2 \Big] = \frac{1}{H^2} \mathbb{E} \Big[ \big\|  \Sigma^{-1/2} \Phi_k \Phi_k^{\top} \Sigma^{-1/2} \mu \big\|_2^2 \Big] \leq \frac{1}{H^2} \mathbb{E} \Big[ \big\|  \Sigma^{-1/2} \Phi_k \big\|_2^2 \big\| \Phi_k^{\top} \Sigma^{-1/2} \mu \big\|_2^2 \Big]. \end{aligned} \end{equation}
	Since $\big| \big(\Phi_k^{\top} \Sigma^{-1} \Phi_k\big)_{ij} \big| \leq C_1 d$ for all $i,j = 1,2,\ldots,H$, we have $\big\| \Sigma^{-1/2} \Phi_k \big\|_2^2 = \big\| \Phi_k^{\top} \Sigma^{-1} \Phi_k \big\|_2 \leq \big\| \Phi_k^{\top} \Sigma^{-1} \Phi_k \big\|_F \leq C_1 d H$. It follows from \eqref{1_2} that
	\[ \mu^{\top} \mathbb{E}\big[ X_k^2 \big] \mu \leq C_1 d \cdot \frac{1}{H} \mathbb{E} \Big[ \big\| \Phi_k^{\top} \Sigma^{-1/2} \mu \big\|_2^2 \Big] = C_1 d \cdot \mu^{\top} \mathbb{E} \big[ X_k \big] \mu = C_1 d \cdot \| \mu \|_2^2, \]
	where we used the identity $\frac{1}{H} \big\| \Phi_k^{\top} \Sigma^{-1/2} \mu \big\|_2^2 = \mu^{\top} X_k \mu$ and $\mathbb{E}\big[ X_k \big] = I$.
	We have
	\begin{equation}\label{Variance1} {\rm Var}(X_k) \preceq \mathbb{E}\big[ X_k^2 \big] \preceq C_1 d \cdot I. \end{equation}
	
	Additionally,
	\[ -I \preceq X_k -I = \frac{1}{H} \sum_{h=0}^{H-1} \Sigma^{-1/2} \phi(s_{k,h},a_{k,h}) \phi(s_{k,h},a_{k,h})^{\top} \Sigma^{-1/2} - I \preceq C_1 d \cdot I - I. \]
	Therefore, $\|X_k-I\|_2 \leq C_1 d$.
	Since $X_1, X_2, \ldots, X_K$ are {\it i.i.d.}, by the matrix-form Bernstein inequality, we have
	\[ \mathbb{P} \Bigg( \bigg\| \sum_{k=1}^K (X_k - I) \bigg\|_2 \geq \varepsilon \Bigg) \leq 2d \cdot \exp\bigg( - \frac{\varepsilon^2/2}{C_1 d K + C_1 d \varepsilon/3} \bigg), \quad \forall \varepsilon \geq 0. \]
	With probability at least $1 - \delta$,
	\begin{equation} \label{1_3} \bigg\| \frac{1}{K}\sum_{k=1}^K (X_k - I) \bigg\|_2 \leq \sqrt{\frac{2\ln(2d/\delta)C_1 d}{K}} + \frac{2 \ln(2d/\delta)C_1 d}{3K}, \end{equation}
	from which we derive \eqref{Term1}.	
\end{proof}

\subsection{Proof of \eqref{eqnE1_new} in Lemma \ref{lemma:E1}} \label{appendix:eqnE1_new}

\begin{proof}[Proof of \eqref{eqnE1_new}]
	The only difference between the proofs of \eqref{eqnE1} and \eqref{eqnE1_new} is the estimate of conditional variance ${\rm Var}\big[ e_n \, \big| \, \mathcal{F}_n \big] $. We will show it in detail.
	
	We expand the conditional variance ${\rm Var}\big[ e_n \, \big| \, \mathcal{F}_n \big] $ into $(H+1)^2$ terms, 
	\begin{equation} \label{VarE1_new} \begin{aligned} & {\rm Var} \big[ e_n \, \big| \, \mathcal{F}_n \big] = \mathbb{E} \Bigg[ \bigg( \sum_{h=0}^H (\nu_h^{\pi})^{\top} \Sigma^{-1} \phi(s_n,a_n) \Big( Q_h^{\pi}(s_n, a_n) - \big( r_n'+V_{h+1}^{\pi}(s_n') \big) \Big) \bigg)^2 \, \Bigg| \, \mathcal{F}_n \Bigg] \\ = & \sum_{h_1=0}^{H} \sum_{h_2=0}^{H} \Big( (\nu_{h_1}^{\pi})^{\top} \Sigma^{-1} \phi(s_n,a_n) \Big)\Big( (\nu_{h_2}^{\pi})^{\top} \Sigma^{-1} \phi(s_n,a_n) \Big) {\rm Cov} \Big( r_n'+V_{h_1+1}^{\pi}(s_n'), r_n'+V_{h_2+1}^{\pi}(s_n') \, \Big| \, s_n, a_n \Big). \end{aligned} \end{equation}
	Recall that $r_n' + V_{h+1}^{\pi}(s_n') \in [0,H-h]$ for all $s \in \mathcal{S}$ and $h=0,1,2,\ldots,H$. It follows that
	\[ \begin{aligned} {\rm Cov} \Big( r_n' + V_{h_1+1}^{\pi}(s_n'), r_n'+V_{h_2+1}^{\pi}(s_n') \, \Big| \, s_n, a_n \Big) \leq & \sqrt{{\rm Var} \big( r_n'+V_{h_1+1}^{\pi}(s_n') \, \big| \, s_n, a_n \big) {\rm Var} \big( r_n'+V_{h_2+1}^{\pi}(s_n') \, \big| \, s_n, a_n \big)} \\ \leq & \frac{1}{4} (H-h_1+1) (H-h_2+1). \end{aligned} \]
	Under the condition $\phi(s,a)^{\top} \Sigma^{-1} \phi(s',a') \geq 0$ for all$(s,a), (s',a') \in \mathcal{X}$, it holds that
	\[ \begin{aligned} & \Big( (\nu_{h_1}^{\pi})^{\top} \Sigma^{-1} \phi(s_n,a_n) \Big)\Big( (\nu_{h_2}^{\pi})^{\top} \Sigma^{-1} \phi(s_n,a_n) \Big) {\rm Cov} \Big( r_n'+V_{h_1+1}^{\pi}(s_n'), r_n'+V_{h_2+1}^{\pi}(s_n') \, \Big| \, s_n, a_n \Big) \\ \leq & \Big( (\nu_{h_1}^{\pi})^{\top} \Sigma^{-1} \phi(s_n,a_n) \Big)\Big( (\nu_{h_2}^{\pi})^{\top} \Sigma^{-1} \phi(s_n,a_n) \Big) \cdot \frac{1}{4}(H-h_1+1)(H-h_2+1). \end{aligned} \]
	Therefore, \eqref{VarE1_new} further implies
	\[ \begin{aligned} {\rm Var} \big[ e_n \, \big| \, \mathcal{F}_n \big] \leq & \sum_{h_1=0}^{H} \sum_{h_2=0}^{H} \Big( (\nu_{h_1}^{\pi})^{\top} \Sigma^{-1} \phi(s_n,a_n) \Big)\Big( (\nu_{h_2}^{\pi})^{\top} \Sigma^{-1} \phi(s_n,a_n) \Big) \cdot \frac{1}{4}(H-h_1+1)(H-h_2+1) \\ = & \bigg( \frac{1}{2} \sum_{h=0}^{H} (H-h+1) (\nu_h^{\pi})^{\top} \Sigma^{-1} \phi(s_n,a_n) \bigg)^2. \end{aligned} \]
	Therefore,
	\[ \sum_{n=1}^N {\rm Var} \big[ e_n \, \big| \, \mathcal{F}_n \big] \leq \frac{1}{4} \Bigg(  \sum_{h=0}^{H} (H-h+1) \nu_h^{\pi} \Bigg)^{\top} \Sigma^{-1} \Bigg( \sum_{n=1}^N \phi(s_n,a_n) \phi(s_n,a_n)^{\top} \Bigg) \Sigma^{-1} \Bigg(  \sum_{h=0}^{H} (H-h+1) \nu_h^{\pi} \Bigg). \]
	Lemma \ref{lemma:Term1} ensures that with probability at least $1 - \delta/2$,
	\begin{equation} \label{E1Var<} \begin{aligned} \sum_{n=1}^N {\rm Var} \big[ e_n \, \big| \, \mathcal{F}_n \big] \leq & N \cdot \frac{1}{4} \Bigg( \sum_{h=0}^{H} (H-h+1) \nu_h^{\pi} \Bigg)^{\top} \Sigma^{-1} \Bigg( \sum_{h=0}^{H} (H-h+1) \nu_h^{\pi} \Bigg) \\ & \qquad \cdot \Bigg( 1 + \sqrt{\frac{2\ln(4d/\delta)C_1dH}{N}} + \frac{2\ln(4d/\delta) C_1 d H}{3N} \Bigg). \end{aligned} \end{equation}
	Note that by triangle inequality,
	\[ \Bigg( \sum_{h=0}^{H} (H-h+1) \nu_h^{\pi} \Bigg)^{\top} \Sigma^{-1} \Bigg( \sum_{h=0}^{H} (H-h+1) \nu_h^{\pi} \Bigg) \leq \Bigg( \sum_{h=0}^{H} (H-h+1) \sqrt{(\nu_h^{\pi})^{\top} \Sigma^{-1} \nu_h^{\pi}} \Bigg)^2 = B^2, \]
	where $B$ is defined by \eqref{B}.
	Then we follow the same arguments as in the proof of \eqref{eqnE1}, while taking $\sigma^2$ to be
	\[ \begin{aligned} \sigma^2 := & \frac{N}{4} \Bigg( \sum_{h=0}^{H} (H-h+1) \nu_h^{\pi} \Bigg)^{\top} \Sigma^{-1} \Bigg( \sum_{h=0}^{H} (H-h+1) \nu_h^{\pi} \Bigg) + N \Bigg( \sqrt{\frac{2\ln(4d/\delta)C_1dH}{N}} + \frac{2\ln(4d/\delta) C_1 d H}{3N} \Bigg) \cdot \frac{B^2}{4}. \end{aligned} \]
Then we obtain \eqref{eqnE1_new}.
\end{proof}

\subsection{Proof of Lemma \ref{lemma:E2decompose}} \label{appendix:E2decompose}

\begin{proof}[Proof of Lemma \ref{lemma:E2decompose}]
	1. Recall the definition of $E_2$,
	\begin{equation} \label{E2} E_2 = \sum_{h=0}^{H} \Big( N ( \nu_0^{\pi} )\!^{\top} \big( \widehat{M}^{\pi} \big)^h \widehat{\Sigma}^{-1} - \big( \nu_h^{\pi} \big)^{\top} \Sigma^{-1} \Big) \Delta W_{h}^{\pi}. \end{equation}
	In order to leverage the contraction inequality \eqref{SMS<1}, we decompose the power term $\big(\widehat{M}^{\pi}\big)^h$ in \eqref{E2} into
	\begin{equation} \label{power} \begin{aligned} \big(\widehat{M}^{\pi}\big)^h = & (\Sigma^{\pi})^{-1/2} \big( (\Sigma^{\pi})^{1/2} \widehat{M}^{\pi} (\Sigma^{\pi})^{-1/2} \big) \cdots \big( (\Sigma^{\pi})^{1/2} \widehat{M}^{\pi} (\Sigma^{\pi})^{-1/2} \big) (\Sigma^{\pi})^{1/2} \\ = & (\Sigma^{\pi})^{-1/2} \Big( (\Sigma^{\pi})^{1/2} M^{\pi} (\Sigma^{\pi})^{-1/2} + (\Sigma^{\pi})^{1/2} \Delta M^{\pi} (\Sigma^{\pi})^{-1/2} \Big) \\ & \qquad \qquad \cdots \Big( (\Sigma^{\pi})^{1/2} M^{\pi} (\Sigma^{\pi})^{-1/2} + (\Sigma^{\pi})^{1/2} \Delta M^{\pi} (\Sigma^{\pi})^{-1/2} \Big) (\Sigma^{\pi})^{1/2}. \end{aligned} \end{equation}
	Plugging \eqref{power} and $N\widehat{\Sigma}^{-1} = \Sigma^{-1} + \Delta X$ into \eqref{E2} and expanding the polynomial, we obtain \vspace{-1cm}
	\[ \begin{aligned} E_2 = & \sum_{h=0}^{H} \sum_{\begin{subarray}{c} ( \delta_{h,0}, \delta_{h,1}, \ldots, \delta_{h,h} ) \\  \in \{ 0,1 \}^h \backslash \{0\}^h \end{subarray}} \!\!\! \begin{aligned} & \\ & \\ & ( \nu_0^{\pi} )\!^{\top} (\Sigma^{\pi})^{-1/2} \big( (\Sigma^{\pi})^{1/2} M^{\pi} (\Sigma^{\pi})^{-1/2} \big)^{1-\delta_{h,1}} \big( (\Sigma^{\pi})^{1/2} \Delta M^{\pi} (\Sigma^{\pi})^{-1/2} \big)^{\delta_{h,1}} \\ & \qquad \cdots \big( (\Sigma^{\pi})^{1/2} M^{\pi} (\Sigma^{\pi})^{-1/2} \big)^{1-\delta_{h,h}} \big( (\Sigma^{\pi})^{1/2} \Delta M^{\pi} (\Sigma^{\pi})^{-1/2} \big)^{\delta_{h,h}} \\ & \quad \qquad \cdot (\Sigma^{\pi})^{1/2} \Sigma^{-1/2} \big( \Sigma^{1/2} (\Delta X) \Sigma^{1/2} \big)^{\delta_{h,0}} \Sigma^{-1/2} \Delta W_{h}^{\pi}. \end{aligned} \end{aligned} \]
	By taking the operator norm $\| \cdot \|_2$, \vspace{-1cm}
	\[ \begin{aligned} |E_2| \leq & \sum_{h=0}^{H} \sum_{\begin{subarray}{c} ( \delta_{h,0}, \delta_{h,1}, \ldots, \delta_{h,h} ) \\  \in \{ 0,1 \}^h \backslash \{0\}^h \end{subarray}} \!\!\!\!\! \begin{aligned} & \\ & \\ & \big\| ( \nu_0^{\pi} )\!^{\top} (\Sigma^{\pi})^{-1/2} \big\|_2 \big\| (\Sigma^{\pi})^{1/2} M^{\pi} (\Sigma^{\pi})^{-1/2} \big\|_2^{1-\delta_{h,1}} \big\| (\Sigma^{\pi})^{1/2} \Delta M^{\pi} (\Sigma^{\pi})^{-1/2} \big\|_2^{\delta_{h,1}} \\ & \quad \cdots \big\| (\Sigma^{\pi})^{1/2} M^{\pi} (\Sigma^{\pi})^{-1/2} \big\|^{1-\delta_{h,h}} \big\| (\Sigma^{\pi})^{1/2} \Delta M^{\pi} (\Sigma^{\pi})^{-1/2} \big\|_2^{\delta_{h,h}} \\ & \qquad \cdot \big\| (\Sigma^{\pi})^{1/2} \Sigma^{-1/2} \big\|_2 \big\| \Sigma^{1/2} (\Delta X) \Sigma^{1/2} \big\|_2^{\delta_{h,0}} \big\| \Sigma^{-1/2} \Delta W_{h}^{\pi} \big\|_2 \end{aligned} \\ \leq & \sqrt{(\nu_0^{\pi})^{\top} (\Sigma^{\pi})^{-1} \nu_0^{\pi}} \cdot \big\| (\Sigma^{\pi})^{1/2} \Sigma^{-1/2} \big\|_2 \\ & \cdot \sum_{h=0}^{H} \bigg( \Big( \big\| (\Sigma^{\pi})^{1/2} M^{\pi} (\Sigma^{\pi})^{-1/2} \big\|_2 + \big\| (\Sigma^{\pi})^{1/2} \Delta M^{\pi} (\Sigma^{\pi})^{-1/2} \big\|_2 \Big)^h\Big( 1 + \big\| \Sigma^{1/2} (\Delta X) \Sigma^{1/2} \big\|_2 \Big) \\ & \qquad \qquad \qquad \qquad \qquad \qquad \qquad \qquad \quad - \big\| (\Sigma^{\pi})^{1/2} M^{\pi} (\Sigma^{\pi})^{-1/2} \big\|_2^h \bigg) \cdot \big\| \Sigma^{-1/2} \Delta W_{h}^{\pi} \big\|_2. \end{aligned} \]
	Since $\big\| (\Sigma^{\pi})^{1/2} M^{\pi} (\Sigma^{\pi})^{-1/2} \big\|_2 \leq 1$, 
	we have
	\begin{equation} \label{E2_M} \begin{aligned} |E_2| \leq & \sqrt{(\nu_0^{\pi})^{\top} (\Sigma^{\pi})^{-1} \nu_0^{\pi}} \cdot \big\| (\Sigma^{\pi})^{1/2} \Sigma^{-1/2} \big\|_2 \\ & \cdot \sum_{h=0}^{H} \bigg( \Big( 1 + \big\| (\Sigma^{\pi})^{1/2} \Delta M^{\pi} (\Sigma^{\pi})^{-1/2} \big\|_2 \Big)^h\Big( 1 + \big\| \Sigma^{1/2} (\Delta X) \Sigma^{1/2} \big\|_2 \Big) - 1 \bigg) \cdot \big\| \Sigma^{-1/2} \Delta W_{h}^{\pi} \big\|_2. \end{aligned} \end{equation}
	
	2. 
	We now analyze $\big\| (\Sigma^{\pi})^{1/2} (\Delta M^{\pi}) (\Sigma^{\pi})^{-1/2} \big\|_2$ in \eqref{E2_M}. Let $\widehat{Y}^{\pi} := \sum_{n=1}^N \phi(s_n,a_n) \phi^{\pi}(s_n')$. By definition \eqref{vector_recursion} of $\widehat{M}^{\pi}$, $\widehat{M}^{\pi} = \widehat{\Sigma}^{-1} \widehat{Y}^{\pi}$. Using the decompositions $N \widehat{\Sigma}^{-1} = \Sigma^{-1} + \Delta X$ and $N^{-1} \widehat{Y}^{\pi} = \Sigma M^{\pi} + \Delta Y^{\pi}$, we learn that $\Delta M^{\pi} = (\Delta X) \Sigma M^{\pi} + \Sigma^{-1} \Delta Y^{\pi} + (\Delta X) (\Delta Y^{\pi})$, therefore,
	\[ \begin{aligned} & \big\| (\Sigma^{\pi})^{1/2} \Delta M^{\pi} (\Sigma^{\pi})^{-1/2} \big\|_2 \\ \leq & \big\| (\Sigma^{\pi})^{1/2} (\Delta X) \Sigma M^{\pi} (\Sigma^{\pi})^{-1/2} \big\|_2 + \big\| (\Sigma^{\pi})^{1/2} \Sigma^{-1} \Delta Y^{\pi} (\Sigma^{\pi})^{-1/2} \big\|_2 + \big\| (\Sigma^{\pi})^{1/2} (\Delta X) (\Delta Y^{\pi}) (\Sigma^{\pi})^{-1/2} \big\|_2 \\ \leq & \big\| (\Sigma^{\pi})^{1/2} \Sigma^{-1/2} \big\|_2 \big\| \Sigma^{1/2} (\Delta X) \Sigma^{1/2} \big\|_2 \big\| \Sigma^{1/2} (\Sigma^{\pi})^{-1/2} \big\|_2 \big\|  (\Sigma^{\pi})^{1/2} M^{\pi} (\Sigma^{\pi})^{-1/2} \big\|_2 \\ & + \big\| (\Sigma^{\pi})^{1/2} \Sigma^{-1/2} \big\|_2 \big\| \Sigma^{-1/2} (\Delta Y^{\pi}) \Sigma^{-1/2} \big\|_2 \big\| \Sigma^{1/2} (\Sigma^{\pi})^{-1/2} \big\|_2 \\ & + \big\| (\Sigma^{\pi})^{1/2} \Sigma^{-1/2} \big\|_2 \big\| \Sigma^{1/2} (\Delta X) \Sigma^{1/2} \big\|_2 \big\| \Sigma^{-1/2} (\Delta Y^{\pi}) \Sigma^{-1/2} \big\|_2 \big\| \Sigma^{1/2} (\Sigma^{\pi})^{-1/2} \big\|_2. \end{aligned} \]
	Recall the definition of condition number $\kappa_1$ in Theorem \ref{UpperBound} and the contraction inequality \eqref{SMS<1}. It follows that
	\begin{equation} \label{DeltaM} \begin{aligned} & \big\| (\Sigma^{\pi})^{1/2} \Delta M^{\pi} (\Sigma^{\pi})^{-1/2} \big\|_2 \\ \leq & \sqrt{\kappa_1} \big\| \Sigma^{1/2} (\Delta X) \Sigma^{1/2} \big\|_2 + \sqrt{\kappa_1} \big\| \Sigma^{-1/2} (\Delta Y^{\pi}) \Sigma^{-1/2} \big\|_2 + \sqrt{\kappa_1} \big\| \Sigma^{1/2} (\Delta X) \Sigma^{1/2} \big\|_2 \big\| \Sigma^{-1/2} \Delta Y^{\pi} \Sigma^{-1/2} \big\|_2 \\ = & \sqrt{\kappa_1} \Big( \big( 1 + \big\| \Sigma^{1/2} (\Delta X) \Sigma^{1/2} \big\|_2 \big) \big( 1 + \big\| \Sigma^{-1/2} (\Delta Y^{\pi}) \Sigma^{-1/2} \big\|_2 \big) - 1 \Big).  \end{aligned} \end{equation}

	3. Consider the error term $\big\| \Sigma^{1/2} (\Delta X) \Sigma^{1/2} \big\|_2$ in \eqref{E2_M} and \eqref{DeltaM}. 
	Note that $\Sigma^{1/2} (\Delta X) \Sigma^{1/2} = N \Sigma^{1/2} \widehat{\Sigma}^{-1} \Sigma^{1/2} - I$.
	If $\big\| N^{-1} \Sigma^{-1/2} \widehat{\Sigma}\Sigma^{-1/2} - I \big\|_2 \leq \frac{1}{2}$,
	\[ N \Sigma^{1/2} \widehat{\Sigma}^{-1} \Sigma^{1/2} = \Big( I - \big(I - N^{-1} \Sigma^{-1/2} \widehat{\Sigma}\Sigma^{-1/2}\big) \Big)^{-1} \!\! = \sum_{c=0}^{\infty} \big(I - N^{-1} \Sigma^{-1/2} \widehat{\Sigma}\Sigma^{-1/2}\big)^c. \]
	Therefore,
	\[ \begin{aligned} \big\| \Sigma^{1/2} (\Delta X) \Sigma^{1/2} \big\|_2 = & \Bigg\| \sum_{c=1}^{\infty} \big(I - N^{-1} \Sigma^{-1/2} \widehat{\Sigma}\Sigma^{-1/2}\big)^c \Bigg\|_2 \leq \sum_{c=1}^{\infty} \big\| I - N^{-1} \Sigma^{-1/2} \widehat{\Sigma}\Sigma^{-1/2} \big\|_2^c \\ \leq & \sum_{c=0}^{\infty} 2^{-c} \big\| I - N^{-1} \Sigma^{-1/2} \widehat{\Sigma}\Sigma^{-1/2} \big\|_2  = 2 \big\| I - N^{-1} \Sigma^{-1/2} \widehat{\Sigma}\Sigma^{-1/2} \big\|_2. \end{aligned} \]
\end{proof}

\subsection{Proof of Lemma \ref{lemma:Term2}} \label{appendix:Term2}


\begin{proof}[Proof of Lemma \ref{lemma:Term2}]
	Recall that $\Delta Y^{\pi} = \frac{1}{N} \sum_{n=1}^N \phi(s_n,a_n) \phi^{\pi}(s_n')^{\top} -  \Sigma M^{\pi}$. 
	For each trajectory $\boldsymbol{\tau}_k = (s_{k,0}, a_{k,0}, s_{k,1}, a_{k,1}, \ldots, s_{k,H})$, we take
	\[ Y^{\pi}_k := \frac{1}{H} \sum_{h=0}^{H-1} \Sigma^{-1/2} \phi(s_{k,h},a_{k,h}) \phi^{\pi}(s_{k,h+1})^{\top} \Sigma^{-1/2}. \]
	Then, $\Sigma^{-1/2} (\Delta Y^{\pi}) \Sigma^{-1/2} = \frac{1}{K} \sum_{k=1}^K \big( Y_k^{\pi} - \Sigma^{1/2} M^{\pi} \Sigma^{-1/2} \big)$.
	
	We first note that
	\begin{equation} \label{SMS0} \begin{aligned} \mathbb{E}\big[ Y_k^{\pi} \big] = & \frac{1}{H} \sum_{h=0}^{H-1} \mathbb{E}\Big[ \Sigma^{-1/2} \phi(s_{k,h},a_{k,h}) \phi^{\pi}(s_{k,h+1})^{\top} \Sigma^{-1/2} \Big]  \\ = & \frac{1}{H}  \sum_{h=0}^{H-1} \mathbb{E} \Big[ \Sigma^{-1/2} \phi(s_{k,h},a_{k,h}) \mathbb{E} \big[ \phi^{\pi}(s_{k,h+1})^{\top} \, \big| \, s_{k,h}, a_{k,h} \big] \Sigma^{-1/2} \Big]  \\ = & \frac{1}{H} \sum_{h=0}^{H-1} \mathbb{E} \Big[ \Sigma^{-1/2} \phi(s_{k,h},a_{k,h}) \phi(s_{k,h},a_{k,h})^{\top} M^{\pi} \Sigma^{-1/2} \Big] = \Sigma^{1/2} M^{\pi} \Sigma^{-1/2}, \end{aligned} \end{equation}
	where we have used $\mathbb{E} \big[ \frac{1}{H} \sum_{h=0}^{H-1} \phi(s_{k,h},a_{k,h}) \phi(s_{k,h},a_{k,h})^{\top} \big] = \Sigma$ and $\mathbb{E}\big[ \phi^{\pi}(s') \, \big| \, s,a \big] = \phi(s,a)^{\top} \! M^{\pi}$ for any $(s,a) \in \mathcal{X}$.
	To this end, $\Sigma^{-1/2} (\Delta Y^{\pi}) \Sigma^{-1/2} = \frac{1}{K} \sum_{k=1}^K \big( Y^{\pi}_k - \mathbb{E}Y^{\pi}_k \big)$. Since $\boldsymbol{\tau}_1, \boldsymbol{\tau}_2, \ldots, \boldsymbol{\tau}_K$ are {\it i.i.d.}, we use the matrix-form Bernstein inequality to estimate $\big\| \Sigma^{-1/2} (\Delta Y^{\pi}) \Sigma^{-1/2} \big\|_2$.
	
	Similar to $\Phi_k$ in \eqref{Phik}, we also define $\Phi_k^{\pi} := \big[ \phi^{\pi}(s_{k,1}), \phi^{\pi}(s_{k,2}), \ldots, \phi^{\pi}(s_{k,H}) \big] \in \mathbb{R}^{d \times H}$. It is easy to see that $Y_k^{\pi} = \frac{1}{H} \Sigma^{-1/2} \Phi_k (\Phi_k^{\pi})^{\top} \Sigma^{-1/2}$. For any $\mu \in \mathbb{R}^d$, we have
	\[ \begin{aligned} \mu^{\top} \mathbb{E} \big[ Y_k^{\pi} (Y_k^{\pi})^{\top} \big] \mu = & \mathbb{E} \big[ \| (Y_k^{\pi})^{\top} \mu \|_2^2 \big] = \frac{1}{H^2} \mathbb{E} \Big[ \big\| \Sigma^{-1/2} \Phi^{\pi}_k \Phi_k^{\top} \Sigma^{-1/2} \mu \big\|_2^2 \Big] \leq \frac{1}{H^2} \mathbb{E} \Big[ \big\| \Sigma^{-1/2} \Phi^{\pi}_k \big\|_2^2 \big\| \Phi_k^{\top} \Sigma^{-1/2} \mu \big\|_2^2 \Big]. \end{aligned} \]
	Parallel to the proof of Lemma \ref{Term1}, it holds that $\big\| \Sigma^{-1/2} \Phi_k^{\pi} \big\|_2^2 \leq C_1 d H$. Therefore,
	\[ \mu^{\top} \mathbb{E} \big[ Y_k^{\pi} (Y_k^{\pi})^{\top} \big] \mu \leq \frac{1}{H^2} \mathbb{E} \Big[ C_1dH \big\| \Phi_k^{\top} \Sigma^{-1/2} \mu \big\|_2^2 \Big] = \frac{C_1 d}{H} \cdot \mu^{\top} \Sigma^{-1/2} \mathbb{E} \big[ \Phi_k \Phi_k^{\top} \big] \Sigma^{-1/2} \mu = C_1 d \cdot \|\mu\|_2^2, \]
	where we have used $\Sigma = \frac{1}{H} \mathbb{E} \big[ \Phi_k \Phi_k^{\top} \big]$.
	It follows that
	\[ \begin{aligned} {\rm Var}_1(Y_k^{\pi}) := & \mathbb{E}\Big[ \big(Y_k^{\pi} - \mathbb{E} Y_k^{\pi} \big) \big( Y_k^{\pi} - \mathbb{E} Y_k^{\pi} \big)^{\top} \Big] \preceq \mathbb{E}\big[ Y_k^{\pi} (Y_k^{\pi})^{\top} \big] \preceq C_1 d \cdot I. \end{aligned} \]
	Analogously,
	\[ \begin{aligned} {\rm Var}_2(Y_k^{\pi}) := & \mathbb{E}\Big[ \big(Y_k^{\pi} - \mathbb{E} Y_k^{\pi} \big)^{\top} \big( Y_k^{\pi} - \mathbb{E} Y_k^{\pi} \big) \Big] \preceq \mathbb{E}\big[ (Y_k^{\pi})^{\top} Y_k^{\pi} \big] \preceq C_1 d \cdot \Sigma^{-1/2} \mathbb{E}\Bigg[ \frac{1}{H} \sum_{h=1}^H \phi^{\pi}(s_{k,h}) \phi^{\pi}(s_{k,h})^{\top} \Bigg] \Sigma^{-1/2}. \end{aligned} \]
	Therefore, \[\max \Big\{ \big\|{\rm Var}_1(Y_k^{\pi})\big\|_2, \big\| {\rm Var}_2(Y_k^{\pi}) \big\|_2 \Big\} \leq C_1 d \cdot \kappa_2^2,\]
	where $\kappa_2$ is defined in Theorem \ref{UpperBound}.
	It also holds that $\|Y_k^{\pi}\|_2 \leq C_1 d$.
	Hence,
	\[ \big\| Y_k^{\pi} - H \Sigma^{1/2} M^{\pi} \Sigma^{-1/2} \big\|_2 \leq 2C_1 d. \]
	Applying Bernstein inequality, we derive for any $\varepsilon \geq 0$,
	\[ \begin{aligned} & \mathbb{P} \Bigg( \bigg\| \sum_{k=1}^K \big(Y_k^{\pi} - \Sigma^{1/2} M^{\pi} \Sigma^{-1/2}\big) \bigg\|_2 \geq \varepsilon \Bigg) \leq 2d \cdot \exp\bigg( - \frac{\varepsilon^2/2}{K \cdot C_1 d \cdot \kappa_2^2 + 2C_1 d \varepsilon/3} \bigg), \end{aligned} \]
	which further implies \eqref{Term2}.
\end{proof}



\subsection{Proof of Lemma \ref{lemma:Term3}} \label{appendix:Term3}
\begin{proof}[Proof of Lemma \ref{lemma:Term3}]
	We first note that $\Sigma^{-1/2} \Delta W_h^{\pi} = \frac{1}{N} \sum_{n=1}^N W_n^{\pi,h}$, where
	\[ W_n^{\pi,h} := \Sigma^{-1/2} \phi(s_n,a_n) \Big( Q_h^{\pi}(s_n, a_n) - \big( r_n'+V_{h+1}^{\pi}(s_n') \big) \Big) \in \mathbb{R}^d \]
	and $\mathbb{E}\big[ W_n^{\pi,h} \, \big| \, \mathcal{F}_n \big] = 0$.
	Similar as the proof of Lemma \ref{lemma:E1}, we apply matrix-form Freedman's inequality \cite{tropp2011freedman} to analyze the concentration property.
	
	Consider conditional variances ${\rm Var}_1\big[ W_n^{\pi,h} \, \big| \, \mathcal{F}_n \big] := \mathbb{E} \big[ W_n^{\pi,h} (W_n^{\pi,h})^{\top} \, \big| \, \mathcal{F}_n \big] \in \mathbb{R}^{d \times d}$ and \\ ${\rm Var}_2\big[ W_n^{\pi,h} \, \big| \, \mathcal{F}_n \big] := \mathbb{E} \big[ (W_n^{\pi,h})^{\top} W_n^{\pi,h} \, \big| \, \mathcal{F}_n \big] \in \mathbb{R}$. It holds that
	\[ \begin{aligned} \big\| {\rm Var}_1\big[ W_n^{\pi,h} \, \big| \, \mathcal{F}_n \big] \big\|_2 = & \big\| \mathbb{E} \big[ W_n^{\pi,h} (W_n^{\pi,h})^{\top} \, \big| \, \mathcal{F}_n \big] \big\|_2 \leq \mathbb{E} \big[ \big\| W_n^{\pi,h} (W_n^{\pi,h})^{\top} \big\|_2 \, \big| \, \mathcal{F}_n \big] \\ = & \mathbb{E} \big[ \| W_n^{\pi,h} \|_2^2 \, \big| \, \mathcal{F}_n \big] = {\rm Var}_2\big[ W_n^{\pi,h} \, \big| \, \mathcal{F}_n \big] \end{aligned} \]
	and
	\[ \begin{aligned} {\rm Var}_2\big[ W_n^{\pi,h} \, \big| \, \mathcal{F}_n \big] = & \mathbb{E} \big[ \| W_n^{\pi,h} \|_2^2 \big| \mathcal{F}_n \big] = \phi(s_n,a_n)^{\top} \Sigma^{-1} \phi(s_n,a_n) \cdot {\rm Var} \big[ r_n' + V_{h+1}^{\pi}(s_n') \, \big| \, s_n, a_n \big] \\ \leq & \frac{1}{4} (H-h+1)^2 \cdot \phi(s_n,a_n)^{\top} \Sigma^{-1} \phi(s_n,a_n),  \end{aligned} \]
	where we have used ${\rm Var} \big[ r_n' + V_{h+1}^{\pi}(s_n') \, \big| \, s_n, a_n \big] \leq \frac{1}{4} (H-h+1)^2$.
	Note that
	\[ \begin{aligned} \sum_{n=1}^N \phi(s_n,a_n)^{\top} \Sigma^{-1} \phi(s_n,a_n) = & N d + N \cdot {\rm Tr}\Bigg( \Sigma^{-1/2} \bigg( \frac{1}{N} \sum_{n=1}^N \phi(s_n,a_n) \phi(s_n,a_n)^{\top} \bigg) \Sigma^{-1/2} - I \Bigg) \\ \leq & N d + Nd \cdot \Bigg\| \Sigma^{-1/2} \bigg( \frac{1}{N} \sum_{n=1}^N \phi(s_n,a_n) \phi(s_n,a_n)^{\top} \bigg) \Sigma^{-1/2} - I \Bigg\|_2. \end{aligned} \]
	We take
	\begin{equation} \label{sigma2} \sigma^2 := Nd \Bigg( 1 + \sqrt{\frac{2\ln\big((3d+1)/\delta\big)C_1dH}{N}} + \frac{2\ln\big((3d+1)/\delta\big)C_1 d H}{3N} \Bigg) \cdot \frac{(H-h+1)^2}{4}. \end{equation}
	According to Lemma \ref{lemma:Term1}, it holds that
	\begin{equation} \label{Var2} \mathbb{P} \Bigg( \bigg\| \sum_{n=1}^N {\rm Var}_1\big[ W_n^{\pi,h} \, \big| \, \mathcal{F}_n \big] \bigg\|_2 \leq \sum_{n=1}^N {\rm Var}_2\big[ W_n^{\pi,h} \, \big| \, \mathcal{F}_n \big] \leq \sigma^2 \Bigg) \geq 1 - \frac{2d}{3d+1} \cdot \delta. \end{equation}
	
	Additionally, we have $\big\| W_n^{\pi,h} \big\|_2 \leq (H-h+1) \sqrt{C_1d}$.
	The Freedman's inequality therefore implies that for any $\varepsilon \in \mathbb{R}$,
	\begin{equation} \label{Freedman2} \begin{aligned} & \mathbb{P}\Bigg( \bigg|\sum_{n=1}^N W_n^{\pi,h}\bigg| \geq \varepsilon, \sum_{n=1}^N {\rm Var}_1\big[ W_n \, \big| \, \mathcal{F}_n \big] \leq \sum_{n=1}^N {\rm Var}_2\big[ W_n \, \big| \, \mathcal{F}_n \big] \leq \sigma^2  \Bigg) \\ \leq & (d+1) \exp \bigg( - \frac{\varepsilon^2/2}{\sigma^2 + (H-h+1)\sqrt{C_1d} \varepsilon / 3} \bigg), \end{aligned} \end{equation}
	where $\sigma^2$ is defined in \eqref{sigma2}.
	We take \[ \varepsilon := \sqrt{2 \ln\big((3d+1)/\delta\big)} \cdot \sigma + 2 \ln\big((3d+1)/\delta\big) (H-h+1) \sqrt{C_1 d} / 3. \]
	In a way similar to the proof of Lemma \ref{lemma:E1}, \eqref{Var2} and \eqref{Freedman2} imply that \eqref{Term3} holds with probability at least $1-\delta$.
\end{proof}


\subsection{Proof of Lemma \ref{lemma:E3}} \label{appendix:E3}
\begin{proof}[Proof of Lemma \ref{lemma:E3}]
	Recall
	\[ \begin{aligned} E_3 = & \lambda \sum_{h=0}^{H} ( \widehat{\nu}_h^{\pi} )^{\top} \widehat{\Sigma}^{-1} w_{h}^{\pi} = \lambda \sum_{h=0}^{H} ( \nu_0^{\pi} )^{\top} (\Sigma^{\pi})^{-1/2} \big( \widehat{M}^{\pi} \big)^h \widehat{\Sigma}^{-1} w_{h}^{\pi} \\ = & \frac{\lambda}{N} \sum_{h=0}^H \big( ( \nu_0^{\pi} )^{\top} (\Sigma^{\pi})^{-1/2} \big) \big( (\Sigma^{\pi})^{1/2} \widehat{M}^{\pi} (\Sigma^{\pi})^{-1/2} \big)^h \big( (\Sigma^{\pi})^{1/2} \Sigma^{-1/2} \big) \big( N \Sigma^{1/2} \widehat{\Sigma}^{-1} \Sigma^{1/2} \big) \Sigma^{-1} \big( \Sigma^{1/2} w_h^{\pi} \big). \end{aligned} \]
	Hence, we have
	\begin{equation} \label{E3decompose} \begin{aligned} |E_3| \leq \frac{\lambda}{N} \sum_{h=0}^H \sqrt{( \nu_0^{\pi} )^{\top} (\Sigma^{\pi})^{-1} \nu_0^{\pi}} & \big\| (\Sigma^{\pi})^{1/2} \widehat{M}^{\pi} (\Sigma^{\pi})^{-1/2} \big\|_2^h \big\| (\Sigma^{\pi})^{1/2} \Sigma^{-1/2} \big\|_2 \\ & \cdot \big\| N \Sigma^{1/2} \widehat{\Sigma}^{-1} \Sigma^{1/2} \big\|_2 \big\| \Sigma^{-1} \big\|_2 \big\| \Sigma^{1/2} w_h^{\pi} \big\|_2. \end{aligned} \end{equation}

	If $N \geq 20 \kappa_1(2+\kappa_2)^2 \ln(8dH/\delta)C_1dH^3 $, $\lambda \leq \ln(8dH/\delta)C_1dH \sigma_{\min}(\Sigma)$ and event $\mathcal{E}_{\delta}$ in \eqref{event} happens, then
	\[ \big\| (\Sigma^{\pi})^{1/2} \Delta M^{\pi} (\Sigma^{\pi})^{-1/2} \big\|_2 \leq \frac{1}{H}, ~~  \big\| \Sigma^{1/2}(\Delta X) \Sigma^{1/2} \big\|_2 \leq \frac{1}{6H} ~~ \text{and} ~~ \frac{\lambda}{N} \big\| \Sigma^{-1} \big\|_2 \leq \frac{\ln(8dH/\delta) C_1dH}{N}. \]
	It follows that
	\begin{equation} \label{hatM} \big\| (\Sigma^{\pi})^{1/2} \widehat{M}^{\pi} (\Sigma^{\pi})^{-1/2} \big\|_2 \leq \big\| (\Sigma^{\pi})^{1/2} M^{\pi} (\Sigma^{\pi})^{-1/2} \big\|_2 + \big\| (\Sigma^{\pi})^{1/2} \Delta M^{\pi} (\Sigma^{\pi})^{-1/2} \big\|_2 \leq 1 + \frac{1}{H}, \end{equation}
	and
	\begin{equation} \label{hatSinv} \big\| N \Sigma^{1/2} \widehat{\Sigma}^{-1} \Sigma^{1/2} \big\|_2 \leq 1 + \big\| N \Sigma^{1/2} (\Delta X) \Sigma^{1/2} \big\|_2 \leq 1 + \frac{1}{6H}. \end{equation}
	We also note that
	\begin{equation} \label{wnorm} \begin{aligned} \big\| \Sigma^{1/2} w_h^{\pi} \big\|_2^2 = & (w_h^{\pi})^{\top} \Sigma w_h^{\pi} = (w_h^{\pi})^{\top} \mathbb{E} \Bigg[ \frac{1}{H} \sum_{h'=0}^{H-1} \phi(s_{k,h'}, a_{k,h'}) \phi(s_{k,h'}, a_{k,h'})^{\top} \Bigg] w_h^{\pi} \\ = & \frac{1}{H} \sum_{h'=0}^{H-1} \mathbb{E} \Big[ \big( \phi(s_{k,h'}, a_{k,h'})^{\top} w_h^{\pi} \big)^2 \Big] \leq (H-h+1)^2. \end{aligned} \end{equation}

	Plugging \eqref{hatM}, \eqref{hatSinv} and \eqref{wnorm} into \eqref{E3decompose} yields
	\[ \begin{aligned} |E_3| \leq & \frac{\lambda}{N}\big\| \Sigma^{-1} \big\|_2 \sum_{h=0}^{H} \sqrt{( \nu_0^{\pi} )^{\top} (\Sigma^{\pi})^{-1} \nu_0^{\pi}} \big\| (\Sigma^{\pi})^{1/2} \Sigma^{-1/2} \big\|_2  \Big( 1 + \frac{1}{H} \Big)^h \Big( 1 + \frac{1}{6H} \Big) (H-h+1) \\ \leq & \frac{\lambda}{N}\big\| \Sigma^{-1} \big\|_2 \sqrt{( \nu_0^{\pi} )^{\top} (\Sigma^{\pi})^{-1} \nu_0^{\pi}} \big\| (\Sigma^{\pi})^{1/2} \Sigma^{-1/2} \big\|_2 \sum_{h=0}^{H} 3 (H-h+1) \\ \leq & \sqrt{( \nu_0^{\pi} )^{\top} (\Sigma^{\pi})^{-1} \nu_0^{\pi}} \big\| (\Sigma^{\pi})^{1/2} \Sigma^{-1/2} \big\|_2 \frac{5\ln(8dH/\delta) C_1dH^2}{N}, \end{aligned} \]
	where we have used $\frac{\lambda}{N} \big\| \Sigma^{-1} \big\|_2 \leq \frac{1}{N}\ln(8dH/\delta) C_1dH$, $(1+1/x)^x \leq 3, ~ \forall x > 0$.
\end{proof}

%

\end{document}